%% file: arxiv.tex
\documentclass[11pt,letterpaper]{article}
\usepackage[margin=1in]{geometry}
\newcommand{\paperroot}{}
\IfFileExists{arxiv/arxiv.tex}{\renewcommand{\paperroot}{arxiv/}}{}

\usepackage[margin=1in]{geometry}
\usepackage[T1]{fontenc}
\usepackage{lmodern}
\usepackage{amsmath,amssymb,amsthm,mathtools}
\usepackage{microtype}
\usepackage{natbib}
\setcitestyle{authoryear,round,citesep={;},aysep={,},yysep={;}}
\usepackage{hyperref}
\hypersetup{
  hidelinks,
  pdftitle={Symmetry-Aware Feature Learning: A Polynomial Separation for Multi-Index Models},
  pdfauthor={Jivan Waber, Vanessa Piccolo, Yatin Dandi, Florent Krzakala}
}
\usepackage{xurl}
\usepackage{graphicx}
\usepackage{caption}
\usepackage{wrapfig}
\usepackage{enumitem}
\newcommand{\R}{\mathbb{R}}
\newcommand{\N}{\mathbb{N}}

\newcommand{\E}{\mathbb{E}}

\newcommand\tr{{\textnormal{Tr}}}

\newcommand{\Id}{\mathrm{I}}

\newtheorem{theorem}{Theorem}[section]
\newtheorem{proposition}[theorem]{Proposition}
\newtheorem{lemma}[theorem]{Lemma}
\newtheorem{corollary}[theorem]{Corollary}
\theoremstyle{definition}
\newtheorem{definition}[theorem]{Definition}
\theoremstyle{remark}
\newtheorem{remark}[theorem]{Remark}
\newtheorem{assumption}{Assumption}

\title{Symmetry-Aware Feature Learning:\\
A Polynomial Separation for Multi-Index Models}

\author{
Jivan Waber$^{1}$, \quad Vanessa Piccolo$^{1}$, \quad Yatin Dandi$^{1,2}$, \quad Florent Krzakala$^{1}$ \\[0.75em]
{\small $^1$\,Information, Learning and Physics Laboratory}\\
{\small $^2$\,Statistical Physics of Computation Laboratory}\\
{\small \'{E}cole Polytechnique F\'{e}d\'{e}rale de Lausanne (EPFL)}
}

\date{}

\begin{document}

\maketitle

\begin{abstract}
We establish a polynomial sample complexity separation between symmetry-aware and symmetry-agnostic feature learning. We study growing-rank multi-index models with high-dimensional Gaussian covariates in $\mathbb{R}^d$ and $r=\Theta(d^\delta)$ teacher directions forming a cyclic symmetry orbit, where $0<\delta<1/2$. We compare three ways of exploiting this structure: architectural weight sharing, data augmentation over the full symmetry group, and learning without access to the symmetry. In particular, we analyze a symmetry-tied convolutional network, an untied network, and the same untied network trained with full-group data augmentation, using spherical online SGD with correlation loss. For a class of polynomial links with information exponent $p\ge3$, we prove matching sample complexity bounds up to logarithmic factors: the tied and augmented learners achieve weak directional recovery in $\widetilde{\Theta}(d^{p-1})$ samples, whereas the symmetry-agnostic learner requires $\widetilde{\Theta}(rd^{p-1})$. For the pure quadratic Hermite link, the same separation holds for weak recovery of the teacher subspace, with sample complexities $\widetilde{\Theta}(d)$ and $\widetilde{\Theta}(rd)$, respectively.
Thus, full-group data augmentation matches the sample efficiency of architectural weight sharing, and both provide a polynomial advantage over training without symmetry. For $p\ge3$, the proof reveals a two-stage mechanism: fluctuations at initialization select one direction in the teacher orbit, after which localized growth amplifies its overlap to the weak recovery scale while competing overlaps remain near their initialization scale.
\end{abstract}

\begingroup
\small
\tableofcontents
\endgroup

\section{Introduction}
\label{sec:introduction}

Symmetry is one of the most effective inductive biases in machine learning.
When transformations of the input preserve the task, the same feature need not be learned independently at every transformed location. Convolutional neural networks exploit this principle through weight sharing, while group-equivariant architectures extend it to more general transformations~\citep{cohen2016group,kondor2018generalization,bronstein2021geometric}.  More broadly, geometric deep learning views symmetry and geometry as
organizing principles for designing neural architectures~\citep{bronstein2021geometric}.

There are two common ways to exploit this prior knowledge. One is \emph{architectural}: constrain the model through weight sharing or equivariance. The other is \emph{algorithmic}: keep the model unconstrained but train on transformed copies of each observation, as in data augmentation. Group averaging provides a natural mathematical description of the latter~\citep{chen2020group,lyle2020benefits}. A substantial literature has shown that invariance can improve variance, complexity, and generalization~\citep{elesedy2021provably,mei2021learning}, and has quantified advantages of convolutional structure in infinite-width regimes~\citep{cagnetta2023wide}. Here we study how symmetry changes the sample complexity of feature learning by stochastic gradient descent (SGD). In particular, we ask
\begin{quote}
\emph{How much does symmetry reduce the number of samples that SGD needs to discover a relevant feature, and how do weight sharing and data augmentation compare?}
\end{quote}
Answering this question mathematically requires analyzing the high-dimensional, nonconvex SGD dynamics from random initialization, rather than comparing hypothesis classes or estimators after optimization.

\paragraph{Feature learning in single- and multi-index models}
To study this question, we focus on single- and multi-index models of the form
\[
f^\ast(x) = g \left ( \langle w^\ast_1, x\rangle, \ldots, \langle w^\ast_r, x \rangle \right ) ,
\]
where \(g \colon \R^r \to \R\) is the link function and \(w_1^\ast, \ldots, w_r^\ast\) are the latent directions. These models have been used to study feature recovery, computational barriers, and the dynamics of gradient-based learning, see e.g.~\citet{dudeja2018learning,barbier2019optimal,benarous2021online,arnaboldi2023high,abbe2023sgd,dandi2024benefits,troiani2025fundamental,simsek2025learning}. In particular, for Gaussian single-index models, the degree of the first nonzero Hermite coefficient of the link function--its information exponent--determines the time scale on which online SGD escapes the nearly uninformative initialization regime~\citep{benarous2021online}.

Interactions between several latent directions produce richer dynamics. Finite-rank analyses of online SGD show that competing correlations can suppress or select one another, leading to sequential feature recovery. In multi-spiked tensor PCA, \citet{benarous2026} characterize this through a greedy maximum selection mechanism, which determines the order in which latent directions are recovered. Related selection mechanisms have also been identified in shallow multi-index models with growing numbers of features~\citep{ren2025emergence}, while other recent works study growing collections of approximately independent or diverse features~\citep{oko2024diverse,ren2025learning}. Our setting combines the difficulty of interacting directions with that of growing rank: the number of competing directions grows with the dimension, while the directions are strongly dependent because they are symmetry-related copies of a single hidden feature. This structure raises a question that is absent from fixed-rank models: can a learner avoid paying separately for every direction by exploiting their common origin? A more detailed literature review is given in Section~\ref{sec:related_work}.

\paragraph{A symmetric growing-rank model}
We study a growing-rank model designed to isolate this effect:
\[
f^\ast (x)  = \frac{1}{\sqrt r} \sum_{k=0}^{r-1} \sigma \left (\langle \Pi^k w^\ast, x \rangle \right),
\]
where $\Pi$ generates a group action of size $r$. Thus the \(r\) latent directions \(w^\ast,\Pi w^\ast,\ldots,\Pi^{r-1}w^\ast\) are not independent parameters: they form the orbit of a single direction \(w^\ast\). For cyclic shifts, this is a convolutional teacher. We consider the growing-rank regime \(r\asymp d^\delta\) with \(0<\delta<\frac12\).

We analyze two regimes for the link function \(\sigma\). Our main results concern polynomial links with information exponent $p\ge 3$. We also treat the pure quadratic Hermite link $\sigma=h_2$, for which the individual orbit directions are generally not identifiable and the natural recovery target is instead the subspace spanned by the teacher orbit.

We compare three ways of learning this target with spherical online SGD. The \emph{tied} procedure ($\mathsf{T}$) uses a single trainable weight shared across its entire orbit. The \emph{augmented} procedure (\(\mathsf{A}\)) uses an untied two-layer model but averages each sample loss over the full group orbit. The \emph{untied} procedure (\(\mathsf{U}\)) uses the same untied model without augmentation. The latter two therefore have the same parametrization; their only difference is whether the symmetry is used during training or not.

\paragraph{An \(r\)-fold separation}
Our main result shows that exploiting the symmetry reduces the sample complexity of feature recovery by a factor $r$. For polynomial links with information exponent $p\ge 3$, we consider \emph{weak directional recovery}: some learned weight must reach a fixed, non-vanishing correlation with an element of the identifiable teacher orbit. For the dimension dependent step sizes prescribed in our analysis, we prove matching upper and lower bounds (up to logarithmic factors):
\begin{equation*} \label{eq:intro-separation}
n_{\mathsf{T}} = \widetilde\Theta(d^{p-1}), \qquad n_{\mathsf{A}} = \widetilde\Theta(d^{p-1}), \qquad n_{\mathsf{U}} = \widetilde\Theta(rd^{p-1}).
\end{equation*}
The scale $\widetilde\Theta(d^{p-1})$ is the weak recovery scale of vanilla online SGD for Gaussian single-index models with information exponent $p$~\citep{benarous2021online, damian2023smoothing}. Thus, despite the teacher rank growing as $r\asymp d^\delta$, weight sharing and exact data augmentation recover a teacher feature at the single-index scale, whereas the symmetry-agnostic untied learner requires an additional factor equal to the number of orbit directions. Since $r\asymp d^\delta$, this yields a polynomial separation in the dimension. We further show that, with width $s\gtrsim r\log r$, the untied and augmented procedures cover the entire teacher orbit within their respective weak recovery sample complexity scales.

For the pure Hermite-$2$ link, individual orbit directions are generally unidentifiable. Thus, we instead consider \emph{weak subspace recovery}: some learned weight must acquire a fixed, non-vanishing fraction of its squared norm in the subspace spanned by the teacher orbit. We obtain the same $r$-fold separation: the tied and augmented procedures require $\widetilde\Theta(d)$ samples, whereas the untied learner requires $\widetilde\Theta(rd)$.

\paragraph{Proof mechanism}
For $p\ge 3$, the proof reveals a two-step \emph{selection-and-growth mechanism}. At random initialization, one of the $r\asymp d^\delta$ correlations with the teacher orbit is slightly larger than the others. During an initial microscopic phase, the advantage of this leading correlation over the competing correlations is amplified. In a subsequent macroscopic phase, the selected correlation grows to a macroscopic level while the competitors remain microscopic. For the pure quadratic Hermite link, individual orbit directions are not identifiable. Rather than tracking a single leading correlation, we track the squared norm of the projection onto the teacher subspace, which undergoes multiplicative growth.

The three procedures differ in how symmetry affects these dynamics. Weight sharing ties all transformed copies of the feature to a single trainable parameter, whereas data augmentation averages the stochastic gradient over the orbit. Although augmentation leaves the population objective unchanged, orbit averaging reduces stochastic fluctuations. With the corresponding learning-rate rescaling, the tied and augmented procedures have the same effective correlation dynamics. In the symmetry-agnostic untied procedure, the corresponding effective dynamics are slower by a factor $r$, yielding the $r$-fold separation in both recovery regimes. See Subsection~\ref{sec:overview_proof} for a fuller proof overview.

\section{Setting}

\subsection{Teacher model}
Let \(d\ge1\), and let \(r\ge1\) be an integer dividing \(d\). Define \(\Delta \coloneqq d / r\), and let \(\Pi\in \R^{d\times d}\) denote the permutation matrix corresponding to a cyclic shift by \(\Delta\) coordinates. Thus \(\Pi^r=\Id_d\), and the cyclic group generated by \(\Pi\) is
\[
\mathcal{G} \coloneqq \left \{\Id_d,\Pi,\ldots,\Pi^{r-1} \right \}.
\]

We observe a sequence of samples \((x_i,y_i)_{i=1}^n\), where \(x_i \overset{\mathrm{i.i.d.}}{\sim}\mathcal N(0,\Id_d)\) and \(y_i = f^\ast(x_i)\), where
\begin{equation}\label{eq:teacher}
f^\ast(x) \coloneqq \frac1{\sqrt r} \sum_{k=0}^{r-1} \sigma \left(\langle \Pi^kw^\ast,x\rangle\right).
\end{equation}
Here \(w^\ast\in\mathbb{S}^{d-1}\) is the teacher direction and \(\sigma \colon \R \to \R\) is a known nonlinear link function. The teacher depends on \(w^\ast\) only through its orbit \(\{\Pi^k w^\ast \colon 0\le k<r\}\). Thus the statistical model is invariant under \(w^\ast\mapsto \Pi^k w^\ast\).

We analyze two regimes for the link function.
\begin{assumption}[Link function]\label{hyp:link}
The link \(\sigma\) belongs to one of the following two classes.
\begin{itemize}[leftmargin=1.5em, itemsep=0pt, topsep=2pt]
    \item[(i)] \textbf{Polynomial links with information exponent \(p\ge3\).}
    The link is a fixed polynomial with finite Hermite expansion
    \[
    \sigma(z)=\sum_{j=p}^P a_jh_j(z), \qquad 3\le p\le P<\infty, \qquad a_p\neq0,
    \]
    where \((h_j)_{j\ge0}\) is the orthonormal Hermite basis of \(L^2(\mathcal N(0,1))\). The integers \(p,P\) and the coefficients \((a_j)_{j=p}^P\) are fixed independently of \(d\). If \(p\) is even, we also assume that \(\sigma\) is even.
    \item[(ii)] \textbf{Pure quadratic link.}
    \[
    \sigma(z)=h_2(z)=\frac1{\sqrt2}(z^2-1).
    \]
\end{itemize}
\end{assumption}

In case of Assumption~\ref{hyp:link}(i), we call $p = p_\sigma \coloneqq \min \{j\ge1 \colon a_j\neq0\}$ the \emph{information exponent} of the link. Since \(p\ge3\), the Hermite coefficients of degrees \(1,\ldots,p-1\) vanish.

\begin{remark}[On the link assumptions]
The finite Hermite expansion in Assumption~\ref{hyp:link}(i) is mainly technical and could be relaxed under suitable regularity and moment-growth conditions on the link. The parity condition for $p$ ensures that the symmetry of the leading-order dynamics agrees with the exact symmetry of the statistical model. This condition can be removed by considering paired opposite initializations; we do not establish this extension here.
\end{remark}

For \(u \in\mathbb{S}^{d-1}\), define the correlations
\begin{equation} \label{eq:correlation}
m_k (u) \coloneqq \left\langle \Pi^kw^\ast,u \right\rangle, \qquad 0\le k<r.
\end{equation}
The teacher direction \(w^\ast\) is identifiable only up to cyclic shifts. When \(\sigma\) is even, there is an additional global sign ambiguity. We therefore define the \emph{identifiable teacher orbit} by
\[
\mathcal{O}_\sigma(w^\ast) \coloneqq
\begin{cases}
\{\pm\Pi^kw^\ast \colon 0\le k<r\}, & \sigma\text{ is even},\\[1mm]
\{\Pi^kw^\ast \colon 0\le k<r\}, & \text{otherwise}.
\end{cases}
\]
For \(u\in\mathbb{S}^{d-1}\), define its correlation with the identifiable teacher orbit by
\begin{equation} \label{eq:orbit_corr}
M_\sigma (u) \coloneqq \max_{v\in\mathcal O_\sigma(w^\ast)} \langle v,u\rangle.
\end{equation}
The quadratic case in Assumption~\ref{hyp:link}(ii) is qualitatively different: the individual directions in the teacher orbit are generally not identifiable. Accordingly, weak recovery in this case will instead be formulated in terms of the subspace spanned by the teacher orbit (see Definition~\ref{def:weak_subspace_recovery}).

\subsection{Learning procedures}
We compare three learning procedures: a tied convolutional student, an untied student, and the same untied student trained with data augmentation. Throughout, let
\[
\ell(y,z)\coloneqq-yz
\]
denote the correlation loss.

\begin{definition}[Tied convolutional student] \label{def:tied_student}
The tied student is parametrized by the weight \(w \in \mathbb{S}^{d-1}\) and is defined by
\[
f_w^\mathsf{T}  (x) \coloneqq \frac{1}{\sqrt{r}} \sum_{k=0}^{r-1} \sigma \left( \langle \Pi^k w, x \rangle\right).
\]
Its sample loss is
\[
L^\mathsf{T}  (w; x,y) \coloneqq \ell \left (y, f_w^\mathsf{T} (x) \right).
\]
\end{definition}

\begin{definition}[Untied student] \label{def:untied_student}
Let \(W = (w^1,\ldots, w^s) \in (\mathbb{S}^{d-1})^s\) with \(s = s(d) \ge 1\). The untied student is parametrized by \(W\) and defined by
\[
f_W^\mathsf{U} (x) \coloneqq \frac{1}{\sqrt{s}} \sum_{j=1}^s \sigma \left( \langle w^j, x \rangle\right).
\]
Its sample loss is
\[
L^\mathsf{U} (W; x,y) \coloneqq \ell \left (y, f_W^\mathsf{U} (x) \right).
\]
\end{definition}

Unlike the tied student,  \(w^1,\ldots, w^s\) are not constrained to belong to a common orbit under \(\mathcal{G}\).

\begin{definition}[Untied student with data augmentation] \label{def:untied_augmented_student}
The augmented procedure uses the same model class \(f_W^\mathsf{U}\), but replaces the sample loss by its group average:
\[
L^\mathsf{A} (W; x,y) \coloneqq \frac{1}{r} \sum_{k=0}^{r-1}  \ell \left (y, f_W^U  (\Pi^k x) \right).
\]
\end{definition}

Set \(s_{\mathsf{T}}\coloneqq 1\) and \(s_{\mathsf{U}}= s_{\mathsf{A}} \coloneqq s \in \N\), with \(s \ge 1\). For \(\mathsf k\in\{\mathsf{T},\mathsf{U},\mathsf{A}\}\), define \(\mathcal{M}_{\mathsf{k}} \coloneqq (\mathbb{S}^{d-1})^{s_{\mathsf{k}}}\), so that
\[
\mathcal{M}_{\mathsf{T}}=\mathbb S^{d-1}, \qquad \mathcal{M}_{\mathsf{U}} = \mathcal{M}_{\mathsf{A}} = (\mathbb S^{d-1})^s.
\]
For \(\theta \in \mathcal{M}_{\mathsf{k}}\), let
\begin{equation} \label{eq:population_loss}
\Phi^{\mathsf k}(\theta) \coloneqq \E \left[L^{\mathsf k}(\theta;x,y)\right]
\end{equation}
denote the corresponding population objective. Since the Gaussian input distribution is invariant under \(\Pi\) and the teacher satisfies \(f^\ast(\Pi^kx)=f^\ast(x)\), the augmented and unaugmented untied procedures have the same population objective:
\[
\Phi^{\mathsf{A}}(W)=\Phi^{\mathsf{U}}(W).
\]
Thus data augmentation leaves the population drift unchanged.

\paragraph{Online spherical SGD}
For \(\theta=(w^1,\ldots,w^{s_{\mathsf{k}}}) \in \mathcal{M}_{\mathsf{k}}\), write
\[
\nabla_{w^j} L(\theta) \coloneqq  (\Id_d -w^j w^{j\top}) \widehat\nabla_{w^j} L(\theta)
\]
for the spherical gradient, where \(\widehat\nabla_{w^j}\) denotes the Euclidean gradient.
Given an initial condition \(\theta_0\in\mathcal{M}_{\mathsf k}\) and a step size \(\eta_d^{\mathsf k}>0\), single-pass spherical SGD is defined componentwise by
\begin{equation} \label{eq:SGD}
\begin{cases}
\widetilde{w}_t^j = w_{t-1}^j - \eta_d^{\mathsf{k}} \nabla_{w^j} L^{\mathsf{k}}(\theta_{t-1};x_t,y_t),\\
w_t^j = \frac{\widetilde{w}_t^j}{\|\widetilde{w}_t^j\|},
\end{cases}
\qquad 1\le j\le s_{\mathsf{k}}.
\end{equation}
For the tied procedure we identify \(\theta_t=w_t^{1,\mathsf{T}}\) with \(w_t\). For the untied and augmented procedures, \(\theta_t=W_t=(w_t^1,\ldots,w_t^s)\). At each iteration, the algorithm uses one new independent observation \((x_t,y_t)\). In the augmented procedure, this same observation additionally generates the \(r\) transformed inputs \(\Pi^kx_t\) for \(0\le k<r\). Hence the number of SGD iterations equals the number of fresh samples, although augmentation incurs a larger per-iteration computational cost.

\section{Main results}

We now state our main results on weak recovery by online spherical SGD. We work on a probability space \((\Omega,\mathcal F,\mathbb P)\) supporting the teacher direction \(w^\ast\), all parameter initializations, and an i.i.d.\ sequence of Gaussian covariates \((x_i)_{i \ge 1}\) used by online SGD.

\begin{assumption} [Random teacher and initialization] \label{hyp:teacher_init}
The teacher direction is sampled according to $w^\ast \sim \operatorname{Unif}(\mathbb{S}^{d-1})$. For each learning procedure \(\mathsf{k} \in\{\mathsf{T},\mathsf{U},\mathsf{A}\}\), initialize
\[
w_{0}^{1,\mathsf k},\ldots,w_{0}^{s_{\mathsf k},\mathsf k} \stackrel{\mathrm{i.i.d.}}{\sim} \operatorname{Unif}(\mathbb S^{d-1}),
\]
where \(s_{\mathsf{T}}=1\) and \(s_{\mathsf{U}}=s_{\mathsf{A}}=s\). The teacher, all parameter initializations, and the Gaussian covariates \((x_i)_{i\ge1}\) are mutually independent.
\end{assumption}
Unless stated otherwise, \(\mathbb P\) denotes probability with respect to all of these sources of randomness. We work throughout in the following high-dimensional regime.

\begin{assumption}[High-dimensional regime] \label{hyp:scaling}
Fix \(0<\delta<\frac12\) and \(\kappa \ge 0\). We work along a sequence of dimensions \(d\to\infty\) for which
\[
r=r(d)=d^\delta, \qquad 1\le s=s(d)\le d^\kappa,
\]
with \(r(d)\) an integer divisor of \(d\).
\end{assumption}

\subsection{Weak directional recovery (\texorpdfstring{\(p\ge3\)}{p >= 3})}
We first consider Assumption~\ref{hyp:link}(i), namely polynomial links with information exponent \(p \ge 3\). For \(\rho\in(0,1)\) and \(\mathsf k\in\{\mathsf{T}, \mathsf{U},\mathsf{A}\}\), define the event
\[
\mathcal R_n^{\mathsf k}(\rho) \coloneqq \left\{ \max_{0\le t\le n} \max_{1\le j\le s_\mathsf{k}} M_\sigma(w_t^{j,\mathsf k}) \ge \rho\right\},
\]
where \(s_{\mathsf{T}}=1\), \(s_{\mathsf{U}}=s_{\mathsf{A}}=s\), and
\(w_t^{1,\mathsf{T}}=w_t^{\mathsf{T}}\).

\begin{definition}[Weak recovery]
\label{def:weak_recovery}
Fix \(\rho\in(0,1)\), independently of \(d\). We say that procedure
\(\mathsf k\in\{\mathsf{T},\mathsf{U},\mathsf{A}\}\) achieves
\emph{weak recovery at level \(\rho\) within \(n=n(d)\) samples} if
\[
\mathbb P\bigl(\mathcal R_n^{\mathsf k}(\rho)\bigr) \to 1 \qquad\text{as }d\to\infty.
\]
\end{definition}
Thus, weak recovery requires the tied procedure's single weight, or at least one weight of the untied or augmented procedure, to attain a fixed non-vanishing correlation with the identifiable teacher orbit. When \(p\) is even, recovery is only up to cyclic shift and global sign, since both are exact symmetries of the model. When \(p\) is odd, only the cyclic shift remains unidentifiable, and recovery requires positive correlation with the teacher orbit.

For \(\mathsf k\in\{\mathsf{T},\mathsf{U},\mathsf{A}\}\), let \(n_{\mathsf k}(\rho)\) denote the sample complexity of achieving weak recovery at level \(\rho\) in the sense of Definition~\ref{def:weak_recovery}. Throughout, \(\widetilde\Theta\) denotes matching upper and lower sample complexity bounds up to logarithmic factors in \(d\). The upper bounds give recovery with probability \(1-o(1)\) by the stated scale, whereas below the lower-bound scale, recovery does not occur with probability \(1-o(1)\). Our main results identify the sample complexity for the three procedures.

\begin{theorem}[Sample complexity of weak recovery, \(p\ge3\)]
\label{thm:weak_recovery}
Suppose that Assumptions~\ref{hyp:link}(i),~\ref{hyp:teacher_init}, and~\ref{hyp:scaling} hold. Fix \(\rho\in(0,1)\), independently of \(d\). Then there exists \(A_\ast=A_\ast(p,P)>0\) such that, for every fixed \(A>A_\ast\), the following holds. Set 
\[
\eta_d=d^{-p/2}(\log d)^{-A},
\]
and run the three online spherical SGD procedures \eqref{eq:SGD} with step sizes
\begin{equation}
\label{eq:student_stepsizes}
\eta_d^{\mathsf{T}}=\eta_d, \qquad \eta_d^{\mathsf{U}}=\sqrt{\frac{s}{r}}\,\eta_d, \qquad \eta_d^{\mathsf{A}}=\sqrt{rs}\,\eta_d.
\end{equation}
Then the sample complexities for weak recovery at level \(\rho\) satisfy
\[
n_{\mathsf{T}}(\rho),\;n_{\mathsf{A}}(\rho) = \widetilde\Theta\!\left(\frac{d^{(p-2)/2}}{\eta_d}\right) = \widetilde\Theta(d^{p-1}), \qquad n_{\mathsf{U}}(\rho) = \widetilde\Theta\!\left(\frac{r\,d^{(p-2)/2}}{\eta_d}\right) = \widetilde\Theta(rd^{p-1}).
\]
\end{theorem}
Thus, up to logarithmic factors, weight tying and exact data augmentation achieve weak recovery at the same $d^{p-1}$ scale as vanilla Gaussian single-index online SGD~\citep{benarous2021online, damian2023smoothing}, while the unaugmented untied procedure requires an additional factor \(r\). The upper bounds in Theorem~\ref{thm:weak_recovery} are stronger than Definition~\ref{def:weak_recovery} requires: any fixed weight of the untied or augmented procedure achieves weak recovery at the corresponding sample scale, so the proof does not rely on selecting the best among \(s\) initializations. Conversely, the lower bounds hold uniformly over all \(s\) weights. The proof combines the upper bounds of Theorem~\ref{thm:weak_recovery_appendix} with the lower bounds of Theorem~\ref{thm:nonescape}; see Subsection~\ref{sec:main_results_proofs}.

Figure~\ref{fig:median-recovery-h4} illustrates the predicted separation empirically for the pure Hermite-\(4\) link at \(\delta=1/2\). Since the upper bounds hold for any fixed weight, we set \(s=1\) and use the corresponding rescaled learning rates with a common initialization across procedures. This isolates the single-weight dynamics and controls for differences in initial alignment that can strongly affect recovery times at small dimensions but do not change the asymptotic sample complexity scales established above.

\begin{remark}[Choice of learning rates]
The reference rate \(\eta_d=d^{-p/2}(\log d)^{-A}\) has the largest polynomial order supported by our perturbative analysis. Up to logarithmic factors, the accumulated stochastic and normalization errors, relative to the relevant growth scale, are controlled by \((\eta_d d^{p/2})^{1/2}\) and \(\eta_d d^{p/2}\), respectively. Choosing the constant \(A\) sufficiently large makes both contributions negligible and allows comparison with the population dynamics.
\end{remark}

\paragraph{Orbit coverage}

For the tied student, recovery of one orbit direction automatically represents the entire orbit through weight sharing. For the untied and augmented students, weak recovery guarantees only that one weight aligns with one orbit direction. We therefore ask the stronger question of whether the collection of \(s\) weights can recover every orbit direction. Define the relevant correlation
\begin{equation} \label{eq:relevant_correlation}
q_k(w) \coloneqq
\begin{cases}
|m_k(w)|, & p\ \text{even},\\
m_k(w), & p\ \text{odd}.
\end{cases}
\end{equation}
For \(\mathsf k\in\{\mathsf{U},\mathsf{A}\}\), define the full-orbit coverage event
\[
\mathcal C_n^{\mathsf k}(\rho) \coloneqq \left\{\min_{0\le k<r} \max_{1\le j\le s_\mathsf{k}} \max_{0\le t\le n} q_k(w_t^{j,\mathsf k}) \ge\rho \right\}.
\]
\begin{definition}[Weak orbit coverage] \label{def:weak_orbit_coverage}
We say that procedure \(\mathsf k \in \{\mathsf{U}, \mathsf{A}\}\) \emph{weakly covers the full teacher orbit at level \(\rho\) within \(n=n(d)\) samples} if
\[
\mathbb P\bigl(\mathcal C_n^{\mathsf k}(\rho)\bigr) \to 1.
\]
Write \(n_{\mathsf k}^{\mathrm{cov}}(\rho)\) for the corresponding sample complexity scale.
\end{definition}

Equivalently, Definition~\ref{def:weak_orbit_coverage} requires every identifiable teacher-orbit direction to be recovered by at least one weight at some time before \(n\). With sufficiently large width \(s\), this stronger requirement can be achieved without changing the sample complexity scale.

\begin{corollary}[Weak orbit coverage] \label{cor:weak_orbit_coverage}
Under the assumptions and learning rates of Theorem~\ref{thm:weak_recovery}, fix \(\varepsilon>0\) and suppose \(s\ge(1+\varepsilon)r\log r\). Then
\[
n_{\mathsf{A}}^{\mathrm{cov}}(\rho) = \widetilde\Theta(d^{p-1}), \qquad n_{\mathsf{U}}^{\mathrm{cov}}(\rho) = \widetilde\Theta(rd^{p-1}).
\]
\end{corollary}

Thus, once the width is sufficient to cover the \(r\) orbit directions, requiring recovery of the entire orbit does not increase the sample complexity scale relative to weak recovery of a single direction. The proof is given in Section~\ref{sec:coverage_proof}.

\begin{figure}[!htbp]
\centering
\begin{minipage}[t]{0.49\textwidth}
  \vspace{0pt}
  \centering
  \captionsetup{type=figure,font=small,justification=justified,singlelinecheck=false,skip=4pt}
  \includegraphics[width=\linewidth]{\paperroot 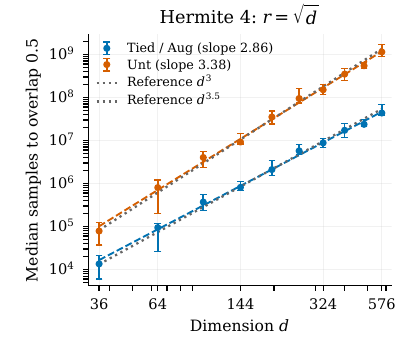}
  \caption{\textbf{Recovery-time scaling.}
Spherical online SGD with correlation loss, Hermite-\(4\) link,
and \(r=\sqrt d\) symmetries.
Points show median samples to reach \(M_\sigma(w)\ge0.5\)
over 10 paired seeds; bars are pointwise 95\% bootstrap intervals.
Colored dashed lines show power-law fits; gray dotted lines show reference powers.
We report equivalent single-weight dynamics with
\(s_{\mathsf T}=s_{\mathsf U}=s_{\mathsf A}=1\),
shared initialization and samples, \(\eta_d=0.1d^{-2}\),
and the rescalings in~\eqref{eq:student_stepsizes}.
Tied and augmented trajectories coincide.}
  \label{fig:median-recovery-h4}
\end{minipage}\hfill\begin{minipage}[t]{0.49\textwidth}
  \vspace{0pt}
  \centering
\captionsetup{type=figure,font=small,justification=justified,singlelinecheck=false,skip=4pt}
\includegraphics[width=\linewidth]{\paperroot 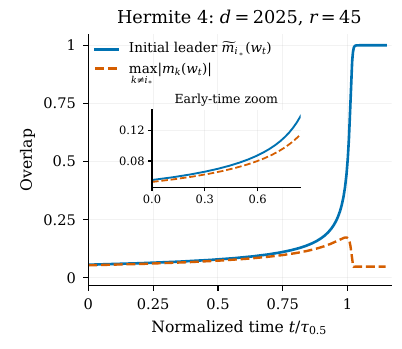}
\caption{\textbf{Selection of an orbit direction in population flow.} Population gradient flow with correlation loss for the tied student with the Hermite-4 link, dimension $d=2025$, and cyclic symmetries $r=\sqrt d=45$. Solid blue follows the correlation $\widetilde m_{i_\ast}(w_t) \coloneqq \operatorname{sgn} (m_{i_\ast} (w_0)) m_{i_\ast} (w_t)$ with index $i_\ast=\arg\max_k|m_k(w_0)|$. Dashed orange shows $\max_{k\ne i_\ast}|m_k(w_t)|$.
The inset magnifies the early gap amplification.
Time is normalized by $\tau_{0.5}$, the first time the population flow
reaches orbit correlation $0.5$.}
\label{fig:population-dynamics-h4}
\end{minipage}
\end{figure}

\subsection{Weak subspace recovery (\texorpdfstring{\(p=2\)}{p = 2})}
We next consider Assumption~\ref{hyp:link}(ii), namely the pure quadratic link. In this case the individual directions in the teacher orbit are generally not identifiable, and the natural recovery target is instead the subspace
\[
\mathcal S^\ast \coloneqq \operatorname{span}\{\Pi^k w^\ast \colon 0\le k<r\}.
\]
For \(w\in\mathbb S^{d-1}\), let \(w^\parallel\) denote its orthogonal projection onto \(\mathcal S^\ast\). Since \(\|w\|=1\), the quantity \(\|w^\parallel\|^2\) measures the fraction of the squared norm of \(w\) contained in the teacher subspace.

\begin{definition}[Weak subspace recovery]
\label{def:weak_subspace_recovery}
Fix \(\rho\in(0,1)\), independently of \(d\). We say that procedure \(\mathsf k\in\{\mathsf{T},\mathsf{U},\mathsf{A}\}\) achieves \emph{weak subspace recovery at level \(\rho\) within \(n=n(d)\) samples} if
\[
\mathbb{P} \left( \max_{0\le t\le n} \max_{1\le j\le s_{\mathsf k}} \left \|(w_t^{j,\mathsf k})^\parallel \right \|^2 \ge \rho \right) \to 1 \qquad\text{as }d\to\infty,
\]
where \(s_{\mathsf{T}}=1\), \(s_{\mathsf{U}}=s_{\mathsf{A}}=s\), and \(w_t^{1,\mathsf{T}}=w_t^{\mathsf{T}}\).
\end{definition}

Let \(n_{\mathsf k}^{(2)}(\rho)\) denote the sample complexity of weak subspace recovery, with upper and lower bounds interpreted as for \(n_{\mathsf k}(\rho)\).

\begin{theorem}[Sample complexity of weak subspace recovery]
\label{thm:h2_weak_recovery}
Suppose Assumption~\ref{hyp:link}(ii),~\ref{hyp:teacher_init}, and~\ref{hyp:scaling} hold. Fix \(\rho\in(0,1)\), independently of \(d\). Then there exists a universal constant \(A_0>0\) such that, for every fixed \(A>A_0\), the following holds. Set
\begin{equation}
\label{eq:h2_learning_rate}
\eta_d=d^{-1}(\log d)^{-A},
\end{equation}
and use the same relative learning-rate scaling across the three procedures as in~\eqref{eq:student_stepsizes}. Then the sample complexities for weak subspace recovery at level \(\rho\) satisfy
\[
\begin{aligned}
n_{\mathsf{T}}^{(2)}(\rho),\; n_{\mathsf{A}}^{(2)}(\rho) = \widetilde\Theta\!\left(\frac{1}{\eta_d}\right) = \widetilde\Theta(d), \qquad n_{\mathsf{U}}^{(2)}(\rho) = \widetilde\Theta\!\left(\frac{r}{\eta_d}\right) = \widetilde\Theta(rd).
\end{aligned}
\]
\end{theorem}

Thus, the pure quadratic case exhibits the same \(r\)-fold separation between symmetry-aware and symmetry-agnostic learning as for \(p\ge3\): weight tying and data augmentation achieve weak subspace recovery at scale \(\widetilde\Theta(d)\), whereas the untied procedure requires \(\widetilde\Theta(rd)\) samples. The proof is provided in Subsection~\ref{sec:main_results_proofs}, combining the upper bounds of Theorem~\ref{thm:h2_weak_recovery_appendix} with the lower bounds of Theorem~\ref{thm:nonescape}.

\begin{remark}[General links with information exponent \(2\)]
For general centered polynomial links with information exponent \(2\), our lower bound still applies, but our matching upper bound does not: we prove the upper bound only for the pure quadratic link. Higher-order Hermite components can change the appropriate recovery notion and the resulting dynamics. Related results for untied models have been obtained by~\citet{ren2025learning}.
\end{remark}

\section{Related work}
\label{sec:related_work}

Our work connects to several strands of a broad and rapidly developing literature, which we review below.

\paragraph{Learning low-dimensional representations}
Multi-index models provide a tractable setting for studying how learning algorithms discover low-dimensional structure in high-dimensional data. \citet{chen2020learning} developed algorithms for Gaussian polynomial multi-index models using filtered PCA and gradient-based refinement. Neural-network representation learning has been established under second-order nondegeneracy conditions~\citep{damian2022neural}, while \citet{mousavihosseini2023neural} showed that SGD with weight decay suppresses weight components orthogonal to the teacher subspace. These approaches investigate how information about the index subspace can be extracted through moment-based methods or through the evolution of network weights.

\paragraph{Staircase structure and leap complexity}
Hierarchical structure can allow previously learned features to facilitate the discovery of new ones. For Boolean inputs, \citet{abbe2021staircase} introduced the staircase property and established hierarchical learning guarantees under a layer-wise training procedure. For Boolean functions depending on a fixed number of coordinates, \citet{abbe2022merged} identified the merged-staircase property as necessary and nearly sufficient for learning with $O(d)$ samples in the two-layer mean-field setting. \citet{abbe2023sgd} extended this perspective through leap complexity, which quantifies how many new coordinates must be learned together, counting Hermite multiplicities for Gaussian inputs. They proved learning guarantees for a structured class of Gaussian targets under a modified SGD procedure, showing how the dynamics move between successive saddles as additional coordinates are recovered.

\paragraph{Gradient dynamics and subspace discovery}
For Gaussian single-index models, \citet{benarous2021online} identified the information exponent governing the initial search phase of online SGD. \citet{ba2022high} quantified representation learning after one gradient step, highlighting the role of learning-rate scaling. For Gaussian multi-index models, \citet{dandi2024two} showed that successive gradient steps with fresh batches of $O(d)$ samples can reveal additional directions through a staircase mechanism. \citet{bietti2025learning} established global convergence and characterized successive learning phases for a two-timescale population gradient flow, with timescales determined by the target's Hermite decomposition. Related hierarchical subspace learning and weak recovery thresholds were analyzed through approximate message passing by \citet{troiani2025fundamental}. For orthogonal additive teachers, \citet{simsek2025learning} studied directional selection under correlation-loss gradient flow. Competition between signal directions also appears in multi-spiked tensor PCA, where~\citet{benarous2026} characterize a greedy maximum selection mechanism leading to the sequential recovery of individual spikes.

\paragraph{Dependence on the learning procedure}
Information exponent rates depend on the updates being analyzed. Loss smoothing~\citep{damian2023smoothing}, batch reuse~\citep{dandi2024benefits}, and alternating layer-wise updates with suitable learning-rate scales~\citep{tsiolis2025information} can improve upon the rates of the corresponding vanilla dynamics. Batch size and the loss also affect online learning times~\citep{arnaboldi2024online}. Complementary algorithmic results introduce generative exponents and generative leap complexity for fixed-rank Gaussian models, accounting for information accessible through nonlinear label processing and sequential subspace estimation~\citep{damian2024computational,damian2025generative}. These works establish lower bounds within statistical-query or low-degree frameworks and develop algorithms that attain the corresponding scales under their assumptions.

\paragraph{Growing rank and subspace recovery}
When the number of teacher directions grows with dimension, the rank dependence of learning guarantees becomes essential. \citet{oko2024diverse} analyzed learning sums of nearly orthogonal ridge functions with growing rank, and \citet{ren2025emergence} related individual feature recovery times to prediction-risk scaling laws. In particular, the sample complexity of the untied learner in Theorem~\ref{thm:weak_recovery} has the same leading dependence on \(r\) and \(d\) as the upper bound of \cite{oko2024diverse}, where they consider learning an untied teacher instead of a symmetry-tied one as in the present work. For quadratic targets, \citet{benarous2025learning} studied growing-rank networks trained by orthogonality-constrained online SGD. Subspace and directional recovery can also occur at different stages: \mbox{\citet{ren2025learning}} showed how second-order information can reveal the index subspace before higher-order components distinguish individual teacher directions.

\paragraph{Weight sharing and data augmentation}
Group-equivariant architectures~\citep{cohen2016group,kondor2018generalization} and group-averaged training~\citep{chen2020group,lyle2020benefits} provide different ways to exploit symmetry. Statistical benefits have been quantified for equivariant, random-feature, and kernel models~\citep{elesedy2021provably,mei2021learning,cagnetta2023wide}. Beyond these settings, \citet{lahoti2024role} established sample complexity separations between convolutional, locally connected, and fully connected networks in a structured image model, distinguishing the benefits of locality and weight sharing.\\

Our work studies how a known symmetry changes online feature learning when the teacher rank grows with dimension. We compare weak directional and subspace recovery under weight sharing, full-group augmentation, and untied training without augmentation. The matching lower bounds are non-escape bounds for the specified spherical SGD updates and learning rates; they do not assert optimality over alternative learning algorithms.

\section{Proof outline}
\label{sec:appendix_roadmap}

We first explain the mechanisms underlying our recovery results, then state the precise upper and lower bounds and combine them to prove Theorems~\ref{thm:weak_recovery} and~\ref{thm:h2_weak_recovery}. Throughout, the three procedures use the learning-rate scaling in~\eqref{eq:student_stepsizes}, expressed in terms of a reference rate \(\eta_d\). Its value is specified in each theorem below.

\subsection{Proof overview}
\label{sec:overview_proof}
The proofs of our main results are based on tracking the correlations \(m_k(w_t)\) defined in~\eqref{eq:correlation} between the SGD iterates and the teacher orbit. For \(p\ge3\), we track the relevant correlations \(q_k(w_t)\) from~\eqref{eq:relevant_correlation}, so that
\[
M_\sigma(w_t)=\max_{0\le k<r}q_k(w_t).
\]
Recall that \(q_k=m_k\) when \(p\) is odd and \(q_k=|m_k|\) when \(p\) is even. We show that random initialization produces, with high probability, a unique leading relevant correlation separated from its competitors by a gap large enough to dominate the error, and that SGD dynamics preserves and amplifies this gap.

\paragraph{Initialization, interactions, and correlation growth}
At random initialization, each correlation \(m_k(w_0)\) is typically of order \(d^{-1/2}\). Since the teacher orbit is nearly orthogonal, the correlation vector is approximately Gaussian with covariance close to \(d^{-1} \Id_r\). In particular, its coordinates behave approximately as independent \(N(0,d^{-1})\) variables, and their largest relevant value \(q_k(w_0)\) is of order
\[
\gamma_d\coloneqq\sqrt{\frac{\log r}{d}}.
\]
Moreover, with high probability, there is a unique leading relevant correlation, separated from all competing correlations by a quantifiable gap. This initial gap breaks the symmetry between the teacher directions and is then amplified by the SGD dynamics. Hermite orthogonality gives
\[
\Phi^{\mathsf{T}} (w) = -\sum_{k=0}^{r-1}\phi(m_k(w)), \qquad \phi(z)=\sum_{\ell=p}^{P}a_\ell^2z^\ell,
\]
with \(\phi'(m) = p a_p^2 m^{p-1} + O(|m|^p)\) near the origin. Writing \(m_k=m_k(w)\), the spherical population drift of the \(k\)th correlation is
\[
-\langle\Pi^kw^\ast,\nabla\Phi^{\mathsf{T}}(w)\rangle = (1-m_k^2)\phi'(m_k) + \sum_{i\ne k} \langle\Pi^kw^\ast,\Pi^iw^\ast\rangle\phi'(m_i)-m_k\sum_{i\ne k}m_i\phi'(m_i).
\]
The first term gives the self-amplifying growth of \(m_k\), while the remaining terms describe the interactions with the other teacher directions and the effect of the spherical constraint.

The main technical difficulty is to control these interaction terms uniformly when the number \(r\) of teacher directions grows with \(d\). Together with the stochastic fluctuations and normalization errors, we show that their cumulative effect remains smaller than the initialization gap throughout the microscopic phase. As a result, the relevant correlation that is largest at initialization remains the leader: no competitor overtakes it, and its advantage grows over time. The dynamics can therefore be reduced to an effectively 1D growth mechanism along the selected teacher direction (see Fig.~\ref{fig:population-dynamics-h4}).

\begin{wrapfigure}{R}{0.45\textwidth}
\vspace{-0.6cm}
\centering
\captionsetup{font=small,justification=justified,singlelinecheck=false,skip=6pt}
\includegraphics[width=\linewidth]{\paperroot 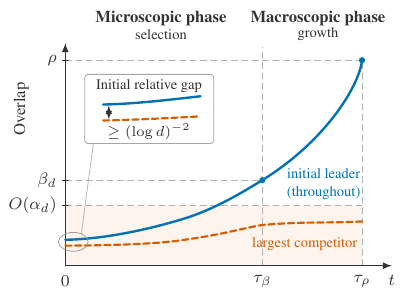}
\caption{\textbf{Proof sketch: two phases of \mbox{recovery.}} Blue follows the initial leader; orange shows the largest competitor, whose index may change. Their initial gap widens during both phases. The shaded band bounds competing absolute overlaps. Axes and trajectories are schematic; the macroscopic phase is expanded for visibility.}
\label{fig:proof-phases}
\vspace{-0.9cm}
\end{wrapfigure}

\paragraph{From microscopic selection to weak recovery}
Introduce the intermediate scales \(\alpha_d,\beta_d\) satisfying
\[
\gamma_d\ll\alpha_d\ll\beta_d\ll1.
\]
During the first, microscopic phase, all correlations grow but at different rates; thus, the initial gap between the leader and its competitors widens. As a result, the correlation that is largest at initialization reaches the threshold \(\beta_d\), while all competing correlations remain at scale \(O(\alpha_d)\). This phase takes \(O((\eta_d \gamma_d^{p-2})^{-1})\) steps. During the second phase, the population drift drives the selected correlation from \(\beta_d\) to the weak recovery level \(\rho\in(0,1)\), while the spherical dynamics keep the competing correlations small. This phase is faster and takes \(O ((\eta_d\beta_d^{p-2})^{-1})\) steps. Figure~\ref{fig:proof-phases} illustrates the sketch of proof with this two-phase argument and Section~\ref{sec:weak_recovery_tied} gives the proof for the tied procedure.

\paragraph{Untied and augmented procedures}
For the untied and augmented learners, it is enough for the upper bound to follow one fixed weight \(w^j\). Under the correlation loss, the population objective is additive across weights, so the gradient with respect to \(w^j\) depends only on \(w^j\):
\[
\nabla_{w^j}\Phi^{\mathsf{U}}(W) = \nabla_{w^j}\Phi^{\mathsf{A}}(W) = \frac{1}{\sqrt{rs}}\nabla\Phi^{\mathsf{T}}(w^j).
\]
Hence, the same correlation argument applies weight by weight. At the prescribed learning rates, an augmented weight has the same effective drift and fluctuation scales as the tied trajectory. For an untied weight, the effective drift is smaller by a factor \(r\), while the stochastic fluctuations are smaller only by \(\sqrt r\). The same selection-and-growth mechanism therefore takes \(r\) times as many samples, producing the separation in Theorem~\ref{thm:weak_recovery}. Section~\ref{sec:weak_recovery_untied} makes these reductions precise.

\paragraph{Subspace recovery for the pure quadratic link}
For $\sigma=h_2$, we seek to grow $\|w_t^\parallel\|^2$ to a fixed positive level. Because the teacher orbit is nearly orthonormal, the squared norm of the projection onto its span is, up to a \(1+o(1)\) factor, the sum of the squared correlations:
\[
\|w_t^\parallel\|^2 =(1+o(1))\sum_{k=0}^{r-1}m_k(w_t)^2.
\]
We therefore track this sum, denoted by $R_t$, with $R_0\asymp r/d$. Its population drift produces multiplicative growth. Controlling the accumulated stochastic and normalization
errors shows that $\|w_t^\parallel\|^2$ reaches the weak recovery level $\rho\in(0,1)$ within $O(\eta_d^{-1}\log(d/r))$ samples for the tied and augmented procedures, and $r$ times as many for the untied procedure (see Section~\ref{sec:weak_recovery_quadratic}).

\paragraph{Matching lower bounds}
For every $p\ge2$, our drift and noise estimates imply that, with high probability, all correlations remain $O(\gamma_d)$ for at least $c/(\eta_d\gamma_d^{p-2})$ steps for the tied and augmented procedures, and for $r$ times as many steps under the untied procedure. Here \(c>0\) is sufficiently small. Since $r\gamma_d^2=o(1)$, this excludes both directional and subspace recovery, uniformly over polynomially many weights (see Section~\ref{sec:weak_recovery_lower_bounds}).

\input{\paperroot proofs}

\section{Conclusion}

We established an $r$-fold separation, up to logarithmic factors, in the sample complexity of symmetry-aware and symmetry-agnostic feature learning. For spherical online SGD with correlation loss, this separation holds for weak directional recovery when $p\geq 3$ and weak subspace recovery for the pure quadratic Hermite link. In this setting, full-group data augmentation matches the statistical benefit of weight sharing: exploiting symmetry through either the architecture or the training procedure improves the balance between informative drift and stochastic fluctuations. We conjecture that the same separation persists for $d^{1/2}\lesssim r\ll d$, even though the initial largest correlation no longer remains dominant until weak recovery, and a different proof mechanism is therefore needed.

Several questions remain open. A natural next step is to study general links with information exponent two, extending the analysis of \citet{ren2025learning} to our symmetric setting. Here, quadratic subspace amplification and the directional information carried by higher-order terms may interact, raising the question of whether the same symmetry advantage persists.
It would also be useful to establish strong recovery and vanishing prediction error. Beyond the recovery criterion, other losses and augmentation schemes deserve investigation: squared loss introduces interactions between filters, while using only a subset of group transformations raises a tradeoff between sample efficiency and computation.

\section*{Acknowledgments}
We acknowledge funding from the Swiss National Science Foundation grants OperaGOST (grant number $200021\_200390$) and DSGIANGO (grant number $225837$). This work was supported by the Simons Collaboration on the Physics of Learning and Neural Computation via the Simons Foundation grant ($\#1257412$) and by the EPFL AI Center PhD Fellowship Program 2026.

\subsection*{AI use statement}

We used generative AI tools to assist with formulating and refining mathematical statements, exploring proof strategies, and drafting and revising parts of the proofs. We also used AI assistance to identify relevant literature and references, improve the organization and exposition of the manuscript, and prepare and revise text, figure captions, and layouts. AI assistance was also used to develop and debug the experimental and analysis code and to prepare plots from the numerical results. The mathematical arguments and exposition were developed through iterative discussion, critical review, and substantive revision by the authors. The authors take responsibility for the final content of this work, including all AI-assisted mathematical claims, proofs, text, and artifacts.

\clearpage
\phantomsection
\addcontentsline{toc}{section}{References}
\bibliography{\paperroot references}
\bibliographystyle{\paperroot references}

\input{\paperroot appendices}

\end{document}

%% file: proofs.tex
\subsection{Upper bounds for directional recovery}

We retain the microscopic correlation scale
\(\gamma_d=\sqrt{\log r/d}\) introduced in
Subsection~\ref{sec:overview_proof}. Recall the orbit correlation \(M_\sigma(w)\) defined in~\eqref{eq:orbit_corr}. For
\(\mathsf k\in\{\mathsf T,\mathsf U,\mathsf A\}\), \(a\in(0,1)\), and \(1\le j\le s_{\mathsf k}\), define the hitting time
\begin{equation}
\label{eq:directional_hitting_time}
\tau_a^{j,\mathsf k} \coloneqq
\inf \left \{ t\ge0 \colon  M_\sigma(w_t^{j,\mathsf k})\ge a \right \},
\end{equation}
with the convention \(\inf\emptyset=\infty\). For the tied trajectory, we write \(w_t^{\mathsf T}= w_t^{1,\mathsf T}\) and \(\tau_a^{\mathsf{T}} = \tau_a^{1, \mathsf{T}}\).

\begin{theorem}[Weak recovery for \(p\ge3\)] \label{thm:weak_recovery_appendix}
Suppose that Assumptions~\ref{hyp:link}(i),~\ref{hyp:teacher_init}, and~\ref{hyp:scaling} hold. Fix \(\rho\in(0,1)\). There exists \(A_\ast=A_\ast(p,P)\ge (p-2)/2\) such that the following holds for every \(A>A_\ast\). Set
\[
\eta_d = d^{-p/2}(\log d)^{-A},
\]
and use the procedure-specific step sizes~\eqref{eq:student_stepsizes}.
\begin{enumerate}
\item[\textup{(T)}]
If tied spherical SGD is run with \(\eta_d^{\mathsf T}\), then there exists \(C_T = C_T (\sigma,\delta,\rho)<\infty\) such that
\[
\mathbb{P} \left( \tau_\rho^{\mathsf T} \le \frac{C_T}{\eta_d \gamma_d^{p-2}} \right) \to 1.
\]
\item[\textup{(U)}] For any fixed \(j\in\{1,\dots,s\}\), if untied spherical SGD is run with \(\eta_d^{\mathsf U}\), then there exists \(C_U=C_U(\sigma,\delta,\rho)<\infty\) such that
\[
\mathbb{P} \left(\tau_\rho^{j,\mathsf U} \le \frac{C_U r}{\eta_d \gamma_d^{p-2}} \right) \to 1.
\]
\item[\textup{(A)}] For any fixed \(j\in\{1,\dots,s\}\), if augmented spherical SGD is run with \(\eta_d^{\mathsf A}\), then there exists \(C_A = C_A (\sigma,\delta,\rho)<\infty\) such that
\[
\mathbb{P} \left( \tau_\rho^{j,\mathsf A} \le \frac{C_A}{\eta_d \gamma_d^{p-2}} \right) \to 1.
\]
\end{enumerate}
\end{theorem}

Thus the tied and augmented procedures achieve weak recovery within
\[
\frac{1}{\eta_d \gamma_d^{p-2}} = O \left ( d^{p-1}\operatorname{polylog}(d) \right)
\]
samples, whereas the untied procedure requires at most \(O \left ( r d^{p-1}\operatorname{polylog}(d) \right )\) samples. Theorem~\ref{thm:weak_recovery_appendix} gives the upper bounds stated in Theorem~\ref{thm:weak_recovery}. The matching lower bounds follow from Theorem~\ref{thm:nonescape} below.

\begin{remark}[Fixed-weight upper bounds]
For the untied and augmented procedures, the upper bound holds for any fixed weight index \(j\). In particular, the proof does not exploit the best among the \(s\) random initializations. The width \(s\) enters only through the learning-rate rescaling in~\eqref{eq:student_stepsizes}.
\end{remark}

\subsection{Upper bounds for quadratic subspace recovery}

Recall the teacher orbit subspace \(\mathcal S^\ast\) and the orthogonal projection \(w^\parallel\) used in Definition~\ref{def:weak_subspace_recovery}. For \(a\in(0,1)\) and \(\mathsf k\in\{\mathsf T,\mathsf U,\mathsf A\}\), define
\begin{equation}
\label{eq:h2_recovery_times}
\tau_a^{(2),\mathsf k} \coloneqq \inf\left \{t\ge0 \colon \max_{1\le j\le s_{\mathsf k}} \|(w_t^{j,\mathsf k})^\parallel \|^2 \ge a\right \},
\end{equation}
with \(\inf\emptyset=\infty\). For the tied trajectory, we write \(\tau_a^{(2)} = \tau_a^{(2),\mathsf T}\).

\begin{theorem}[Weak subspace recovery for the quadratic link] \label{thm:h2_weak_recovery_appendix}
Suppose that Assumptions~\ref{hyp:link}(ii),~\ref{hyp:teacher_init}, and~\ref{hyp:scaling} hold. Fix $\rho\in(0,1)$ independently of $d$. There exists a universal constant $A_0>0$ such that the following holds for every \(A>A_0\). Set
\[
\eta_d = d^{-1}(\log d)^{-A},
\]
and use the procedure-specific step sizes~\eqref{eq:student_stepsizes}.
\begin{enumerate}
\item[\textup{(T)}]
If tied spherical SGD is run with $\eta_d^{\mathsf T}$, then there exists $C_T=C_T(\delta, \rho)<\infty$ such that
\[
\mathbb P\left(\tau_\rho^{(2),\mathsf T} \le \frac{C_T\log(d/r)}{\eta_d}\right)\to1.
\]

\item[\textup{(U)}]
If untied spherical SGD is run with $\eta_d^{\mathsf U}$, then there exists $C_U=C_U(\delta, \rho)<\infty$ such that
\[
\mathbb P\left(\tau_\rho^{(2),\mathsf U} \le \frac{C_U r\log(d/r)}{\eta_d}\right)\to1.
\]

\item[\textup{(A)}]
If augmented spherical SGD is run with $\eta_d^{\mathsf A}$, then there exists $C_A=C_A(\delta, \rho)<\infty$ such that
\[
\mathbb P\left(\tau_\rho^{(2),\mathsf A} \le \frac{C_A\log(d/r)}{\eta_d}\right)\to1.
\]
\end{enumerate}
\end{theorem}

Since \(\log(d/r)=(1-\delta)\log d\), these bounds are \(\widetilde O(d)\) for the tied and augmented procedures, and \(\widetilde O(rd)\) for the untied.

\subsection{Non-escape and lower bounds}

We next give a non-escape estimate that yields the matching lower bounds for both recovery notions above. The result applies to any fixed centered polynomial link with information exponent \(p\ge2\). For $\mathsf k\in\{\mathsf T,\mathsf U,\mathsf A\}$, define
\begin{equation} \label{eq:nonescape_escape_times}
\tau_{\mathrm{esc}}^{\mathsf k} (C) \coloneqq \inf\left\{t\ge0 \colon\max_{1\le j\le s_{\mathsf k}}\max_{0\le\ell<r}|m_\ell(w_t^{j,\mathsf k})| \ge C \gamma_d \right\},
\end{equation}
with \(\inf\emptyset=\infty\). Thus \(\tau_{\mathrm{esc}}^{\mathsf k}\) is the first time that any weight of procedure \(\mathsf k\) develops a correlation larger than a constant multiple of the initialization scale \(\gamma_d\).

\begin{theorem}[Non-escape for the three learning procedures]\label{thm:nonescape}
Suppose that Assumptions~\ref{hyp:teacher_init} and~\ref{hyp:scaling} hold, and let
\[
\sigma(z)=\sum_{j=p}^P a_jh_j(z), \qquad 2\le p\le P<\infty, \qquad a_p\ne0,
\]
where \(P\) and the coefficients are fixed independently of \(d\). Fix \(A>0\), set
\[
\eta_d = d^{-p/2}(\log d)^{-A},
\]
and use the procedure-specific step sizes~\eqref{eq:student_stepsizes}. Then there exist constants \(c,C_{\mathrm{esc}}>0\), depending only on \(\sigma,\delta,\kappa\), such that, with
\[
T_d\coloneqq\frac{c}{\eta_d\gamma_d^{p-2}},
\]
we have
\[
\mathbb P\left(\tau_{\mathrm{esc}}^{\mathsf T} (C_{\mathrm{esc}}) >T_d,\; \tau_{\mathrm{esc}}^{\mathsf A} (C_{\mathrm{esc}})>T_d,\; \tau_{\mathrm{esc}}^{\mathsf U} (C_{\mathrm{esc}})>rT_d\right)\to1.
\]
\end{theorem}

On the event in Theorem~\ref{thm:nonescape}, every weight remains at microscopic correlation with every teacher direction throughout the corresponding time interval. Hence, for \(p\ge3\), weak recovery is impossible on these intervals. For the quadratic link, the same event also rules out subspace weak recovery. Indeed, on the teacher-geometry event of Lemma~\ref{lem:teacher_geometry},
\[
\|(w_t^{j,\mathsf k})^\parallel\|^2 \le \frac{\sum_{\ell=0}^{r-1}m_\ell(w_t^{j,\mathsf k})^2}{1-r\mu_\ast} \le \frac{C_{\mathrm{esc}}^2r\gamma_d^2}{1-r\mu_\ast} = o(1),
\]
uniformly over all weights and times in the corresponding intervals. Since each iteration uses one fresh observation, Theorem~\ref{thm:nonescape} yields sample complexity lower bounds \(\widetilde\Omega(d^{p-1})\) for the tied and augmented procedures and $\widetilde\Omega(rd^{p-1})$ for the untied procedure, for every polynomial link covered by Theorem~\ref{thm:nonescape}.

\subsection{Proofs of the main results}
\label{sec:main_results_proofs}

We combine the upper bounds with the non-escape estimate to prove Theorems~\ref{thm:weak_recovery} and~\ref{thm:h2_weak_recovery}.

\begin{proof}[Proof of Theorem~\ref{thm:weak_recovery}]
Fix \(A>A_\ast\), where \(A_\ast\) is given by Theorem~\ref{thm:weak_recovery_appendix}. For the upper bounds, apply Theorem~\ref{thm:weak_recovery_appendix} with \(j=1\). Since recovery of one weight is sufficient for Definition~\ref{def:weak_recovery}, weak recovery occurs with probability \(1-o(1)\) within $C/(\eta_d\gamma_d^{p-2})$ samples for the tied and augmented procedures, and within $Cr/(\eta_d\gamma_d^{p-2})$ samples for the untied procedure.

For the lower bounds, Theorem~\ref{thm:nonescape} applies under the same assumptions and learning rates. Since $M_\sigma(w)\le\max_{0\le\ell<r}|m_\ell(w)|$ and $C_{\mathrm{esc}}\gamma_d<\rho$ for all sufficiently large $d$, its non-escape event excludes weak recovery by every weight up to times $c/(\eta_d\gamma_d^{p-2})$ for the tied and augmented procedures, and $cr/(\eta_d\gamma_d^{p-2})$ for the untied procedure.

Finally,
\[
\frac{1}{\eta_d\gamma_d^{p-2}} = \frac{d^{(p-2)/2}}{\eta_d(\log r)^{(p-2)/2}} = \widetilde\Theta\!\left(\frac{d^{(p-2)/2}}{\eta_d}\right) = \widetilde\Theta(d^{p-1}),
\]
where we used that $r=d^\delta$ and $\eta_d=d^{-p/2}(\log d)^{-A}$. This proves the claim.
\end{proof}

\begin{proof}[Proof of Theorem~\ref{thm:h2_weak_recovery}]
Fix \(A>A_0\), where \(A_0\) is given by Theorem~\ref{thm:h2_weak_recovery_appendix}. By~\eqref{eq:h2_recovery_times}, the event in Definition~\ref{def:weak_subspace_recovery} is exactly $\{\tau_\rho^{(2),\mathsf k}\le n\}$. Theorem~\ref{thm:h2_weak_recovery_appendix} therefore gives upper bounds of $C\log(d/r)/\eta_d$ samples for the tied and augmented procedures and of $Cr\log(d/r)/\eta_d$ samples for the untied procedure.

For the lower bounds, apply Theorem~\ref{thm:nonescape} with \(\sigma=h_2\) and \(p=2\). Then \(T_d=c/\eta_d\), and the subspace estimate above shows that, with probability \(1-o(1)\), every weight satisfies \(\|(w_t^{j,\mathsf k})^\parallel\|^2=o(1)\) throughout the corresponding time interval. Since \(\rho>0\) is fixed, weak subspace recovery is therefore impossible before $c/\eta_d$ for the tied and
augmented procedures, and before $cr/\eta_d$ for the untied procedure.

The upper and lower bounds differ only by the logarithmic factor $\log(d/r)=(1-\delta)\log d$. Since \(\eta_d^{-1} = d(\log d)^A\), we obtain
\[
n_{\mathsf T}^{(2)}(\rho),\;n_{\mathsf A}^{(2)}(\rho) = \widetilde\Theta(\eta_d^{-1}) = \widetilde\Theta(d), \qquad n_{\mathsf U}^{(2)}(\rho) = \widetilde\Theta(r\eta_d^{-1}) = \widetilde\Theta(rd).
\]
\end{proof}

\section{Correlation dynamics and auxiliary estimates} \label{sec:corr_dynamics}

The purpose of this section is to collect the dynamical identities and stochastic estimates used in the weak recovery proofs. We derive them for the tied student, for which the correlation dynamics take their cleanest form. Section~\ref{sec:weak_recovery_tied} uses these estimates to prove weak recovery for the tied student, while Section~\ref{sec:weak_recovery_untied} shows how the same structure changes under augmentation and untied training.

Throughout this section, we suppress the superscript \(\mathrm{T}\) and write
\[
w_t\coloneqq w_t^{\mathsf T}, \qquad L \coloneqq L^{\mathsf T}, \qquad \Phi\coloneqq\Phi^{\mathsf T}, \qquad \eta_d\coloneqq\eta_d^{\mathsf T}.
\]
We also work on the teacher-geometry event \(\mathcal G_d\) introduced in Lemma~\ref{lem:teacher_geometry}. All estimates below that condition on the teacher are uniform over \(w^\ast\in\mathcal G_d\).

\subsection{Correlation recursion and population drift}

We begin by deriving the exact evolution of the correlations
\[
m_k(w_t)=\langle \Pi^k w^\ast,w_t\rangle, \qquad 0\le k<r,
\]
under spherical SGD. We first identify the population contribution. For \(k \in\{0,\ldots,r-1\}\), define the teacher correlations
\[
c_k \coloneqq\langle w^\ast,\Pi^k w^\ast\rangle,
\]
where indices are understood modulo \(r\). By cyclicity, \(\langle \Pi^k w^\ast,\Pi^j w^\ast\rangle =
c_{j-k}\).

\begin{lemma}[Population drift]\label{lem:population_drift}
For every \(w\in\mathbb S^{d-1}\) and \(k\in\{0,\ldots,r-1\}\),
\begin{equation}\label{eq:general_population_drift}
-\langle \Pi^k w^\ast,\nabla\Phi(w)\rangle = \phi'(m_k(w))+\Gamma_k(w),
\end{equation}
where
\begin{equation}\label{eq:general_Gamma}
\Gamma_k(w) \coloneqq \sum_{j\neq k}c_{j-k}\phi'(m_j(w)) - m_k(w)\sum_{j=0}^{r-1}m_j(w)\phi'(m_j(w)).
\end{equation}
\end{lemma}

\begin{proof}
On the unit sphere, Hermite orthogonality gives
\[
\Phi(w)=-\sum_{j=0}^{r-1}\phi(m_j(w)).
\]
The spherical gradient of the population objective is
\[
\nabla\Phi(w) = (\Id_d-ww^\top)\widehat\nabla\Phi(w) = -\sum_{j=0}^{r-1}\phi'(m_j(w))\bigl(\Pi^j w^\ast-m_j(w)w\bigr).
\]
Taking the scalar product with \(\Pi^k w^\ast\) yields
\[
-\langle \Pi^k w^\ast,\nabla\Phi(w)\rangle = \sum_{j=0}^{r-1}\phi'(m_j(w))\left(\langle\Pi^k w^\ast,\Pi^j w^\ast\rangle-m_k(w)m_j(w)\right).
\]
Separating the term \(j=k\), using \(\langle\Pi^k w^\ast,\Pi^k w^\ast\rangle=1\), and recalling \(\langle\Pi^k w^\ast,\Pi^j w^\ast\rangle=c_{j-k}\), gives~\eqref{eq:general_population_drift}.
\end{proof}

We now combine this population identity with the sample-level SGD update. Define
\[
G_t \coloneqq \nabla L(w_{t-1};x_t,y_t),
\]
and introduce the sample-wise centered loss
\[
H_t(w) \coloneqq L(w;x_t,y_t)-\Phi(w).
\]
Then \(G_t = \nabla\Phi(w_{t-1}) + \nabla H_t(w_{t-1})\). We also define the coordinate-wise stochastic fluctuation
\begin{equation} \label{eq:stoch_fluctuation}
\xi_{k,t} \coloneqq -\left \langle \Pi^k w^\ast, \nabla H_t(w_{t-1}) \right\rangle.
\end{equation}

\begin{lemma}[Exact correlation evolution]\label{lem:correlation_evolution}
For every \(0\le k<r\) and \(t\ge1\),
\begin{equation}\label{eq:micro_decomposition}
m_k(w_t) = m_k(w_0) + p a_p^2\eta_d \sum_{\ell=1}^t m_k(w_{\ell-1})^{p-1} + E_{k,t},
\end{equation}
where
\begin{equation}\label{eq:micro_error}
\begin{split}
E_{k,t} & \coloneqq \eta_d \sum_{\ell=1}^t R_\sigma(m_k(w_{\ell-1})) + v\eta_d\sum_{\ell=1}^t\Gamma_k(w_{\ell-1}) + \eta_d\sum_{\ell=1}^t\xi_{k,\ell}\\& \quad - \sum_{\ell=1}^t\left(\sqrt{1+\eta_d^2\|G_\ell\|^2}-1\right)m_k(w_\ell),
\end{split}
\end{equation}
with
\[
R_\sigma(m)\coloneqq\sum_{k=p+1}^Pk a_k^2m^{k-1}.
\]
\end{lemma}

\begin{proof}
By the spherical SGD update~\eqref{eq:SGD},
\[
w_t = \frac{w_{t-1}-\eta_d G_t}{\|w_{t-1}-\eta_d G_t\|}.
\]
Since \(G_t\) belongs to the tangent space of \(\mathbb S^{d-1}\) at \(w_{t-1}\), we have \(\langle G_t,w_{t-1}\rangle=0\), and therefore
\[
\|w_{t-1}-\eta_dG_t\|^2 = 1+\eta_d^2\|G_t\|^2.
\]
Taking the scalar product of \(w_t\) with \(\Pi^k w^\ast\) gives
\[
m_k(w_t) = \langle \Pi^k w^\ast, w_t \rangle =  \frac{m_k(w_{t-1}) - \eta_d \langle   \Pi^k w^\ast,G_t  \rangle}{\|w_{t-1}-\eta_d G_t\|},
\]
or equivalently,
\[
\sqrt{1+\eta_d^2\|G_t\|^2}\, m_k(w_t) = m_k(w_{t-1}) - \eta_d \langle\Pi^k w^\ast,G_t\rangle.
\]
Using \(G_t = \nabla\Phi(w_{t-1}) + \nabla H_t(w_{t-1})\), Lemma~\ref{lem:population_drift}, and the definition of \(\xi_{k,t}\), we obtain
\[
\sqrt{1+\eta_d^2\|G_t\|^2}\, m_k(w_t) = m_k(w_{t-1}) + \eta_d\phi'(m_k(w_{t-1})) + \eta_d\Gamma_k(w_{t-1}) + \eta_d\xi_{k,t}.
\]
Therefore
\begin{equation}\label{eq:ev_corr}
\begin{split}
m_k(w_t) & = m_k(w_{t-1}) + \eta_d\phi'(m_k(w_{t-1})) + \eta_d\Gamma_k(w_{t-1}) + \eta_d\xi_{k,t}\\& \quad -\bigl(\sqrt{1+\eta_d^2\|G_t\|^2}-1\bigr)m_k(w_t).
\end{split}
\end{equation}
Finally,
\[
\phi'(m) = \sum_{k=p}^P k a_k^2 m^{k-1} = p a_p^2 m^{p-1} + R_\sigma (m).
\]
Summing~\eqref{eq:ev_corr} over \(\ell=1,\ldots, t\) and separating the leading Hermite contribution in~\eqref{eq:ev_corr} gives~\eqref{eq:micro_decomposition}.
\end{proof}

Since $\phi(z)=\sum_{\ell=p}^P a_\ell^2z^\ell$, termwise estimates give,
for a constant $C_\sigma<\infty$,
\begin{equation}\label{eq:phi_bounds}
|\phi'(z)|\le C_\sigma|z|^{p-1}, \qquad |R_\sigma(z)|\le C_\sigma|z|^p, \qquad |z|\le1.
\end{equation}
Moreover, the coefficients of $\phi$ are nonnegative, so
\begin{equation}\label{eq:phi_positive_lower}
\phi'(z)\ge p a_p^2z^{p-1}, \qquad 0\le z\le1.
\end{equation}

The four terms in \(E_{k,t}\) have distinct origins: they are, respectively, the higher-Hermite remainder, the interaction with the other teacher directions, the centered stochastic fluctuation, and the spherical normalization error. The deterministic remainder and interaction terms will be controlled directly from the correlation bounds and the geometry of the teacher orbit. The remaining two contributions require probabilistic estimates on \(\xi_{k,t}\) and \(\|G_t\|\), which we establish in the next two subsections.

\subsection{Coordinate noise estimates}

Define the filtration
\[
\mathcal{F}_t \coloneqq \sigma (w^\ast,w_0,(x_\ell,y_\ell)_{1\le\ell\le t}).
\]
Since $(x_t,y_t)$ is fresh conditionally on $\mathcal{F}_{t-1}$,
\[
\E[G_t\mid\mathcal{F}_{t-1}]=\nabla\Phi(w_{t-1}).
\]
Therefore,
\[
\xi_{k,t} = -\left \langle \Pi^k w^\ast, G_t-\E[G_t\mid\mathcal{F}_{t-1}] \right \rangle
\]
is a martingale difference with respect to \((\mathcal{F}_t)_{t\ge0}\).

We first derive an identity for the second moments of the feature gradient. Besides being used below to control \(\xi_{k,t}\), this identity will also be useful in the proof of the uniform gradient bound.

\begin{lemma}[Feature-gradient covariance]\label{lem:feature_gradient_covariance}
Let \(w\in\mathbb{S}^{d-1}\) and \(v\perp w\), and define \(a_j(w)\coloneqq\langle w,\Pi^jw\rangle\). Then
\[
\begin{split}
\E \left[ \langle v,\widehat\nabla f_w(x)\rangle^2 \right]  = \sum_{j=0}^{r-1} \left ( \phi'(a_j(w))\langle v,\Pi^jv\rangle + \phi''(a_j(w)) \langle w,\Pi^jv\rangle \langle\Pi^jw,v\rangle \right ).
\end{split}
\]
\end{lemma}

The proof of Lemma~\ref{lem:feature_gradient_covariance} is given in Appendix~\ref{sec:feature_gradient_covariance_proof}.

We now control the conditional moments of the coordinate noise.

\begin{lemma}[Coordinate noise moments] \label{lem:coordinate_noise_moments}
On the teacher-geometry event \(\mathcal{G}_d\), there exists \(C_\sigma<\infty\) such that, for all sufficiently large \(d\), all \(0\le k<r\), \(t\ge1\), and \(\nu\ge2\),
\begin{align}
\E [\xi_{k,t}^2 \mid \mathcal{F}_{t-1}] & \le C_\sigma \left( 1+r m_k(w_{t-1})^2 \right), \label{eq:coordinate_noise_variance}\\
\left( \E [ |\xi_{k,t}|^\nu \mid \mathcal{F}_{t-1}] \right)^{1/\nu} & \le C_\sigma\nu^P \left(1+\sqrt r |m_k(w_{t-1})| \right).
\label{eq:coordinate_noise_moments}
\end{align}
The bounds hold uniformly over \(w_{t-1}\in\mathbb S^{d-1}\) and over \(w^\ast\in\mathcal{G}_d\).
\end{lemma}

The proof of Lemma~\ref{lem:coordinate_noise_moments} is given in Appendix~\ref{sec:coordinate_noise_moments_proof}.

\subsection{Uniform control of the stochastic gradient}

We finally establish a high-probability bound on the norm of the stochastic gradient, uniformly over polynomially many SGD iterations.

\begin{lemma} \label{lem:uniform_gradient}
For every fixed $C>0$ there exists $C_{\sigma,C}<\infty$ such that, for every deterministic horizon $T\le d^C$,
\[
\mathbb{P}\left( \max_{1\le t\le T}\|G_t\| > C_{\sigma,C} (\log d)^P \sqrt{d} \right)=o(1).
\]
\end{lemma}

The proof of Lemma~\ref{lem:uniform_gradient} is given in Appendix~\ref{sec:uniform_gradient_proof}.

\section{Proof of weak recovery for the tied student} \label{sec:weak_recovery_tied}

In this section we prove part~\textup{(T)} of Theorem~\ref{thm:weak_recovery_appendix}. The proof shows that the leading polynomial dependence of the recovery time is determined by the information exponent \(p\): higher Hermite components affect the constants and logarithmic factors, but not the leading \(d^{p-1}\) sample complexity exponent.

\subsection{Proof strategy}

The proof proceeds in two main steps. The first step selects and amplifies one of the $r$ teacher correlations from its random initialization scale until the trajectory enters a suitable intermediate region. The second step starts from any point in this region and propagates the selected correlation to the fixed recovery level $\rho$.

Fix \(B>\frac{2}{p-2}\), depending only on \(p\), and define the two intermediate scales
\[
\alpha_d\coloneqq \gamma_d(\log d)^{2/(p-2)}, \qquad \beta_d\coloneqq \gamma_d(\log d)^B.
\]
Then 
\[
\gamma_d\ll\alpha_d\ll\beta_d\ll1. 
\]

We next introduce notation for the correlation selected at initialization. Define
\[
i_\ast \coloneqq \min \operatorname*{arg\,max}_{0\le k<r} q_k (w_0),
\]
where \(q_k(w)\) denotes the relevant correlation defined in~\eqref{eq:relevant_correlation}. Define the corresponding oriented correlation by 
\begin{equation} \label{eq:oriented_correlation}
\widetilde{m}_{i_\ast} (w_t) \coloneqq
\begin{cases}
\operatorname{sgn}(m_{i_\ast}(w_0)) m_{i_\ast} (w_t), & p\text{ even},\\[1mm]
m_{i_\ast} (w_t), & p\text{ odd},
\end{cases}
\end{equation}
with the convention \(\operatorname{sgn}(0)=1\). Lemma~\ref{lem:initialization_gap} shows that, with probability tending to one, the maximizer is unique, \(q_{i_\ast} (w_0) \asymp \gamma_d\), and is separated from all competing relevant correlations by a quantitative relative gap. 

Let 
\[
\mathcal{F}_t \coloneqq \sigma \left ( w^\ast, w_0, (x_\ell,y_\ell)_{1 \le \ell \le t} \right)
\]
denote the natural filtration. Then \(i_\ast\) is \(\mathcal F_0\)-measurable, \(w_t\) is \(\mathcal F_t\)-measurable, and each hitting time \(\tau_a\) is an \((\mathcal F_t)_{t\ge0}\)-stopping time.

The first phase shows that the small advantage present at initialization survives the stochastic fluctuations and interaction errors and is amplified until the selected correlation reaches the scale \(\beta_d\), while all competing correlations remain at the smaller scale \(\alpha_d\).

\begin{proposition}[Microscopic selection]\label{prop:microscopic_selection}
Suppose that Assumptions~\ref{hyp:link}(i),~\ref{hyp:teacher_init}, and~\ref{hyp:scaling} hold. There exists \(A_\ast=A_\ast(p,P)\ge(p-2)/2\) such that
the following holds for every fixed \(A>A_\ast\).
Run tied spherical SGD~\eqref{eq:SGD} with step size
\[
\eta_d=d^{-p/2}(\log d)^{-A}.
\]
Then there exist constants \(C_\alpha,C_\beta<\infty\), depending only on the fixed model parameters, such that, with probability \(1-o(1)\), 
\[
\tau_{\beta_d}\le\frac{C_\beta}{\eta_d\gamma_d^{p-2}},
\]
and
\[
\beta_d\le \widetilde{m}_{i_\ast} (w_{\tau_{\beta_d}}) \le2\beta_d, \qquad \max_{k\ne i_\ast}|m_k(w_{\tau_{\beta_d}})|\le C_\alpha\alpha_d.
\]
Moreover, \(i_\ast\) remains the unique relevant leader up to time \(\tau_{\beta_d}\).
\end{proposition}

The proof of Proposition~\ref{prop:microscopic_selection} is given in Section~\ref{sec:micro_tied}.

Since \(\tau_{\beta_d}\) is a stopping time and the SGD recursion is driven by i.i.d.\ observations, the strong Markov property implies that, conditionally on \(\mathcal F_{\tau_{\beta_d}}\), the post-\(\tau_{\beta_d}\) trajectory has the law of tied spherical SGD started from \(w_{\tau_{\beta_d}}\) and driven by fresh observations, with \(w^\ast\) fixed. This allows us to restart the analysis at \(\tau_{\beta_d}\).

\begin{proposition}[Growth to macroscopic size] \label{prop:macroscopic_growth}
Suppose that Assumptions~\ref{hyp:link}(i),~\ref{hyp:teacher_init}, and~\ref{hyp:scaling} hold. Fix \(\rho\in(0,1)\), independently of \(d\). There exists \(A_\ast=A_\ast(p,P)\ge(p-2)/2\) such that
the following holds for every fixed \(A>A_\ast\).
Run tied spherical SGD~\eqref{eq:SGD} with step size
\[
\eta_d=d^{-p/2}(\log d)^{-A}.
\]
Then there exist \(C_\rho<\infty\) and a deterministic sequence \(\varepsilon_d\to 0\) such that, almost surely on 
\[
\mathcal G_d \cap \left\{ \beta_d \le \widetilde m_{i_\ast}(w_{\tau_{\beta_d}}) \le 2\beta_d, \quad \max_{k\neq i_\ast}|m_k(w_{\tau_{\beta_d}})| \le C_\alpha\alpha_d \right\},
\]
it holds that
\[
\mathbb{P} \left( \tau_\rho-\tau_{\beta_d} \le \frac{C_\rho}{\eta_d\beta_d^{p-2}} \;\middle|\; \mathcal{F}_{\tau_{\beta_d}} \right) \ge1-\varepsilon_d,
\]
where \(\mathcal G_d\) is the teacher-geometry event of Lemma~\ref{lem:teacher_geometry}. Moreover, with the same conditional probability lower bound, \(i_\ast\) remains the unique favorable leader up to time \(\tau_\rho\), while all competing overlaps remain microscopic.
\end{proposition}

The proof of Proposition~\ref{prop:macroscopic_growth} is given in Section~\ref{sec:macro_tied}. We now provide the proof of part \textup{(T)} of Theorem~\ref{thm:weak_recovery_appendix}.

\begin{proof}[Proof of part \textup{(T)} of Theorem~\ref{thm:weak_recovery_appendix}]
Let
\[
\mathcal E_d \coloneqq \mathcal G_d \cap \left\{ \tau_{\beta_d} \le \frac{C_\beta}{\eta_d \gamma_d^{p-2}}, \quad \beta_d \le \widetilde m_{i_\ast}(w_{\tau_{\beta_d}}) \le 2\beta_d, \quad \max_{k\neq i_\ast} |m_k(w_{\tau_{\beta_d}})| \le C_\alpha\alpha_d \right\}.
\]
By Proposition~\ref{prop:microscopic_selection} and Lemma~\ref{lem:teacher_geometry}, \(\mathbb P(\mathcal E_d)=1-o(1)\). Moreover, \(\mathcal E_d\in\mathcal F_{\tau_{\beta_d}}\).

By Proposition~\ref{prop:macroscopic_growth} and the tower property,
\begin{align*}
&\mathbb P\left( \mathcal E_d\cap\left\{\tau_\rho-\tau_{\beta_d}>\frac{C_\rho}{\eta_d\beta_d^{p-2}}\right\}\right) = \mathbb E\left[\mathbf 1_{\mathcal E_d}\mathbb P\left(\tau_\rho-\tau_{\beta_d}>\frac{C_\rho}{\eta_d\beta_d^{p-2}}\,\middle|\,\mathcal F_{\tau_{\beta_d}}\right)\right] \le \varepsilon_d.
\end{align*}
Hence, with probability \(1-o(1)\),
\[
\tau_\rho \le \frac{C_\beta}{\eta_d \gamma_d^{p-2}} + \frac{C_\rho}{\eta_d\beta_d^{p-2}}.
\]
Since $\beta_d=\gamma_d(\log d)^B$, we have
\[
\frac{1}{\eta_d\beta_d^{p-2}} = (\log d)^{-B(p-2)}\frac{1}{\eta_d \gamma_d^{p-2}} = o\left(\frac{1}{\eta_d \gamma_d^{p-2}}\right).
\]
Therefore, after increasing the constant \(C\) if necessary,
\[
\tau_\rho \le \frac{C}{\eta_d \gamma_d^{p-2}}
\]
with probability \(1-o(1)\).
\end{proof}

\subsection{Microscopic selection}
\label{sec:micro_tied}

We now study the dynamics while all correlations remain microscopic and prove Proposition~\ref{prop:microscopic_selection}. According to~\eqref{eq:micro_decomposition}, we have the evolution
\[
m_k(w_t) = m_k(w_0) +p a_p^2\eta_d\sum_{\ell=1}^t m_k(w_{\ell-1})^{p-1} +E_{k,t},
\]
where \(E_{k,t}\) is defined in~\eqref{eq:micro_error}. Set
\begin{equation}\label{eq:time_horizon}
T_0\coloneqq\left\lceil\frac{C_0}{\eta_d\gamma_d^{p-2}}\right\rceil,
\end{equation}
where $C_0>0$ is fixed and sufficiently large, and define
\[ 
\hat\tau_{\beta_d} \coloneqq \inf\left\{t\ge0 \colon \max_{0\le k\le r-1}|m_k(w_t)|\ge\beta_d\right\}.
\]
For even $p$, $\hat\tau_{\beta_d}=\tau_{\beta_d}$. For odd $p$ we prove below that the first absolute crossing is necessarily positive, and hence the two stopping times agree up to $T_0$.

\begin{lemma}\label{lem:micro_error}
Suppose $\beta_d=\gamma_d(\log d)^B$ with $B>2/(p-2)$ and
$\eta_d=d^{-p/2}(\log d)^{-A}$ with $A>A_\ast(p,P)$. Then, with probability $1-o(1)$,
\[
\max_{0\le k\le r-1} \max_{0\le t\le\hat\tau_{\beta_d}\wedge T_0}|E_{k,t}| = o\bigl((\log d)^{-2}\gamma_d\bigr).
\]
\end{lemma}

\begin{corollary}\label{cor:stopping_times_agree}
With probability $1-o(1)$,
\[
\hat\tau_{\beta_d}\wedge T_0 = \tau_{\beta_d}\wedge T_0.
\]
Consequently Lemma~\ref{lem:micro_error} remains valid with $\hat\tau_{\beta_d}$ replaced by $\tau_{\beta_d}$.
\end{corollary}

\begin{proof}
For even $p$ the two stopping times coincide by definition. Suppose $p$ is odd. Then $p-1$ is even, so the leading term in~\eqref{eq:micro_decomposition} is nonnegative. Hence, for $t\le\hat\tau_{\beta_d}\wedge T_0$,
\[
m_k(w_t)\ge m_k(w_0)-|E_{k,t}|.
\]
On the initialization event, $\max_k|m_k(w_0)|\le C\gamma_d$, while Lemma~\ref{lem:micro_error} gives $\max_{k,t}|E_{k,t}|=o(\gamma_d)$. Thus
\[
m_k(w_t)\ge-C\gamma_d-o(\gamma_d)>-\beta_d
\]
for all sufficiently large $d$, uniformly in $k$ and in the stopped time interval. No coordinate can therefore exit through $-\beta_d$ before $T_0$, proving the claim.
\end{proof}

\begin{lemma} \label{lem:martingale}
With probability $1-o(1)$,
\[
\max_{0\le k\le r-1} \max_{0\le t\le\hat\tau_{\beta_d}\wedge T_0} \left|\eta_d\sum_{\ell=1}^t\xi_{k,\ell}\right| \le C\eta_d\left(\sqrt{T_0\log d}+(\log d)^{P+1}\right).
\]
\end{lemma}

\begin{proof}
For $1\le\ell\le T_0$ set
\[
Y_{k,\ell}=\mathbf1_{\{\ell\le\hat\tau_{\beta_d}\}}\xi_{k,\ell}.
\]
Then $(Y_{k,\ell})_\ell$ is a martingale difference sequence and
\[
\max_{t\le\hat\tau_{\beta_d}\wedge T_0}\left|\sum_{\ell=1}^t\xi_{k,\ell}\right| = \max_{t\le T_0}\left|\sum_{\ell=1}^tY_{k,\ell}\right|.
\]
On $\{\ell\le\hat\tau_{\beta_d}\}$, $|m_k(w_{\ell-1})|<\beta_d$. Since
$r\beta_d^2=o(1)$, Lemma~\ref{lem:coordinate_noise_moments} gives, for all sufficiently large $d$,
\begin{equation}\label{eq:Y_uniform_moments}
\E[Y_{k,\ell}^2\mid\mathcal{F}_{\ell-1}]\le C_\sigma, \qquad \left(\E[|Y_{k,\ell}|^q\mid\mathcal{F}_{\ell-1}]\right)^{1/q} \le C_\sigma q^P.
\end{equation}
Choose $q_d=\lceil c\log d\rceil$ and $R_d=C(\log d)^P$, with $c,C$ sufficiently large. Conditional Markov's inequality gives
\[
\mathbb{P}(|Y_{k,\ell}|>R_d\mid\mathcal{F}_{\ell-1})\le4^{-q_d}.
\]
Since $T_0=d^{p-1}\operatorname{polylog}(d)\le d^p$ and $r\le d$, a union bound shows that, with probability $1-o(1)$,
\[
|Y_{k,\ell}|\le R_d \qquad\text{for all }k<r,\ \ell\le T_0.
\]

Define the truncated and recentered increments
\[
\widehat Y_{k,\ell} = Y_{k,\ell}\mathbf1_{\{|Y_{k,\ell}|\le R_d\}} - \E\bigl[Y_{k,\ell}\mathbf1_{\{|Y_{k,\ell}| \le R_d\}}\mid\mathcal{F}_{\ell-1}\bigr].
\]
Then $|\widehat Y_{k,\ell}|\le2R_d$ and the predictable quadratic variation up to $T_0$ is at most $C_\sigma T_0$. Freedman's maximal inequality therefore gives, after a union bound over $k<r$,
\[
\max_{k<r}\max_{t\le T_0}\left|\sum_{\ell=1}^t\widehat Y_{k,\ell}\right| \le C\left(\sqrt{T_0\log d}+R_d\log d\right)
\]
with probability $1-o(1)$.

Finally, the truncation bias satisfies
\[
\left|\E\bigl[Y_{k,\ell}\mathbf1_{\{|Y_{k,\ell}|\le R_d\}}\mid\mathcal{F}_{\ell-1}\bigr]\right| \le R_d4^{-q_d},
\]
so $T_0R_d4^{-q_d}=o(1)$ after increasing $c$. Since $R_d=C(\log d)^P$, the claim follows after multiplying by $\eta_d$.
\end{proof}

\begin{proof}[Proof of Lemma~\ref{lem:micro_error}]
Set $\nu_d=(\log d)^{-2}$. We control the four terms in~\eqref{eq:micro_error}.

\medskip
\noindent\textbf{Higher-Hermite remainder.}
For $\ell\le\hat\tau_{\beta_d}\wedge T_0$,~\eqref{eq:phi_bounds} gives
\[
|R_\sigma(m_k(w_{\ell-1}))|\le C_\sigma\beta_d^p.
\]
Therefore
\[
\frac{\displaystyle\max_{k,t}\left|\eta_d\sum_{\ell=1}^tR_\sigma(m_k(w_{\ell-1}))\right|}{\nu_d\gamma_d} \le C_\sigma\,\frac{\eta_dT_0\beta_d^p}{\nu_d\gamma_d} \lesssim \frac{\gamma_d}{\nu_d}\left(\frac{\beta_d}{\gamma_d}\right)^p = o(1).
\]

\medskip
\noindent\textbf{Interaction term.}
On the same stopped interval,~\eqref{eq:general_Gamma} and~\eqref{eq:phi_bounds} imply
\[
|\Gamma_k(w_{\ell-1})| \le C_\sigma\bigl(r\mu_\ast+r\beta_d^2\bigr)\beta_d^{p-1}.
\]
Hence
\[
\frac{\displaystyle\max_{k,t}\left|\eta_d\sum_{\ell=1}^t\Gamma_k(w_{\ell-1})\right|}{\nu_d\gamma_d} \le C_\sigma\bigl(r\mu_\ast+r\beta_d^2\bigr)\nu_d^{-1}\left(\frac{\beta_d}{\gamma_d}\right)^{p-1} = o(1),
\]
because $r\mu_\ast$ and $r\beta_d^2$ decay polynomially in $d$, whereas all remaining factors are polylogarithmic.

\medskip
\noindent\textbf{Normalization term.}
Since $T_0$ is polynomial in $d$, Lemma~\ref{lem:uniform_gradient} gives, with probability $1-o(1)$,
\begin{equation}\label{eq:uniform_gradient_norm}
\max_{1\le\ell\le T_0}\|G_\ell\| \le C_\sigma\sqrt d\,(\log d)^P.
\end{equation}
Our choice of $A$ ensures $C_\sigma\sqrt d\,(\log d)^P=o(\beta_d)$. Thus, also at the exit step,
\[
|m_k(w_\ell)|\le\beta_d+ C_\sigma \eta_d \sqrt d\,(\log d)^P \le2\beta_d.
\]
Using $\sqrt{1+x}-1\le x/2$,
\[
\max_{k,t} \left|\sum_{\ell=1}^t\bigl(\sqrt{1+\eta_d^2\|G_\ell\|^2}-1\bigr)m_k(w_\ell)\right|\le C \eta_d^2 T_0 d (\log d )^{2P} \beta_d.
\]
After division by $\nu_d\gamma_d$ this is at most
\[
C_\sigma(\log d)^{-A-(p-2)/2+2P+B+2},
\]
which is $o(1)$ for $A>A_\ast(p,P)$.

\medskip
\noindent\textbf{Martingale term.}
By Lemma~\ref{lem:martingale}, after division by $\nu_d\gamma_d$ the first stochastic contribution is
\[
\frac{\eta_d\sqrt{T_0\log d}}{\nu_d\gamma_d} \lesssim (\log d)^{-A/2+5/2-p/4}=o(1)
\]
for $A$ sufficiently large, while
\[
\frac{\eta_d(\log d)^{P+1}}{\nu_d\gamma_d} \lesssim d^{-(p-1)/2}(\log d)^{-A+P+5/2}=o(1).
\]
Combining the four estimates proves
\[
\max_{k<r}\max_{t\le\hat\tau_{\beta_d}\wedge T_0}|E_{k,t}| = o(\nu_d\gamma_d)
\]
with probability $1-o(1)$, uniformly over admissible teachers and initializations.
\end{proof}

We now give the proof of Proposition~\ref{prop:microscopic_selection}.

\begin{proof}[Proof of Proposition~\ref{prop:microscopic_selection}]
Set $\nu_d=(\log d)^{-2}$ and work on the intersection of the high-probability events from Lemma~\ref{lem:initialization_gap}, Lemma~\ref{lem:micro_error}, Corollary~\ref{cor:stopping_times_agree}, and~\eqref{eq:uniform_gradient_norm}. Let
\[
m_0 \coloneqq M_\sigma (w_0),\qquad c\gamma_d\le m_0\le C\gamma_d,
\]
and
\[
e_d\coloneqq \max_{k<r}\max_{t\le\tau_{\beta_d}\wedge T_0}|E_{k,t}| = o(\nu_dm_0).
\]
Then, for every $t\le\tau_{\beta_d}\wedge T_0$,
\begin{equation}\label{eq:corr_inequality}
m_k(w_0)+p a_p^2\eta_d\sum_{\ell=0}^{t-1}m_k(w_\ell)^{p-1}-e_d \le m_k(w_t) \le m_k(w_0)+p a_p^2\eta_d\sum_{\ell=0}^{t-1}m_k(w_\ell)^{p-1}+e_d.
\end{equation}

Recall the oriented correlation \(\widetilde{m}_{i_\ast} (w_t)\) defined in~\eqref{eq:oriented_correlation}. Then \(\widetilde m_{i_\ast}(w_0)=m_0\) and
\[
\widetilde m_{i_\ast}(w_t) = m_0+p a_p^2\eta_d\sum_{\ell=0}^{t-1}\widetilde m_{i_\ast}(w_\ell)^{p-1}+\overline E_{\ast,t}, \qquad |\overline E_{\ast,t}|\le e_d.
\]
Define
\[
L_t \coloneqq m_0-e_d+p a_p^2\eta_d \sum_{\ell=0}^{t-1}\widetilde m_{i_\ast}(w_\ell)^{p-1}.
\]
Then \(\widetilde m_{i_\ast}(w_t)\ge L_t\), \(L_0>0\), and
\begin{equation}\label{eq:lower_bound}
L_t-L_{t-1} = p a_p^2\eta_d\widetilde m_{i_\ast}(w_{t-1})^{p-1} \ge p a_p^2\eta_d L_{t-1}^{p-1}.
\end{equation}

For a favorable competitor, set $b_{k,t}=|m_k(w_t)|$ when $p$ is even. When $p$ is odd, first consider competitors with $m_k(w_0)\ge0$ and set $b_{k,t}=m_k(w_t)$ as long as this remains nonnegative. Define
\[
U_{k,t} \coloneqq b_{k,0}+e_d+p a_p^2\eta_d\sum_{\ell=0}^{t-1}b_{k,\ell}^{p-1}.
\]
Then $b_{k,t}\le U_{k,t}$ and
\[
U_{k,t}-U_{k,t-1} \le p a_p^2\eta_dU_{k,t-1}^{p-1}.
\]
The initialization gap gives
\[
L_0-U_{k,0} \ge \nu_dm_0-2e_d>0,
\]
so monotonicity of $x\mapsto x^{p-1}$ on $\R_+$ implies inductively that the selected coordinate remains strictly ahead of every favorable competitor.

\medskip
\noindent\textbf{Time to reach $\beta_d$.}
Set $q=p-2$ and
\[
t_\beta\coloneqq\inf\{t\ge0:L_t\ge\beta_d\}.
\]
Before this time, the mean value theorem and~\eqref{eq:lower_bound} give
\[
L_{t-1}^{-q}-L_t^{-q} \ge \frac{qp a_p^2\eta_d}{(1+p a_p^2\eta_d\beta_d^q)^{q+1}}.
\]
Hence
\begin{equation}\label{eq:sharp_leader_time}
t_\beta \le 1 + \frac{(1+p a_p^2\eta_d\beta_d^q)^{q+1}}{qp a_p^2\eta_dL_0^q}\left[1-\left(\frac{L_0}{\beta_d}\right)^q\right].
\end{equation}
Since $p a_p^2\eta_d\beta_d^q=o(\nu_d)$ and $L_0\asymp \gamma_d$,
\[
t_\beta\lesssim\frac1{\eta_d\gamma_d^{p-2}}.
\]
Taking $C_0$ sufficiently large gives $t_\beta\le T_0$, and therefore $\tau_{\beta_d}\le t_\beta$.

\medskip
\noindent\textbf{Size of the competitors.}
The initialization gap implies
\[
\frac{U_{k,0}}{L_0} \le 1-(1-o(1))\nu_d.
\]
Iterating the upper scalar recursion gives
\[
U_{k,t} \le U_{k,0}\left(1-qp a_p^2\eta_dU_{k,0}^qt\right)^{-1/q}
\]
whenever the denominator is positive. Combining this with~\eqref{eq:sharp_leader_time}, and using the extra $+1$ term only through $qp a_p^2\eta_dU_{k,0}^q=o(\nu_d)$, yields
\[
1-qp a_p^2\eta_dU_{k,0}^qt \ge c_p\nu_d, \qquad t\le\tau_{\beta_d},
\]
for some $c_p>0$. Consequently
\begin{equation}\label{eq:competitor_final_bound}
\max_{t\le\tau_{\beta_d}}|m_k(w_t)| \lesssim \gamma_d\nu_d^{-1/(p-2)} = \alpha_d = o(\beta_d).
\end{equation}

Suppose now that $p$ is odd and $m_k(w_0)<0$. Let
\[
\nu_k \coloneqq\inf\{t\ge0 \colon m_k(w_t)\ge0\}.
\]
Before $\nu_k$,~\eqref{eq:micro_decomposition} gives $|m_k(w_t)|\lesssim \gamma_d$. If $\nu_k\le\tau_{\beta_d}$ and $m_k(w_{\nu_k-1})=-u$, then $u<\beta_d$ and
\[
-u+p a_p^2\eta_du^{p-1} = -u(1-p a_p^2\eta_du^{p-2})\le 0.
\]
The crossing overshoot is therefore at most $2e_d$. Starting from this size, the positive scalar envelope remains $O(e_d)$ because
\[
p a_p^2\eta_de_d^{p-2}\tau_{\beta_d} \lesssim \left(\frac{e_d}{\gamma_d}\right)^{p-2} = o(1).
\]
This proves~\eqref{eq:competitor_final_bound} for initially negative competitors as well. The same argument handles later sign changes.

Finally, the preceding comparison shows that $i_\ast$ is the coordinate that reaches the threshold. Immediately before the crossing its favorable magnitude is below $\beta_d$, and the exact normalized update gives
\[
\beta_d \le M_\sigma (w_{\tau_{\beta_d}}) \le \beta_d+\eta_d\|G_{\tau_{\beta_d}}\|.
\]
By~\eqref{eq:uniform_gradient_norm}, $\eta_d\|G_{\tau_{\beta_d}}\|=o(\beta_d)$, so
\[
M_\sigma (w_{\tau_{\beta_d}})=\beta_d(1+o(1)).
\]
This proves Proposition~\ref{prop:microscopic_selection}.
\end{proof}

\subsection{Growth to a macroscopic level} \label{sec:macro_tied}

We now prove Proposition~\ref{prop:macroscopic_growth}. By Proposition~\ref{prop:microscopic_selection}, with probability $1-o(1)$ there is a unique index $i_\ast$ such that
\[
\beta_d\le \widetilde{m}_{i_\ast}  (w_{\tau_{\beta_d}} )\le2\beta_d, \qquad \max_{k\ne i_\ast}|m_k(w_{\tau_{\beta_d}})|\le C_\alpha\alpha_d,
\]
where \(\widetilde{m}_{i_\ast} (w_t)\) is the oriented correlation defined in~\eqref{eq:oriented_correlation}. We may thus orient all correlations according to the selected direction and suppress the bars below, so that $m_{i_\ast}(w_t)=\widetilde{m}_{i_\ast} (w_t)>0$.

The competitor dynamics contain a contractive factor generated by the selected overlap. We first record the deterministic estimate used to exploit this damping.

\begin{lemma}[Damping preserves partial-sum bounds]\label{lem:exact_damping}
Suppose real sequences satisfy
\[
d_t x_{t+1}=q_t x_t+F_t+N_{t+1}-N_t, \qquad 0\le q_t\le1, \quad d_t\ge1, \quad N_0=0.
\]
Then, for every $t\ge0$,
\[
|x_t| \le |x_0| + \sum_{\ell=0}^{t-1}|F_\ell| + 2\max_{0\le\ell\le t}|N_\ell|.
\]
\end{lemma}

\begin{proof}
Iterating the recursion, an increment introduced at time $\ell$ is multiplied at time $t$ by
\[
\omega_{\ell,t} \coloneqq \frac1{d_\ell}\prod_{j=\ell+1}^{t-1}\frac{q_j}{d_j}, \qquad0\le\ell<t.
\]
The weights lie in $[0,1]$ and are nondecreasing in $\ell$. Summation by parts therefore gives
\[
\left|\sum_{\ell=0}^{t-1}\omega_{\ell,t}(N_{\ell+1}-N_\ell)\right| \le 2\max_{0\le\ell\le t}|N_\ell|.
\]
The coefficients multiplying $x_0$ and the forcing terms are at most one, which proves the claim.
\end{proof}

\begin{proof}[Proof of Proposition~\ref{prop:macroscopic_growth}]
Condition on $\mathcal{F}_{\tau_{\beta_d}}$, then
\[
\beta_d\le \widetilde{m}_{i_\ast}(w_0)\le2\beta_d, \qquad \max_{k\ne i_\ast}|m_k(w_0)|\le C_\alpha\alpha_d.
\]
Set
\[
c\coloneqq\frac{p a_p^2(1-\rho^2)}4, \qquad \varepsilon\coloneqq\frac14, \qquad T\coloneqq\left\lceil\frac{C_T}{\eta_d\beta_d^{p-2}}\right\rceil,
\]
where $C_T$ is chosen below. Define
\[
D_t\coloneqq c\eta_d\sum_{\ell=0}^{t-1}m_{i_\ast}(w_\ell)^{p-1}
\]
and
\[
\tau \coloneqq T\wedge\inf\left\{t\ge0:
\begin{array}{l}
m_{i_\ast}(w_t)\notin(\beta_d/2,\rho),\\
\text{or }\max_{k\ne i_\ast}|m_k(w_t)|\ge2C_\alpha\alpha_d,\\
\text{or }D_t\ge1
\end{array}
\right\}.
\]
For $t<\tau$, $0<m_{i_\ast}(w_t)<\rho$ and $D_{t\wedge\tau}\le1+c\eta_d$.

\medskip
\noindent\textbf{Step 1: Progress of the selected correlation.}
From Lemma~\ref{lem:population_drift},
\[
\begin{split}
-\langle \Pi^{i_\ast}w^\ast,\nabla\Phi(w_t)\rangle = {}&(1-m_{i_\ast}(w_t)^2)\phi'(m_{i_\ast}(w_t))\\ &+ \sum_{j\ne i_\ast}c_{j-i_\ast}\phi'(m_j(w_t)) - m_{i_\ast}(w_t)\sum_{j\ne i_\ast}m_j(w_t)\phi'(m_j(w_t)).
\end{split}
\]
For $m\in[0,\rho]$,~\eqref{eq:phi_positive_lower} gives $\phi'(m)\ge p a_p^2m^{p-1}$. For the competitors,~\eqref{eq:phi_bounds} gives $|\phi'(m_j)|\le C_\sigma|m_j|^{p-1}$. Hence, for $t<\tau$,
\[
-\langle \Pi^{i_\ast}w^\ast,\nabla\Phi(w_t)\rangle \ge p a_p^2(1-\rho^2)m_{i_\ast}(w_t)^{p-1} - C_\sigma r\mu_\ast\alpha_d^{p-1} - C_\sigma r\alpha_d^p m_{i_\ast}(w_t).
\]
Since $m_{i_\ast}(w_t)\ge\beta_d/2$, both error terms are $o(m_{i_\ast}(w_t)^{p-1})$. Thus, for all sufficiently large $d$,
\begin{equation}\label{eq:continuation_population_lower}
-\langle \Pi^{i_\ast}w^\ast,\nabla\Phi(w_t)\rangle \ge 2c\,m_{i_\ast}(w_t)^{p-1}.
\end{equation}

Lemma~\ref{lem:uniform_gradient} applies uniformly after the restart. Since $T$ is polynomial in $d$, with conditional probability $1-o(1)$,
\begin{equation}\label{eq:continuation_gradient_event}
\max_{1\le t\le T}\|G_t\|\le C_\sigma\sqrt d\,(\log d)^P.
\end{equation}
Our choice of $A$ ensures
\begin{equation}\label{eq:continuation_gradient_scales}
\frac{C_\sigma \eta_d \sqrt d\,(\log d)^P }{\beta_d}=o(1), \qquad \frac{C_\sigma^2 \eta_d d (\log d)^{2P}}{\beta_d^{p-2}}=o(1).
\end{equation}
The normalization correction is therefore bounded by
\[
\left|\bigl(\sqrt{1+\eta_d^2\|G_{t+1}\|^2}-1\bigr)m_{i_\ast}(w_{t+1})\right| \le c\eta_dm_{i_\ast}(w_t)^{p-1}.
\]
Combining with~\eqref{eq:continuation_population_lower} gives
\begin{equation}\label{eq:continuation_selected_step}
m_{i_\ast}(w_{t+1}) \ge m_{i_\ast}(w_t) + c\eta_dm_{i_\ast}(w_t)^{p-1} + \eta_d\xi_{i_\ast,t+1}, \qquad t<\tau.
\end{equation}

Define the stopped martingales
\[
N_{k,t} \coloneqq \eta_d\sum_{\ell=0}^{t-1}\mathbf1_{\{\ell<\tau\}}\xi_{k,\ell+1}.
\]
For the selected coordinate,
\[
\sum_{\ell<\tau}m_{i_\ast}(w_\ell)^2 \le \frac{C}{\eta_d\beta_d^{p-3}},
\]
because $m_{i_\ast}(w_\ell)\ge\beta_d/2$ and $D_\tau\le1+c\eta_d$. Lemma~\ref{lem:coordinate_noise_moments} yields
\[
\E[N_{i_\ast,T}^2] \le C_\sigma\eta_d^2\left(T+\frac{r}{\eta_d\beta_d^{p-3}}\right) \le C_\sigma\eta_d^2T(1+r\beta_d) = o(\beta_d^2).
\]
Hence Doob's maximal inequality implies
\begin{equation}\label{eq:selected_noise_event}
\max_{0\le n\le T}|N_{i_\ast,n}|\le\varepsilon\beta_d
\end{equation}
with conditional probability $1-o(1)$. Summing~\eqref{eq:continuation_selected_step} gives, for every $t\le\tau$,
\begin{equation}\label{eq:continuation_progress}
m_{i_\ast}(w_t) \ge (1-\varepsilon)\beta_d+D_t.
\end{equation}
Thus neither the lower boundary $\beta_d/2$ nor the cutoff $D_t=1$ can trigger $\tau$.

\medskip
\noindent\textbf{Step 2: Persistence of the microscopic separation.}
For $k\ne i_\ast$ and $\ell<\tau$, $|m_k(w_\ell)|<2C_\alpha\alpha_d$, so Lemma~\ref{lem:coordinate_noise_moments} gives uniformly
\[
\E[(\mathbf1_{\{\ell<\tau\}}\xi_{k,\ell+1})^2\mid\mathcal{F}_\ell]\le C_\sigma, \qquad \|\mathbf1_{\{\ell<\tau\}}\xi_{k,\ell+1}\|_{L^q(\,\cdot\mid\mathcal{F}_\ell)} \le C_\sigma q^P.
\]
Repeating the truncation--Freedman argument from Lemma~\ref{lem:martingale} gives, with conditional probability $1-o(1)$,
\begin{equation}\label{eq:continuation_comp_noise}
\max_{k\ne i_\ast}\max_{0\le t\le T}|N_{k,t}| \le C\eta_d\left(\sqrt{T\log d}+(\log d)^{P+1}\right) = o(\alpha_d).
\end{equation}

For $t<\tau$ and $k\ne i_\ast$, isolate the damping generated by the selected coordinate in the exact recursion:
\[
\sqrt{1+\eta_d^2\|G_{t+1}\|^2}\,m_k(w_{t+1}) = q_t m_k(w_t)+\eta_dF_{k,t}+\eta_d\xi_{k,t+1},
\]
where
\[
q_t \coloneqq 1-\eta_dm_{i_\ast}(w_t)\phi'(m_{i_\ast}(w_t))
\]
and
\[
F_{k,t} \coloneqq \phi'(m_k(w_t)) + \sum_{j\ne k}c_{j-k}\phi'(m_j(w_t)) - m_k(w_t)\sum_{j\ne i_\ast}m_j(w_t)\phi'(m_j(w_t)).
\]
Since $m_{i_\ast}>0$ and $\phi'(m_{i_\ast})>0$, while $\eta_d\to0$, we have $0\le q_t\le1$ for all sufficiently large $d$. Moreover,~\eqref{eq:phi_bounds} implies
\[
|F_{k,t}| \le C_\sigma\left(\alpha_d^{p-1} + \mu_\ast m_{i_\ast}(w_t)^{p-1} + r\alpha_d^{p+1}\right) \le C_\sigma\left(\alpha_d^{p-1} + \mu_\ast m_{i_\ast}(w_t)^{p-1}\right),
\]
where we used $r\alpha_d^2=o(1)$. Lemma~\ref{lem:exact_damping} therefore gives, for $t\le\tau$,
\[
|m_k(w_t)| \le C_\alpha\alpha_d + C_\sigma\eta_dT\alpha_d^{p-1} + C_\sigma\mu_\ast\eta_d\sum_{\ell<t}m_{i_\ast}(w_\ell)^{p-1} + 2\max_{\ell\le t}|N_{k,\ell}|.
\]
Now
\[
\eta_dT\alpha_d^{p-2} = O\!\left(\left(\frac{\alpha_d}{\beta_d}\right)^{p-2}\right)=o(1),
\]
$D_t\le1+c\eta_d$, $\mu_\ast=o(\alpha_d)$, and~\eqref{eq:continuation_comp_noise} is $o(\alpha_d)$. Hence
\[
|m_k(w_t)| < 2C_\alpha\alpha_d
\]
for every $k\ne i_\ast$ and $t\le\tau$, so no competitor can trigger the stopping time.

\medskip
\noindent\textbf{Step 3: Reaching the target.}
By~\eqref{eq:continuation_progress},
\[
m_{i_\ast}(w_t) \ge (1-\varepsilon)\beta_d + c\eta_d\sum_{\ell=0}^{t-1}m_{i_\ast}(w_\ell)^{p-1} \eqqcolon L_t.
\]
The standard discrete comparison for $x'=cx^{p-1}$ gives
\[
t_\rho \coloneqq \inf\{n:L_n\ge\rho\} \le \frac{C_{\sigma,\rho}}{\eta_d((1-\varepsilon)\beta_d)^{p-2}}.
\]
Choosing $C_T$ sufficiently large ensures $t_\rho\le T$. Since the other stopping mechanisms have already been excluded, the selected coordinate reaches $\rho$ before $T$, and
\[
\tau_\rho \le \left\lceil\frac{C_\rho}{\eta_d\beta_d^{p-2}}\right\rceil.
\]
Throughout this interval,
\[
m_{i_\ast}(w_t)\ge\frac{\beta_d}{2}, \qquad \max_{k\ne i_\ast}|m_k(w_t)|<2C_\alpha\alpha_d.
\]
Since $\alpha_d=o(\beta_d)$, $i_\ast$ remains the unique favorable leader. Therefore \(M_\sigma(w_{\tau_\rho})\ge\rho\). This proves Proposition~\ref{prop:macroscopic_growth}.
\end{proof}

\section{Proofs for the augmented and untied students} \label{sec:weak_recovery_untied}

We now prove parts~\textup{(U)} and~\textup{(A)} of Theorem~\ref{thm:weak_recovery_appendix}. The key observation is that, because the correlation loss is linear in the predictor, the gradient with respect to one untied weight does not involve the other weights. Consequently, both procedures reduce to one-vector dynamics.

\subsection{Componentwise reductions}

For \(u\in\mathbb S^{d-1}\), define the single-weight loss
\begin{equation}\label{eq:single_filter_loss}
L^{\mathrm S}(u;x,y) \coloneqq -y\,\sigma(\langle u,x\rangle),
\end{equation}
and its population counterpart
\[
\Psi(u)\coloneqq\E[L^{\mathrm S}(u;x,y)].
\]
By linearity of the untied predictor,
\begin{equation}\label{eq:untied_component_gradient}
\nabla_{w^j}L^{\mathsf U}(W;x,y) = \frac1{\sqrt s}\,\nabla L^{\mathrm S}(w^j;x,y), \qquad 1\le j\le s.
\end{equation}
For the augmented procedure, using
\(\langle w,\Pi^k x\rangle=\langle \Pi^{-k} w,x\rangle\) and the fact that \(k\mapsto-k\pmod r\) permutes the orbit,
\begin{equation} \label{eq:augmented_loss_tied_decomposition}
L^{\mathsf A}(W;x,y) =-\frac{y}{r\sqrt s} \sum_{j=1}^s\sum_{k=0}^{r-1} \sigma(\langle w^j,\Pi^k x\rangle) = \frac1{\sqrt{rs}}\sum_{j=1}^s L^{\mathsf T}(w^j;x,y).
\end{equation}
Hence
\begin{equation}\label{eq:augmented_component_gradient}
\nabla_{w^j} L^{\mathsf A}(W;x,y) = \frac1{\sqrt{rs}}\,\nabla L^{\mathsf T}(w^j;x,y).
\end{equation}

Using the step sizes in~\eqref{eq:student_stepsizes} and the identities in~\eqref{eq:untied_component_gradient} and~\eqref{eq:augmented_component_gradient}, we have
\begin{align}
\widetilde w_t^{j,\mathsf U} &= w_{t-1}^{j,\mathsf U} -\frac{\eta_d}{\sqrt r}\nabla L^{\mathrm S}(w_{t-1}^{j,\mathsf U};x_t,y_t),\label{eq:untied_effective_update}\\ \widetilde w_t^{j,\mathsf A} &= w_{t-1}^{j,\mathsf A} -\eta_d \nabla L^{\mathsf T}(w_{t-1}^{j,\mathsf A};x_t,y_t).\label{eq:augmented_effective_update}
\end{align}
The normalization to the sphere is then performed componentwise as in~\eqref{eq:SGD}. Note that~\eqref{eq:augmented_effective_update} is an exact identity: conditionally on its initialization, every augmented weight follows precisely the tied SGD trajectory with reference step size \(\eta_d\), driven by the same stream of fresh samples.

\begin{proof}[Proof of part~\textup{(A)} of Theorem~\ref{thm:weak_recovery_appendix}]
Fix any weight index $1\le j\le s$.
By~\eqref{eq:augmented_effective_update}, the trajectory of $w_t^{j,\mathsf A}$ has the same law as the tied trajectory analyzed in part~\textup{(T)}, with an independent uniform initialization. Therefore, with probability $1-o(1)$, there exists
\[
t\le\frac{C}{\eta_d\gamma_d^{p-2}}
\]
such that $M_\sigma(w_t^{j,\mathsf A})\ge\rho$.
This proves the required bound on
$\tau_\rho^{j,\mathsf A}$.
\end{proof}

It remains to analyze the unaugmented untied student. Again it is enough to follow one weight, which we denote by \(u_t\coloneqq w_t^{1,\mathsf U}\).

\subsection{One-weight untied dynamics}

The population loss of the single-weight model is
\begin{equation}\label{eq:single_filter_population}
\Psi(u) = - \E \bigl[y\,\sigma(\langle u,x\rangle)\bigr] = -\frac1{\sqrt r}\sum_{k=0}^{r-1}\phi(m_k(u)).
\end{equation}
Thus the population landscape is the same scalar overlap landscape as for the tied model, but its magnitude is smaller by a factor \(r^{-1/2}\).

Let
\[
\mathcal F_t^{\mathrm S} \coloneqq \sigma\bigl(w^\ast,u_0,(x_\ell,y_\ell)_{1\le \ell\le t}\bigr)
\]
be the natural filtration of the one-weight process. Set
\[
G_t^{\mathrm S} \coloneqq \nabla L^{\mathrm S}(u_{t-1};x_t,y_t), \qquad H_t^{\mathrm S}(u) \coloneqq L^{\mathrm S}(u;x_t,y_t)-\Psi(u),
\]
and define
\[
\xi_{k,t}^{\mathrm S} \coloneqq -\langle \Pi^k w^\ast, \nabla H_t^{\mathrm S}(u_{t-1})\rangle.
\]
Exactly as in Lemma~\ref{lem:population_drift}, differentiating~\eqref{eq:single_filter_population} gives
\begin{equation}\label{eq:untied_population_drift}
-\langle \Pi^kw^\ast,\nabla\Psi(u)\rangle = \frac1{\sqrt r}\bigl(\phi'(m_k(u))+\Gamma_k(u)\bigr),
\end{equation}
where \(\Gamma_k\) is the same function as in~\eqref{eq:general_Gamma}. Using the effective update~\eqref{eq:untied_effective_update}, the exact overlap recursion is therefore
\begin{equation}\label{eq:untied_overlap_exact}
\sqrt{1+\frac{\eta_d^2}{r}\|G_t^{\mathrm S}\|^2}\,m_k(u_t) = m_k(u_{t-1})+\frac{\eta_d}{r}\bigl(\phi'(m_k(u_{t-1}))+\Gamma_k(u_{t-1})\bigr) +\frac{\eta_d}{\sqrt r}\,\xi_{k,t}^{\mathrm S}.
\end{equation}
Consequently,
\begin{equation}\label{eq:untied_micro_decomposition}
m_k(u_t) = m_k(u_0) + \frac{p a_p^2\eta_d}{r}\sum_{\ell=1}^t m_k(u_{\ell-1})^{p-1} + E_{k,t}^{\mathsf U},
\end{equation}
where
\begin{equation}\label{eq:untied_micro_error}
\begin{split}
E_{k,t}^{\mathsf U} \coloneqq{}& \frac{\eta_d}{r}\sum_{\ell=1}^t R_\sigma(m_k(u_{\ell-1})) + \frac{\eta_d}{r}\sum_{\ell=1}^t \Gamma_k(u_{\ell-1})\\ &+ \frac{\eta_d}{\sqrt r}\sum_{\ell=1}^t\xi_{k,\ell}^{\mathrm S} - \sum_{\ell=1}^t\left(\sqrt{1+\frac{\eta_d^2}{r}\|G_\ell^{\mathrm S}\|^2}-1\right)m_k(u_\ell).
\end{split}
\end{equation}

The stochastic estimates needed for~\eqref{eq:untied_micro_error} are simpler than in the tied model.

\begin{lemma}[Single-weight stochastic bounds]\label{lem:single_filter_stochastic}
Condition on a teacher satisfying the geometry event of Lemma~\ref{lem:teacher_geometry}. There exists \(C_\sigma<\infty\) such that, uniformly over \(u\in\mathbb S^{d-1}\), \(k<r\), and \(t\ge1\),
\begin{align}
\E[(\xi_{k,t}^{\mathrm S})^2\mid\mathcal F_{t-1}^{\mathrm S}]&\le C_\sigma,\label{eq:single_filter_noise_var}\\
\left(\E[|\xi_{k,t}^{\mathrm S}|^a\mid\mathcal F_{t-1}^{\mathrm S}]\right)^{1/a} &\le C_\sigma a^P, \qquad a\ge2.\label{eq:single_filter_noise_mom}
\end{align}
Moreover, for every fixed \(K<\infty\), there exists \(C_{\sigma,K}<\infty\) such that, for every deterministic horizon \(T\le d^K\),
\begin{equation}\label{eq:single_filter_gradient_bound}
\mathbb P\left(\max_{1\le t\le T}\|G_t^{\mathrm S}\|>C_{\sigma,K}\sqrt d\,(\log d)^P\right) = o(1).
\end{equation}
\end{lemma}

\begin{proof}
Fix \(u\in\mathbb S^{d-1}\) and set \(v=(\Id_d-uu^\top)\Pi^kw^\ast\), so \(\|v\|\le1\). Then
\[
\langle \Pi^kw^\ast,G^{\mathrm S}\rangle = -y\,\sigma'(\langle u,x\rangle)\langle v,x\rangle.
\]
On the teacher-geometry event, \(\|y\|_{L^a}\le C_\sigma a^{P/2}\) for every \(a\ge2\), by the same \(L^2\) bound and Gaussian hypercontractivity used in Lemma~\ref{lem:coordinate_noise_moments}. The factor \(\sigma'(\langle u,x\rangle)\langle v,x\rangle\) is a Gaussian polynomial of degree at most \(P\) with uniformly bounded \(L^2\)-norm; hence hypercontractivity and Cauchy--Schwarz give
\[
\|\langle \Pi^kw^\ast,G^{\mathrm S}\rangle\|_{L^a} \le C_\sigma a^P.
\]
Centering can only decrease the variance and changes the \(L^a\)-norm by at most a factor two, proving~\eqref{eq:single_filter_noise_var}--\eqref{eq:single_filter_noise_mom}.

For the gradient norm,
\[
G^{\mathrm S} = -y\,\sigma'(\langle u,x\rangle)(\Id_d-uu^\top)x.
\]
Choosing an orthonormal basis of \(u^\perp\), applying the preceding polynomial moment estimate coordinate-wise, and using Minkowski's inequality yields
\[
\left(\E[\|G^{\mathrm S}\|^a]\right)^{1/a} \le C_\sigma\sqrt d\,a^P, \qquad a\ge2,
\]
uniformly in \(u\). Taking \(a\asymp\log d\), applying Markov's inequality conditionally on the past, and then taking a union bound over \(t\le T\le d^K\) proves~\eqref{eq:single_filter_gradient_bound}.
\end{proof}

\subsection{Proof of the untied weak recovery bound}

We now compare~\eqref{eq:untied_micro_decomposition} with the tied decomposition~\eqref{eq:micro_decomposition}. The deterministic drift has been reduced by a factor \(r\), while the centered stochastic increment has been reduced by a factor \(\sqrt r\). These two changes exactly balance when the time horizon is multiplied by \(r\).

Fix the same scales \(\alpha_d,\beta_d\) as in the tied proof, set \(\nu_d=(\log d)^{-2}\), and define
\[
T_0^{\mathsf U} \coloneqq \left\lceil\frac{C_0r}{\eta_d\gamma_d^{p-2}}\right\rceil, \qquad \widehat\tau_{\beta_d}^{\mathsf U} \coloneqq \inf\left\{t\ge0:\max_{k<r}|m_k(u_t)|\ge\beta_d\right\}.
\]
The initialization Lemma~\ref{lem:initialization_gap} applies unchanged to \(u_0\).

We first record the analogue of Lemma~\ref{lem:micro_error}.

\begin{lemma}[Untied microscopic error]\label{lem:untied_micro_error}
For \(A>A_\ast(p,P)\) sufficiently large,
\[
\max_{k<r}\max_{t\le \widehat\tau_{\beta_d}^{\mathsf U}\wedge T_0^{\mathsf U}}|E_{k,t}^{\mathsf U}| = o(\nu_d\gamma_d),
\]
with probability \(1-o(1)\).
\end{lemma}

\begin{proof}
The higher-Hermite and interaction terms are identical in size to their tied counterparts because
\[
\frac{\eta_d}{r}T_0^{\mathsf U} \asymp \gamma_d^{-(p-2)}.
\]
Thus the calculations in the proof of Lemma~\ref{lem:micro_error} apply verbatim.

For the martingale term, Lemma~\ref{lem:single_filter_stochastic} and the same truncation--Freedman argument as in Lemma~\ref{lem:martingale} give
\[
\max_{k<r}\max_{t\le \widehat\tau_{\beta_d}^{\mathsf U}\wedge T_0^{\mathsf U}}\left|\frac{\eta_d}{\sqrt r}\sum_{\ell=1}^t\xi_{k,\ell}^{\mathrm S}\right| \le C\frac{\eta_d}{\sqrt r}\left(\sqrt{T_0^{\mathsf U}\log d} + (\log d)^{P+1}\right)
\]
with probability \(1-o(1)\). Since
\[
\frac{\eta_d}{\sqrt r}\sqrt{T_0^{\mathsf U}\log d} \asymp \sqrt{\frac{\eta_d\log d}{\gamma_d^{p-2}}},
\]
this is exactly the leading martingale scale in the tied proof, while the truncation term is smaller by \(r^{-1/2}\). Hence the whole term is \(o(\nu_d\gamma_d)\).

Finally, on the gradient event~\eqref{eq:single_filter_gradient_bound},
\[
\sqrt{1+\frac{\eta_d^2}{r}\|G_\ell^{\mathrm S}\|^2}-1 \le \frac{\eta_d^2}{2r}\|G_\ell^{\mathrm S}\|^2.
\]
Therefore the cumulative normalization error is bounded by
\[
C\frac{\eta_d^2}{r}T_0^{\mathsf U}d(\log d)^{2P}\beta_d,
\]
which coincides with the tied bound after using the definition of \(T_0^{\mathsf U}\). This is again \(o(\nu_d\gamma_d)\). Combining the four estimates proves the claim.
\end{proof}

If \(p\) is odd, the analogue of Corollary~\ref{cor:stopping_times_agree} also holds: before \(T_0^{\mathsf U}\), no coordinate can be the first to leave the microscopic region through \(-\beta_d\). Indeed, the leading drift in~\eqref{eq:untied_micro_decomposition} is nonnegative and Lemma~\ref{lem:untied_micro_error} gives a cumulative error \(o(\gamma_d)\), whereas \(\gamma_d=o(\beta_d)\).

On this event, the scalar comparison proof of Proposition~\ref{prop:microscopic_selection} applies without any structural change, with \(\eta_d\) in the leading scalar recursion replaced by \(\eta_d/r\). Consequently, for a unique favorable index \(i_\ast\),
\begin{equation}\label{eq:untied_micro_result}
\tau_{\beta_d}^{(1),\mathsf U} \le \frac{C_\beta r}{\eta_d\gamma_d^{p-2}}, \qquad \beta_d \le \widetilde m_{i_\ast}(u_{\tau_{\beta_d}^{(1),\mathsf U}}) \le 2\beta_d,
\end{equation}
and
\[
\max_{k\ne i_\ast}|m_k(u_{\tau_{\beta_d}^{(1),\mathsf U}})| \le C_\alpha\alpha_d,
\]
with probability \(1-o(1)\). Here \(\tau_a^{(1),\mathsf U}\coloneqq\inf\{t:M_\sigma(u_t)\ge a\}\).

The continuation argument rescales in the same way. Starting from the configuration in~\eqref{eq:untied_micro_result}, define
\[
T^{\mathsf U} \coloneqq \left\lceil\frac{C_T r}{\eta_d\beta_d^{p-2}}\right\rceil.
\]
The selected population drift in~\eqref{eq:untied_overlap_exact} is \(r^{-1}\) times the corresponding drift per tied SGD step, the martingale increment is multiplied by \(r^{-1/2}\), and the normalization coefficient by \(r^{-1}\). The cumulative scales over \(T^{\mathsf U}\) are therefore the same as, or smaller than, those in the tied continuation proof. Indeed,
\[
\frac{\eta_d}{r}T^{\mathsf U}\asymp\beta_d^{-(p-2)}, \qquad \frac{\eta_d}{\sqrt r}\sqrt{T^{\mathsf U}\log d} \asymp \sqrt{\frac{\eta_d\log d}{\beta_d^{p-2}}},
\]
and
\[
\frac{\eta_d^2}{r}T^{\mathsf U} \asymp \frac{\eta_d}{\beta_d^{p-2}}.
\]
Thus the deterministic forcing, martingale fluctuations, and normalization errors satisfy exactly the bounds used in Proposition~\ref{prop:macroscopic_growth}; the selected-coordinate noise is in fact simpler because~\eqref{eq:single_filter_noise_var} is uniform in the overlap. The damping argument for the competitors is unchanged after replacing each occurrence of \(\eta_d\) in the population drift by \(\eta_d/r\). Consequently, the selected overlap reaches \(\rho\) within an additional
\[
\frac{C_\rho r}{\eta_d\beta_d^{p-2}} = o\left(\frac{r}{\eta_d\gamma_d^{p-2}}\right)
\]
steps, while all competitors remain microscopic. Therefore
\begin{equation}\label{eq:untied_one_filter_recovery}
\tau_\rho^{(1),\mathsf U} \le \frac{Cr}{\eta_d\gamma_d^{p-2}}
\end{equation}
with probability \(1-o(1)\).

\begin{proof}[Proof of part~\textup{(U)} of
Theorem~\ref{thm:weak_recovery_appendix}]
For any fixed $1\le j\le s$, the $j$th untied weight has the same trajectory law as the first weight analyzed above. Thus~\eqref{eq:untied_one_filter_recovery} gives
\[
\tau_\rho^{j,\mathsf U} \le \frac{Cr}{\eta_d\gamma_d^{p-2}}
\]
with probability \(1-o(1)\). Since \(\gamma_d=\sqrt{\delta\log d/d}\), this is
\[
O\left(r d^{p-1}(\log d)^{A-(p-2)/2}\right),
\]
which proves the theorem.
\end{proof}

\subsection{Proof of weak orbit coverage}
\label{sec:coverage_proof}

We now turn to the proof of Corollary~\ref{cor:weak_orbit_coverage}, which describes how full weak orbit coverage is achieved by the augmented and untied students given enough weights.

Retain the assumptions and learning rates of Theorem~\ref{thm:weak_recovery_appendix}, and fix $\rho\in(0,1)$.
The preceding untied argument already shows that any fixed weight recovers its initial leader, by~\eqref{eq:untied_micro_result} and the continuation leading to~\eqref{eq:untied_one_filter_recovery}. We now quantify its stochastic failure probability to apply this conclusion simultaneously to all weights satisfying the initialization conditions.

Recall the teacher-geometry event $\mathcal G_d$ from Lemma~\ref{lem:teacher_geometry}. Moreover, recall the initialization conditions in Lemmas~\ref{lem:initialization_scale} and~\ref{lem:initialization_gap}: for a weight $u$ with leader $i_\ast=\arg\max_{k<r}q_k(u)$, 
\begin{equation}\label{eq:coverage_initialization}
\max_{k<r}|m_k(u)|\le C\gamma_d, \qquad q_{i_\ast}(u)\ge c\gamma_d, \qquad \max_{k\ne i_\ast}q_k(u)\le(1-\nu_d)q_{i_\ast}(u),
\end{equation}
where $\nu_d=(\log d)^{-2}$. Let $\mathcal E_d^j$ denote the event that~\eqref{eq:coverage_initialization} holds for $u=w_0^{j,\mathsf U}$, and write $i_\ast(j)$ for its initial leader. We break ties by the smallest index; on $\mathcal E_d^j$ the leader is unique. For a constant $C_U>0$, set
\[
T_{\mathsf U} \coloneqq \left\lceil\frac{C_Ur}{\eta_d\gamma_d^{p-2}}\right\rceil.
\]

\begin{lemma}[Simultaneous untied recovery of initial leaders]
\label{lem:coverage_untied_recovery}
There exists $C_U=C_U(\sigma,\delta,\rho)<\infty$ such that, for every fixed $D>0$ and all sufficiently large $d$, uniformly over $w^\ast\in\mathcal G_d$,
\[
\mathbb P\left(\bigcup_{j=1}^s\left(\mathcal E_d^j\cap\left\{\max_{0\le t\le T_{\mathsf U}}q_{i_\ast(j)}(w_t^{j,\mathsf U})<\rho\right\}\right)\,\middle|\,w^\ast\right)\le s d^{-D}.
\]
In particular, taking $D>\kappa$ in Assumption~\ref{hyp:scaling}, with conditional probability $1-o(1)$ every weight satisfying \eqref{eq:coverage_initialization} recovers its initial leader by $T_{\mathsf U}$.
\end{lemma}

\begin{proof}
Fix a teacher in $\mathcal G_d$ and a weight index $j$, and condition further on an initialization $u_0=w_0^{j,\mathsf U}$ satisfying
\eqref{eq:coverage_initialization}. By~\eqref{eq:untied_effective_update}, its trajectory $u_t=w_t^{j,\mathsf U}$ follows the one-weight dynamics.
The samples remain fresh under this conditioning by Assumption~\ref{hyp:teacher_init}. The conditional moment estimates of Lemma~\ref{lem:single_filter_stochastic} and its proof, together with the truncation argument in the proof of Lemma~\ref{lem:martingale}, give an event of conditional probability at least $1-d^{-D}$ on which
\begin{equation}\label{eq:coverage_untied_martingale}
\frac{\eta_d}{\sqrt r}\max_{\ell<r}\max_{0\le n\le T_{\mathsf U}}\left|\sum_{t=1}^n\xi_{\ell,t}^{\mathrm S}\right| \le C_{\sigma,D}\frac{\eta_d}{\sqrt r}\left(\sqrt{T_{\mathsf U}\log d}+(\log d)^{P+1}\right) = o(\nu_d\gamma_d),
\end{equation}
and
\[
\max_{1\le t\le T_{\mathsf U}}\|G_t^{\mathrm S}\| \le C_{\sigma,D}\sqrt d\,(\log d)^P.
\]
Indeed, choose the moment orders proportional to $\log d$ so that the truncation, concentration, and gradient estimates have combined failure probability at most
\[
C_{\sigma,b}r(T_{\mathsf U}+1)d^{-b}\le d^{-D},
\]
where $b>D+p+1$ is fixed.
The last inequality holds for sufficiently large $d$, since $T_{\mathsf U}\le d^p$ and $r\le d$. The estimates hold over the full horizon because the untied conditional moments are uniform over the current weight. The little-$o$ bound in \eqref{eq:coverage_untied_martingale} follows from the same calculation as in the proof of Lemma~\ref{lem:untied_micro_error}.

Any noise sum over an interval inside $[0,T_{\mathsf U}]$ is the difference of two partial sums, so its magnitude is at most twice the maximum in~\eqref{eq:coverage_untied_martingale}. Thus the same event also controls the noise after the microscopic stopping time. The deterministic error bounds from Lemma~\ref{lem:untied_micro_error} and the comparisons leading to~\eqref{eq:untied_micro_result} and~\eqref{eq:untied_one_filter_recovery} now apply with $i_\ast=i_\ast(j)$. They bring this initial leader to $\rho$ within
\[
\frac{C_\beta r}{\eta_d\gamma_d^{p-2}} + \frac{C_\rho r}{\eta_d\beta_d^{p-2}} \le T_{\mathsf U},
\]
for sufficiently large $C_U$ and $d$, since $\beta_d/\gamma_d\to\infty$.
Here $\beta_d$ is the microscopic threshold from Section~\ref{sec:weak_recovery_tied}, also used in Lemma~\ref{lem:untied_micro_error}. Increasing $D$ changes the concentration constants and the dimension threshold, but neither the little-$o$ comparisons nor the choice of $C_U$.

The estimates are uniform over the conditioned initialization. Integrating over $u_0$ therefore gives, for every $j$,
\[
\mathbb P\left(\mathcal E_d^j\cap\left\{\max_{0\le t\le T_{\mathsf U}}q_{i_\ast(j)}(w_t^{j,\mathsf U})<\rho\right\}\,\middle|\,w^\ast\right)\le d^{-D}.
\]
A union bound over $j\le s$ proves the lemma. This step requires no independence between the trajectories, which share the sample sequence.
\end{proof}

Fix $\varepsilon>0$ and suppose $s\ge(1+\varepsilon)r\log r$. Recall the relevant correlations $q_k$ from \eqref{eq:relevant_correlation} and the coverage events $\mathcal C_n^{\mathsf k}(\rho)$ from Definition~\ref{def:weak_orbit_coverage}. Use the horizon $T_{\mathsf U}$ above, with $C_U$ chosen according to Lemma~\ref{lem:coverage_untied_recovery}.
Choose $C_A$ as in part~\textup{(A)} of Theorem~\ref{thm:weak_recovery_appendix}, and set
\[
T_{\mathsf A} \coloneqq \left\lceil\frac{C_A}{\eta_d\gamma_d^{p-2}}\right\rceil.
\]

\begin{proof}[Proof of Corollary~\ref{cor:weak_orbit_coverage}]
We first prove the coverage upper bounds. According to~\eqref{eq:correlation} and~\eqref{eq:relevant_correlation}, we have the shift identity
\begin{equation}\label{eq:coverage_shift_correlations}
q_k(\Pi^a u)=q_{k-a}(u), \qquad 0\le a,k<r,
\end{equation}
where indices are understood modulo $r$.

\medskip
\noindent\textbf{Untied procedure.}
Condition on a teacher $w^\ast\in\mathcal G_d$. Using the initialization events $\mathcal E_d^j$, write
\[
p_{\mathrm{init}}\coloneqq\mathbb P(\mathcal E_d^1\mid w^\ast)=1-o(1).
\]
The last equality follows from the joint proof of Lemmas~\ref{lem:initialization_scale} and~\ref{lem:initialization_gap} in Appendix~\ref{sec:teacher_geometry}, uniformly over $w^\ast\in\mathcal G_d$.

By~\eqref{eq:coverage_shift_correlations}, shifting an initial weight
by $\Pi^a$ preserves~\eqref{eq:coverage_initialization} and shifts its
unique leader by $a$. Since uniform initialization on the sphere is
invariant under these shifts, for every $j\le s$ and $k<r$,
\[
\mathbb P\bigl(\mathcal E_d^j\cap\{i_\ast(j)=k\}\mid w^\ast\bigr) = \frac{p_{\mathrm{init}}}{r}.
\]

The initializations are independent by Assumption~\ref{hyp:teacher_init}. Consequently, for a fixed teacher direction $k$, the probability that no initial weight satisfies~\eqref{eq:coverage_initialization} and has leader $k$ is $(1-p_{\mathrm{init}}/r)^s$. Taking a union bound over $k<r$, the probability that at least one teacher direction has no such initial weight is at most
\[
r\left(1-\frac{p_{\mathrm{init}}}{r}\right)^s \le r\exp\left(-\frac{s p_{\mathrm{init}}}{r}\right) \le r^{-\varepsilon+o(1)} = o(1).
\]

If every teacher direction has such an initial weight and all these weights recover their initial leaders, full orbit coverage holds. Choose $D>\kappa$, so $sd^{-D}=o(1)$ by Assumption~\ref{hyp:scaling}. Combining the preceding bound on missing initial leaders with the recovery failure bound in Lemma~\ref{lem:coverage_untied_recovery} gives
\[
\mathbb P\left((\mathcal C_{T_{\mathsf U}}^{\mathsf U}(\rho))^c\right) \le \mathbb P(\mathcal G_d^c)+r^{-\varepsilon+o(1)}+s d^{-D}=o(1),
\]
The three terms account, respectively, for failure of the teacher geometry event, a teacher direction with no suitable initial weight, and a recovery failure among weights satisfying \eqref{eq:coverage_initialization}. The first term tends to zero by Lemma~\ref{lem:teacher_geometry}.

\medskip
\noindent\textbf{Augmented procedure.}
Here we condition on the teacher and the samples through the recovery horizon. Define
\[
\mathcal D_{\mathsf A} \coloneqq \sigma\bigl(w^\ast,(x_t,y_t)_{1\le t\le T_{\mathsf A}}\bigr).
\]
Under this conditioning, the initial weights remain independent and uniform by Assumption~\ref{hyp:teacher_init}. The reduction~\eqref{eq:augmented_effective_update} therefore makes their trajectories conditionally independent and identically distributed.

For every fixed sample, the tied loss from Definition~\ref{def:tied_student} satisfies
\[
L^{\mathsf T}(\Pi^a u;x,y)=L^{\mathsf T}(u;x,y), \qquad \nabla L^{\mathsf T}(\Pi^a u;x,y) = \Pi^a\nabla L^{\mathsf T}(u;x,y).
\]
The normalization in~\eqref{eq:SGD} also commutes with $\Pi^a$. Hence, by~\eqref{eq:augmented_effective_update}, shifting an initial weight shifts its entire augmented trajectory by the same amount. Together with~\eqref{eq:coverage_shift_correlations}, this implies that the conditional probability of recovering direction $k$ by time $T_{\mathsf A}$ is the same for every $k$.

Using the orbit recovery time defined in \eqref{eq:directional_hitting_time}, let
\[
p_{\mathrm{rec}} \coloneqq \mathbb P\left(\tau_\rho^{1,\mathsf A}\le T_{\mathsf A}\,\middle|\,\mathcal D_{\mathsf A}\right).
\]
By~\eqref{eq:orbit_corr} and~\eqref{eq:relevant_correlation},
$M_\sigma(u)=\max_{k<r}q_k(u)$.
The event defining $p_{\mathrm{rec}}$ is therefore the union of the $r$ direction-recovery events. Since these events have equal conditional probabilities, a union bound gives
\[
\mathbb P\left(\max_{0\le t\le T_{\mathsf A}}q_k(w_t^{1,\mathsf A})\ge\rho\,\middle|\,\mathcal D_{\mathsf A}\right) \ge \frac{p_{\mathrm{rec}}}{r}, \qquad 0\le k<r.
\]
This bound allows a single weight to recover more than one direction.

For any fixed direction, conditional independence across weights bounds the probability that all $s$ trajectories miss it by $(1-p_{\mathrm{rec}}/r)^s$. A union bound over directions consequently gives
\[
\mathbb P\left(\bigl(\mathcal C_{T_{\mathsf A}}^{\mathsf A}(\rho)\bigr)^c\,\middle|\,\mathcal D_{\mathsf A}\right) \le r\left(1-\frac{p_{\mathrm{rec}}}{r}\right)^s \le r\exp\left(-\frac{s p_{\mathrm{rec}}}{r}\right).
\]

We now remove the conditioning.
Part~\textup{(A)} of Theorem~\ref{thm:weak_recovery_appendix} and the tower property imply
\[
\mathbb E[1-p_{\mathrm{rec}}] = \mathbb P\left(\tau_\rho^{1,\mathsf A}>T_{\mathsf A}\right) = o(1).
\]
Fix any constant $\theta\in((1+\varepsilon)^{-1},1)$. On $\{p_{\mathrm{rec}}\ge\theta\}$, the conditional coverage failure probability is at most $r^{1-(1+\varepsilon)\theta}$. On the complementary event, Markov's inequality gives
\[
\mathbb P(p_{\mathrm{rec}}<\theta) \le \frac{\mathbb E[1-p_{\mathrm{rec}}]}{1-\theta}.
\]
Averaging the conditional coverage estimate thus yields
\[
\mathbb P\left(\bigl(\mathcal C_{T_{\mathsf A}}^{\mathsf A}(\rho)\bigr)^c\right) \le \frac{\mathbb E[1-p_{\mathrm{rec}}]}{1-\theta} + r^{1-(1+\varepsilon)\theta} = o(1).
\]

\medskip
\noindent\textbf{Sample complexity and matching lower bounds.}
The preceding estimates prove coverage by $T_{\mathsf U}$ and $T_{\mathsf A}$ for the respective procedures. Different directions may be recovered at different times, as allowed by Definition~\ref{def:weak_orbit_coverage}. Each iteration of~\eqref{eq:SGD} uses one fresh observation. Using $\eta_d=d^{-p/2}(\log d)^{-A}$ from Theorem~\ref{thm:weak_recovery_appendix} and $\gamma_d=\sqrt{\log r/d}$ from Section~\ref{sec:teacher_geometry}, the horizons satisfy
\[
T_{\mathsf A}=\widetilde O(d^{p-1}), \qquad T_{\mathsf U}=\widetilde O(rd^{p-1}).
\]
For the lower bounds, note that by Definitions~\ref{def:weak_recovery} and~\ref{def:weak_orbit_coverage}, we have
\[
\mathcal C_n^{\mathsf k}(\rho) \subseteq\mathcal R_n^{\mathsf k}(\rho), \qquad \mathsf k\in\{\mathsf U,\mathsf A\}.
\]
Theorem~\ref{thm:nonescape} therefore supplies the matching lower bounds, giving
\[
n_{\mathsf A}^{\mathrm{cov}}(\rho)=\widetilde\Theta(d^{p-1}), \qquad n_{\mathsf U}^{\mathrm{cov}}(\rho)=\widetilde\Theta(rd^{p-1}),
\]
which proves the corollary.
\end{proof}

\section{Quadratic link and weak subspace recovery}  \label{sec:weak_recovery_quadratic}

In this section, we prove Theorem~\ref{thm:h2_weak_recovery_appendix} with $\sigma=h_2$. We retain the assumptions and recovery times from Section~\ref{sec:appendix_roadmap}, with the learning-rate scaling in~\eqref{eq:student_stepsizes}. We first analyze the tied trajectory, using the notation of Section~\ref{sec:corr_dynamics}, and then treat the augmented and untied procedures using the reductions from Section~\ref{sec:weak_recovery_untied}.

\subsection{Proof strategy}
\label{sec:h2_proof_strategy}

For $p\ge3$, the microscopic argument in Section~\ref{sec:weak_recovery_tied} amplifies the initial
gap between the leading correlation and its competitors. For the pure quadratic link, the leading coordinate drift is linear, and we instead study the growth of the total mass in the teacher orbit subspace.

Recall that $w^\parallel$ denotes the orthogonal projection of $w$ onto $\mathcal S^\ast$. We track its squared norm through the total overlap
\[
R(w)\coloneqq\sum_{k=0}^{r-1}m_k(w)^2, \qquad R_t\coloneqq R(w_t).
\]
Subsection~\ref{sec:h2_geometry_initialization} shows that, on the teacher-geometry event $\mathcal G_d$ defined in Lemma~\ref{lem:teacher_geometry},
\[
R(w)=(1+o(1))\|w^\parallel\|^2
\]
uniformly over unit weights, and $R_0\asymp r/d$ with probability $1-o(1)$. Thus weak subspace recovery follows once $R_t$ reaches a macroscopic level.

The identity $\Phi=-R$ gives a direct description of the population growth: the population contribution to the one-step increase of $R_t$ is $\eta_d\|\nabla\Phi(w_{t-1})\|^2$. Below any fixed threshold $a<1$, this quantity is comparable to $\eta_d R_{t-1}$ on $\mathcal G_d$. Subsection~\ref{sec:h2_overlap_dynamics} establishes this comparison and the exact stochastic recursion.

To turn this proportional drift into a growth estimate, we study $\log R_t$. Its increments have a positive drift of order $\eta_d$, together with relative noise and discretization errors. The relative noise involves division by $R_{t-1}$, so we first stop the process when $R_t$ either reaches $a$ or falls below a fixed fraction of $R_0$. Before this exit, the denominator remains bounded below at the initialization scale. Using the moment estimates from Subsection~\ref{sec:h2_noise_moments}, Subsection~\ref{sec:h2_tied_recovery} shows that the accumulated relative noise, logarithmic remainders, and spherical normalization errors are $o(1)$ over a horizon of order $\eta_d^{-1}\log(d/r)$. The resulting lower bound on $\log(R_t/R_0)$ rules out the lower exit and forces $R_t$ to reach $a$ within this horizon. Choosing $a\in(\rho,1)$ then yields weak subspace recovery.

Finally, Subsection~\ref{sec:h2_other_students} treats the other procedures. Under the prescribed learning rates, each augmented weight follows the tied update exactly. For an untied weight, the drift carries a factor $\eta_d/r$ and the centered stochastic increment a factor $\eta_d/\sqrt r$. The same comparison therefore applies over a horizon longer by a factor $r$, with the accumulated errors still negligible.

\subsection{Total overlap and initialization}
\label{sec:h2_geometry_initialization}

We work on the teacher-geometry event $\mathcal G_d$. Recall the teacher Gram matrix $W^\ast$ from the proof of Lemma~\ref{lem:initialization_gap} in Appendix~\ref{sec:teacher_geometry}. On this event,
\[
\|W^\ast-\Id_r\|_{\mathrm{op}} \le r\mu_\ast=o(1).
\]
To relate total overlap to projection mass, define
\[
B^\ast \coloneqq \sum_{k=0}^{r-1}(\Pi^kw^\ast)(\Pi^kw^\ast)^\top.
\]
This matrix vanishes on $(\mathcal S^\ast)^\perp$, and
\[
R(w) = \langle w,B^\ast w\rangle = \langle w^\parallel,B^\ast w^\parallel\rangle.
\]
For sufficiently large $d$, $W^\ast$ is positive definite. The restriction of $B^\ast$ to $\mathcal S^\ast$ has the same eigenvalues as $W^\ast$, all lying in $[1-r\mu_\ast,1+r\mu_\ast]$. Consequently, uniformly over $w\in\mathbb S^{d-1}$,
\begin{equation}
(1-r\mu_\ast)\|w^\parallel\|^2 \le R(w) \le (1+r\mu_\ast)\|w^\parallel\|^2.
\label{eq:h2_subspace_alignment}
\end{equation}
Thus it suffices to control the total overlap. We next establish its initial scale.

\begin{lemma}[Initial total overlap]
\label{lem:h2_initialization}
Condition on a teacher $w^\ast\in\mathcal G_d$, and let $w_0$ be uniform on $\mathbb S^{d-1}$, independently of the teacher. There exist constants $c,C>0$ such that, for all sufficiently large $d$ and every $\varepsilon\in(0,1)$,
\[
\mathbb P\left(\left|R_0-\frac rd\right|>\varepsilon\frac rd\,\middle|\,w^\ast\right) \le C\exp(-c\varepsilon^2r).
\]
The constants are uniform over $w^\ast\in\mathcal G_d$. In particular, $R_0\asymp r/d$ with conditional probability $1-o(1)$.
\end{lemma}

\begin{proof}
Write $w_0=h/\|h\|$, where $h\sim\mathcal N(0,\Id_d)$ is independent of the teacher. Then
\[
R_0=\frac{h^\top B^\ast h}{\|h\|^2}.
\]
The matrix $B^\ast$ is positive semidefinite and satisfies
\[
\tr(B^\ast)=r, \qquad \|B^\ast\|_{\mathrm{op}} = \|W^\ast\|_{\mathrm{op}} \le 1+r\mu_\ast\le C,
\]
and hence
\[
\|B^\ast\|_{\mathrm F}^2 = \tr((B^\ast)^2) \le \|B^\ast\|_{\mathrm{op}}\tr(B^\ast) \le Cr.
\]

For $\theta\in(0,1)$, apply the Hanson--Wright inequality~\eqref{eq:Hanson_Wright} to $B^\ast$ and $\Id_d$. A union bound gives
\[
\begin{split}
\mathbb P\Bigl(&|h^\top B^\ast h-r|>\theta r \ \text{or}\ \bigl|\|h\|^2-d\bigr|>\theta d \,\Bigm|\,w^\ast\Bigr)\\& \le C e^{-c\theta^2r}+C e^{-c\theta^2d} \le C e^{-c\theta^2r},
\end{split}
\]
where we used $r\le d$. Outside this event,
\[
\frac{1-\theta}{1+\theta}\frac rd \le R_0 \le \frac{1+\theta}{1-\theta}\frac rd.
\]
Taking $\theta$ to be a sufficiently small universal multiple of $\varepsilon$ proves the concentration bound. The final assertion follows by fixing $\varepsilon\in(0,1)$ and using $r=d^\delta\to\infty$.
\end{proof}

Combining Lemma~\ref{lem:h2_initialization}, \eqref{eq:h2_subspace_alignment}, and $\mathbb P(\mathcal G_d)=1-o(1)$ yields
\[
\|w_0^\parallel\|^2\asymp R_0\asymp\frac rd
\]
with probability $1-o(1)$. The same initialization estimate applies to any fixed weight of the untied or augmented student.

\subsection{Recursion and population growth}
\label{sec:h2_overlap_dynamics}

Recall that $R(w)=\langle w,B^\ast w\rangle$ and $\Phi=-R$. The spherical gradient calculation in Lemma~\ref{lem:population_drift}, specialized to $\phi(z)=z^2$, gives
\[
\nabla\Phi(w)=-2\bigl(B^\ast w-R(w)w\bigr).
\]
This identity allows us to express the population contribution to the growth of $R$ as the squared norm of the spherical gradient.

Using the coordinate increments from~\eqref{eq:stoch_fluctuation}, define
\[
\Xi_t \coloneqq \sum_{k=0}^{r-1}m_k(w_{t-1})\xi_{k,t} = -\langle B^\ast w_{t-1}, \nabla H_t(w_{t-1})\rangle.
\]
The coefficients $m_k(w_{t-1})$ are $\mathcal F_{t-1}$-measurable, so
\[
\E[\Xi_t\mid\mathcal F_{t-1}]=0.
\]

\begin{lemma}[Exact recursion for the total overlap]
\label{lem:h2_dynamics}
For every $t\ge1$,
\begin{equation}
(1+\eta_d^2\|G_t\|^2)R_t = R_{t-1}+\eta_d\|\nabla\Phi(w_{t-1})\|^2 + 2\eta_d\Xi_t + \eta_d^2\langle G_t,B^\ast G_t\rangle.
\label{eq:h2_energy_decomposition}
\end{equation}
\end{lemma}

\begin{proof}
Since $G_t\perp w_{t-1}$, the spherical update~\eqref{eq:SGD} yields
\[
(1+\eta_d^2\|G_t\|^2)R_t = R_{t-1} - 2\eta_d\langle B^\ast w_{t-1},G_t\rangle + \eta_d^2\langle G_t,B^\ast G_t\rangle.
\]
Using $G_t=\nabla\Phi(w_{t-1})+\nabla H_t(w_{t-1})$ and the spherical gradient formula above,
\[
\begin{aligned}
-2\langle B^\ast w_{t-1},G_t\rangle &=\langle\nabla\Phi(w_{t-1}),G_t\rangle\\ &=\|\nabla\Phi(w_{t-1})\|^2+2\Xi_t.
\end{aligned}
\]
Substitution proves the identity.
\end{proof}

The population contribution is therefore nonnegative. Teacher geometry gives the quantitative lower bound needed to amplify a small total overlap.

\begin{proposition}[Population growth of the total overlap]
\label{prop:h2_population_growth}
On the teacher-geometry event $\mathcal G_d$, for every fixed $a\in(0,1)$ there exist constants $c_a,C_a>0$, depending only on $a$, such that, for all sufficiently large $d$, uniformly over $w\in\mathbb S^{d-1}$ with $0<R(w)\le a$,
\[
c_aR(w) \le \|\nabla\Phi(w)\|^2 \le C_aR(w).
\]
\end{proposition}

\begin{proof}
Since $\langle w,B^\ast w\rangle=R(w)$ and $\|w\|=1$,
\begin{equation}
\|\nabla\Phi(w)\|^2 = 4\bigl(\|B^\ast w\|^2-R(w)^2\bigr).
\label{eq:h2_population_drift}
\end{equation}
By the spectral comparison established in Subsection~\ref{sec:h2_geometry_initialization}, the nonzero eigenvalues of $B^\ast$ belong to $[1-r\mu_\ast,1+r\mu_\ast]$. Hence,
\[
(1-r\mu_\ast)R(w) \le \|B^\ast w\|^2 \le (1+r\mu_\ast)R(w).
\]
Combining these bounds with \eqref{eq:h2_population_drift} gives
\[
4R(w)(1-R(w)-r\mu_\ast) \le \|\nabla\Phi(w)\|^2 \le 4R(w)(1-R(w)+r\mu_\ast).
\]
On $\mathcal G_d$, we have $r\mu_\ast=o(1)$ uniformly over the teacher. Since $R(w)\le a<1$, the claimed bounds follow for sufficiently large $d$.
\end{proof}

The matrix $B^\ast$ is positive semidefinite, so the last term in~\eqref{eq:h2_energy_decomposition} is nonnegative. Discarding it and applying Proposition~\ref{prop:h2_population_growth}, we obtain, on $\mathcal G_d$ and for sufficiently large $d$,
\begin{equation}
(1+\eta_d^2\|G_t\|^2)R_t \ge (1+c_a\eta_d)R_{t-1}+2\eta_d\Xi_t
\label{eq:h2_overlap_lower_recursion}
\end{equation}
whenever $0<R_{t-1}<a$.
The next subsection controls the centered increment $\Xi_t$ and the normalization term $\eta_d^2\|G_t\|^2$.

\subsection{Noise and gradient moments}
\label{sec:h2_noise_moments}

We work on the teacher-geometry event $\mathcal G_d$. The coordinate and gradient estimates from Sections~\ref{sec:corr_dynamics} and~\ref{sec:weak_recovery_untied}, specialized to $p=P=2$, give the stochastic bounds needed for \eqref{eq:h2_overlap_lower_recursion}.

For the untied procedure, recall the single-weight trajectory $u_t=w_t^{1,\mathsf U}$ and the quantities $G_t^{\mathrm S}$, $\xi_{k,t}^{\mathrm S}$, and $\mathcal F_t^{\mathrm S}$ from Section~\ref{sec:weak_recovery_untied}. Define
\[
\Xi_t^{\mathrm S} \coloneqq \sum_{k=0}^{r-1}m_k(u_{t-1})\xi_{k,t}^{\mathrm S}.
\]
As for the tied process, $\E[\Xi_t^{\mathrm S}\mid\mathcal F_{t-1}^{\mathrm S}]=0$.

\begin{corollary}[Noise and gradient moments]
\label{cor:h2_noise_moments}
There exists a constant $C>0$ such that, for all sufficiently large $d$, every $t\ge1$, and every $\nu\ge2$,
\begin{align}
\E[\Xi_t^2\mid\mathcal F_{t-1}] &\le CrR_{t-1},
\label{eq:h2_noise_variance}\\
\left(\E[|\Xi_t|^\nu\mid\mathcal F_{t-1}]\right)^{1/\nu} &\le C\nu^2\sqrt{rR_{t-1}},
\label{eq:h2_noise_moments}\\
\left(\E[\|G_t\|^\nu\mid\mathcal F_{t-1}]\right)^{1/\nu} &\le C\nu^2\sqrt d.
\label{eq:h2_gradient_moments}
\end{align}
The estimates are uniform over the current weight and over teachers in $\mathcal G_d$. They also hold for the untied single-weight process, with
\[
(G_t,\Xi_t,R_{t-1},\mathcal F_{t-1}) \quad \text{replaced by}\quad (G_t^{\mathrm S},\Xi_t^{\mathrm S}, R(u_{t-1}),\mathcal F_{t-1}^{\mathrm S}).
\]
\end{corollary}

\begin{proof}
By~\eqref{eq:h2_subspace_alignment},
\[
R(w)\le1+r\mu_\ast\le C
\]
uniformly over unit weights on $\mathcal G_d$, for sufficiently large $d$. 

Condition on $\mathcal F_{t-1}$. Conditional Minkowski's inequality and Lemma~\ref{lem:coordinate_noise_moments}, with $P=2$, give
\[
\begin{aligned}
\left(\E[|\Xi_t|^\nu\mid\mathcal F_{t-1}]\right)^{1/\nu} &\le C\nu^2\sum_{k=0}^{r-1}
|m_k(w_{t-1})|\bigl(1+\sqrt r\,|m_k(w_{t-1})|\bigr)\\ &\le C\nu^2\bigl(\sqrt{rR_{t-1}}+\sqrt r\,R_{t-1}\bigr)\\ &\le C\nu^2\sqrt{rR_{t-1}}.
\end{aligned}
\]
Here the second inequality uses Cauchy--Schwarz, and the last uses the uniform bound on $R$. This proves~\eqref{eq:h2_noise_moments}; taking $\nu=2$ also gives~\eqref{eq:h2_noise_variance}. The gradient estimate~\eqref{eq:h2_gradient_moments} is the conditional moment bound~\eqref{eq:gradient_high_moments} established in the proof of Lemma~\ref{lem:uniform_gradient} in Appendix~\ref{sec:uniform_gradient_proof}, specialized to $P=2$.

For the untied process, Lemma~\ref{lem:single_filter_stochastic} gives
\[
\left(\E[|\xi_{k,t}^{\mathrm S}|^\nu\mid\mathcal F_{t-1}^{\mathrm S}]\right)^{1/\nu} \le C\nu^2.
\]
Consequently, the same conditional Minkowski and Cauchy--Schwarz inequalities yield
\[
\left(\E[|\Xi_t^{\mathrm S}|^\nu\mid\mathcal F_{t-1}^{\mathrm S}]\right)^{1/\nu} \le C\nu^2\sum_{k=0}^{r-1}|m_k(u_{t-1})| \le C\nu^2\sqrt{rR(u_{t-1})}.
\]
Taking $\nu=2$ gives the corresponding variance bound. The gradient-moment calculation in the proof of Lemma~\ref{lem:single_filter_stochastic}, with $P=2$,
gives the asserted estimate for $G_t^{\mathrm S}$. All these bounds are uniform over the current weight and the conditioned teacher.
\end{proof}

We record two consequences for the growth argument. First, for every fixed $c_0>0$, whenever $R_{t-1}\ge c_0r/d$,
\begin{equation}
\left(
\E\left[\left|\frac{\Xi_t}{R_{t-1}}\right|^\nu\,\middle|\,\mathcal F_{t-1}\right]\right)^{1/\nu} \le C_{c_0}\nu^2\sqrt d, \qquad \nu\ge2.
\label{eq:h2_relative_noise_moments}
\end{equation}
Indeed, $R_{t-1}$ is measurable with respect to the conditioning, and $\sqrt{r/R_{t-1}}\le\sqrt{d/c_0}$. This estimate controls the relative fluctuations of the total overlap.

Second, for every fixed $B,D>0$, there exists $C_{B,D}<\infty$ such that, for every deterministic integer horizon $1\le T\le d^B$,
\begin{equation}
\mathbb P\left(\max_{1\le t\le T}\|G_t\|^2>C_{B,D}d(\log d)^4\,\middle|\,w^\ast,w_0\right) \le d^{-D}.
\label{eq:h2_uniform_gradient_norm}
\end{equation}
This holds for sufficiently large $d$, uniformly over $w^\ast\in\mathcal G_d$ and $w_0\in\mathbb S^{d-1}$. To see this, take $\nu$ to be a sufficiently large multiple of $\log d$ in~\eqref{eq:h2_gradient_moments}. Conditional Markov's inequality makes the failure probability at each step at most $d^{-B-D-1}$, and a union bound over $t\le T$ proves the claim.

Both~\eqref{eq:h2_relative_noise_moments} and~\eqref{eq:h2_uniform_gradient_norm} hold for the untied process with the substitutions in Corollary~\ref{cor:h2_noise_moments}, conditioning on $w^\ast,u_0$ in the latter bound.

\subsection{Growth and proof of the tied upper bound}
\label{sec:h2_tied_recovery}

We combine the preceding estimates to prove exponential growth of the total overlap until it reaches a fixed threshold. Fix $a,\theta,c_{\mathrm{init}}\in(0,1)$ and $C_0>0$, and define
\[
\tau_a^R\coloneqq\inf\{t\ge0:R_t\ge a\}, \qquad T_2\coloneqq \left\lceil\frac{C_0\log(d/r)}{\eta_d}\right\rceil.
\]
We first condition on a teacher $w^\ast\in\mathcal G_d$ and an initialization satisfying
\begin{equation}
R_0\ge c_{\mathrm{init}}\frac rd.
\label{eq:h2_initial_lower_bound}
\end{equation}
By Lemma~\ref{lem:h2_initialization}, a uniform independent initialization satisfies this condition with conditional probability at least $1-Ce^{-cr}$, uniformly over $w^\ast\in\mathcal G_d$.

\begin{proposition}[Growth of the total overlap]
\label{prop:h2_logarithmic_growth}
There exist a universal constant $A_0>0$ and a constant $c_a>0$, depending only on $a$, such that, for every fixed $A>A_0$ and $D>0$, the following holds for sufficiently large $d$.

For every fixed teacher $w^\ast\in\mathcal G_d$ and initialization satisfying~\eqref{eq:h2_initial_lower_bound}, with probability at least $1-d^{-D}$ over the SGD samples, simultaneously for every integer $0\le n\le T_2$,
\begin{equation}
R_{n\wedge\tau_a^R} \ge \theta R_0\exp\bigl(c_a\eta_d(n\wedge\tau_a^R)\bigr).
\label{eq:h2_total_overlap_growth}
\end{equation}
\end{proposition}

To apply the relative-noise estimate, fix $\beta\in(0,\theta)$ and introduce
\[
\widehat\tau_a^R \coloneqq \inf\{t\ge0:R_t\le\beta R_0\ \text{or}\ R_t\ge a\}.
\]
For $1\le t\le\widehat\tau_a^R$,
\[
R_{t-1}>\beta R_0 \ge \beta c_{\mathrm{init}}\frac rd,
\]
so~\eqref{eq:h2_relative_noise_moments} applies. We will show that the lower boundary cannot be reached before $T_2$.

\begin{lemma}[Stopped relative-noise bounds]
\label{lem:h2_relative_martingale}
Fix a teacher $w^\ast\in\mathcal G_d$ and an initialization satisfying~\eqref{eq:h2_initial_lower_bound}. For every fixed $D>0$, there exists $C_D<\infty$ such that, for sufficiently large $d$, with conditional probability at least $1-d^{-D}$ over the SGD samples,
\[
\max_{0\le n\le T_2}\left|2\eta_d\sum_{t=1}^{n\wedge\widehat\tau_a^R}\frac{\Xi_t}{R_{t-1}}\right| \le C_D\eta_d\sqrt{dT_2\log d},
\]
and, simultaneously,
\[
\left|\frac{\Xi_t}{R_{t-1}}\right| \le C_D\sqrt d(\log d)^2, \qquad 1\le t\le\widehat\tau_a^R\wedge T_2.
\]
The constant may depend on the fixed parameters, but is uniform over the conditioned teachers and initializations.
\end{lemma}

\begin{proof}
Fix the teacher and initialization as in the statement. All probabilities below are conditional on this pair.

\medskip
\noindent\textbf{Step 1: Stopping and conditional moments.}

For $1\le t\le T_2$, define
\[
Y_t\coloneqq
\begin{cases}
\displaystyle\frac{\Xi_t}{\sqrt d\,R_{t-1}},
&t \le \widehat\tau_a^R,\\[1ex]
0,&t>\widehat\tau_a^R.
\end{cases}
\]
The event $\{t\le\widehat\tau_a^R\}$ belongs to $\mathcal F_{t-1}$, and on this event,
\[
R_{t-1}>\beta R_0 \ge \beta c_{\mathrm{init}}\frac rd.
\]
Thus the denominator is positive and $\E[Y_t\mid\mathcal F_{t-1}]=0$. By~\eqref{eq:h2_relative_noise_moments}, there exists $K<\infty$, uniform over the conditioned teacher and initialization, such that
\[
\E[Y_t^2\mid\mathcal F_{t-1}] \le K, \qquad \left(\E[|Y_t|^\nu\mid\mathcal F_{t-1}]\right)^{1/\nu} \le K\nu^2, \qquad \nu\ge2.
\]

\medskip
\noindent\textbf{Step 2: Truncation and recentering.}

Fix $b>D+2$ and set
\[
q_d\coloneqq\lceil b\log d\rceil, \qquad L_d\coloneqq eKq_d^2.
\]
Conditional Markov's inequality gives
\[
\mathbb P\bigl(|Y_t|>L_d\mid\mathcal F_{t-1}\bigr) \le \frac{\E[|Y_t|^{q_d}\mid\mathcal F_{t-1}]}{L_d^{q_d}} \le e^{-q_d}.
\]
Since $T_2\asymp d(\log d)^{A+1}\le d^2$ for sufficiently large $d$, a union bound yields
\[
\mathbb P\left(\max_{1\le t\le T_2}|Y_t|>L_d\right) \le T_2e^{-q_d}.
\]

Define the truncated and recentered increments
\[
\widehat Y_t \coloneqq Y_t\mathbf1_{\{|Y_t|\le L_d\}} - \E\bigl[Y_t\mathbf1_{\{|Y_t|\le L_d\}}\mid\mathcal F_{t-1}\bigr].
\]
They are martingale differences and satisfy $|\widehat Y_t|\le2L_d$. To control the centering correction, use $\E[Y_t\mid\mathcal F_{t-1}]=0$ and the same moment bound:
\[
\begin{aligned}
\left|\E\bigl[Y_t\mathbf1_{\{|Y_t|\le L_d\}}\mid\mathcal F_{t-1}\bigr]\right| &= \left|\E\bigl[Y_t\mathbf1_{\{|Y_t|>L_d\}}\mid\mathcal F_{t-1}\bigr]\right|\\ &\le \frac{\E[|Y_t|^{q_d}\mid\mathcal F_{t-1}]}{L_d^{q_d-1}}\\ &\le L_de^{-q_d}.
\end{aligned}
\]
Consequently, on the event that no increment is truncated,
\[
\max_{0\le n\le T_2}\left|\sum_{t=1}^nY_t\right| \le \max_{0\le n\le T_2}\left|\sum_{t=1}^n\widehat Y_t\right| + T_2L_de^{-q_d}.
\]

\medskip
\noindent\textbf{Step 3: Maximal concentration.}

Centering decreases the second moment, so
\[
\E[\widehat Y_t^2\mid\mathcal F_{t-1}] \le \E\bigl[Y_t^2\mathbf1_{\{|Y_t|\le L_d\}}\mid\mathcal F_{t-1}\bigr] \le K.
\]
The predictable quadratic variation is therefore at most $KT_2$. Applying Freedman's maximal inequality \citep{freedman1975tail} to both signs of the martingale
gives, for every $z>0$,
\[
\mathbb P\left(\max_{0\le n\le T_2}\left|\sum_{t=1}^n\widehat Y_t\right| \ge z\right) \le 2\exp\left(-\frac{z^2}{2(KT_2+\frac23L_dz)}\right).
\]
Take
\[
z_d\coloneqq \sqrt{2KT_2q_d}+\frac43L_dq_d.
\]
This choice satisfies
\[
z_d^2 \ge 2q_d\left(KT_2+\frac23L_dz_d\right),
\]
so the preceding failure probability is at most
$2e^{-q_d}$.

Combining this event with the absence of truncation, with probability at least
\[
1-(T_2+2)e^{-q_d}\ge1-d^{-D}
\]
for sufficiently large $d$, we have simultaneously
\[
\max_{0\le n\le T_2}\left|\sum_{t=1}^nY_t\right|\le z_d+T_2L_de^{-q_d}, \qquad \max_{1\le t\le T_2}|Y_t|\le L_d.
\]
Here
\[
z_d \le C_D\bigl(\sqrt{T_2\log d}+(\log d)^3\bigr), \qquad T_2L_de^{-q_d} \le C_D(\log d)^2d^{2-b} = o(1).
\]
Since $(\log d)^3=o(\sqrt{T_2\log d})$, multiplying the partial-sum bound by $2\eta_d\sqrt d$ proves the first assertion. The second follows from $|\Xi_t/R_{t-1}|=\sqrt d\,|Y_t|$ before and including the stopped update, and $L_d\le C_D(\log d)^2$.
\end{proof}

\begin{proof}[Proof of Proposition~\ref{prop:h2_logarithmic_growth}]
Fix a teacher $w^\ast\in\mathcal G_d$ and an initialization satisfying~\eqref{eq:h2_initial_lower_bound}. All probabilities below are conditional on this pair. If $R_0\ge a$, then $\tau_a^R=0$ and the conclusion is immediate. Assume henceforth that $R_0<a$.

\medskip
\noindent\textbf{Step 1: Control of the stochastic terms.}

Fix $D'>D$. Apply Lemma~\ref{lem:h2_relative_martingale} and \eqref{eq:h2_uniform_gradient_norm}, both with failure
exponent $D'$. With probability at least $1-2d^{-D'}$, the following bounds hold simultaneously:
\begin{align}
\max_{0\le n\le T_2}\left|2\eta_d\sum_{t=1}^{n\wedge\widehat\tau_a^R}\frac{\Xi_t}{R_{t-1}}\right| &\le C\eta_d\sqrt{dT_2\log d},
\label{eq:h2_cumulative_relative_noise}\\
\max_{1\le t\le\widehat\tau_a^R\wedge T_2}\left|\frac{\Xi_t}{R_{t-1}}\right| &\le C\sqrt d(\log d)^2,
\label{eq:h2_individual_relative_noise}\\
\max_{1\le t\le T_2}\|G_t\|^2 &\le Cd(\log d)^4.
\label{eq:h2_growth_normalization}
\end{align}
We work on this intersection. The first estimate controls the accumulated linear
noise term. The second ensures that each logarithmic increment is well defined and bounds its quadratic remainder. The third controls the spherical normalization.

\medskip
\noindent\textbf{Step 2: Logarithmic comparison.}

For $1\le t\le\widehat\tau_a^R\wedge T_2$, we have $0<R_{t-1}<a$. Therefore \eqref{eq:h2_overlap_lower_recursion} gives
\[
\frac{R_t}{R_{t-1}} \ge \frac{1+c_a\eta_d+2\eta_d\Xi_t/R_{t-1}}{1+\eta_d^2\|G_t\|^2},
\]
where $c_a$ is the lower-bound constant from Proposition~\ref{prop:h2_population_growth}. By~\eqref{eq:h2_individual_relative_noise},
\[
\left|c_a\eta_d+2\eta_d\frac{\Xi_t}{R_{t-1}}\right| \le C\eta_d\sqrt d(\log d)^2=o(1).
\]

Thus the numerator is positive for sufficiently large $d$, including at the update reaching either stopping boundary. In particular, $R_t>0$ throughout these updates, so we may take logarithms.

Using $\log(1+x)\ge x-Cx^2$ for sufficiently small $|x|$ and $\log(1+y)\le y$ for $y\ge0$, we obtain
\[
\log\frac{R_t}{R_{t-1}} \ge c_a\eta_d + 2\eta_d\frac{\Xi_t}{R_{t-1}}-C\eta_d^2\left(1+\left|\frac{\Xi_t}{R_{t-1}}\right|^2\right) - \eta_d^2\|G_t\|^2.
\]
Applying~\eqref{eq:h2_individual_relative_noise} and~\eqref{eq:h2_growth_normalization} therefore gives
\begin{equation}
\log\frac{R_t}{R_{t-1}} \ge c_a\eta_d + 2\eta_d\frac{\Xi_t}{R_{t-1}} - C\eta_d^2d(\log d)^4.
\label{eq:h2_logarithmic_increment}
\end{equation}

Summing the logarithmic increments up to $n\wedge\widehat\tau_a^R$ and using \eqref{eq:h2_cumulative_relative_noise}, we find, simultaneously for every $0\le n\le T_2$,
\[
\log\frac{R_{n\wedge\widehat\tau_a^R}}{R_0} \ge c_a\eta_d(n\wedge\widehat\tau_a^R)-C\left[\eta_d\sqrt{dT_2\log d} + \eta_d^2dT_2(\log d)^4\right].
\]
Since $T_2\asymp\eta_d^{-1}\log d$,
\[
\begin{aligned}
\eta_d\sqrt{dT_2\log d} &= O\bigl(\sqrt{\eta_dd}\,\log d\bigr) = O\bigl((\log d)^{1-A/2}\bigr),\\
\eta_d^2dT_2(\log d)^4 &= O\bigl(\eta_dd(\log d)^5\bigr) = O\bigl((\log d)^{5-A}\bigr).
\end{aligned}
\]
Both errors tend to zero for a sufficiently large universal choice of $A_0$. For sufficiently large $d$, their total contribution, including the constant $C$, is at most $-\log\theta$. Exponentiating yields
\[
R_{n\wedge\widehat\tau_a^R} \ge \theta R_0 \exp\bigl(c_a\eta_d(n\wedge\widehat\tau_a^R)\bigr), \qquad 0\le n\le T_2.
\]

\medskip
\noindent\textbf{Step 3: Exclusion of the lower exit.}

Suppose that the auxiliary process exits through its lower boundary at some $\widehat\tau_a^R\le T_2$. Evaluating the preceding inequality at this time gives
\[
R_{\widehat\tau_a^R} \ge \theta R_0\exp(c_a\eta_d\widehat\tau_a^R) \ge\theta R_0>\beta R_0.
\]
This contradicts the condition $R_{\widehat\tau_a^R}\le\beta R_0$ defining a lower exit. Therefore the auxiliary stopping time can only record an upper exit before $T_2$, and
\[
\widehat\tau_a^R\wedge T_2=\tau_a^R\wedge T_2.
\]
Replacing the auxiliary stopping time by $\tau_a^R$ in the growth estimate proves \eqref{eq:h2_total_overlap_growth}.

The event used in the proof has conditional probability at least $1-2d^{-D'}\ge1-d^{-D}$ for sufficiently large $d$. All constants are uniform over the fixed teacher and initialization under consideration.
\end{proof}

We now translate total-overlap growth into subspace recovery.

\begin{proof}[Proof of part~\textup{(T)} of Theorem~\ref{thm:h2_weak_recovery_appendix}]
Choose $a=a(\rho)\in(\rho,1)$ and fix $\theta,c_{\mathrm{init}}\in(0,1)$. Let $c_a$ be the constant from Proposition~\ref{prop:h2_logarithmic_growth}, and choose $C_0=C_0(\rho)$ such that $c_aC_0>1$.

By Lemmas~\ref{lem:teacher_geometry} and~\ref{lem:h2_initialization}, the teacher belongs to $\mathcal G_d$ and $R_0\ge c_{\mathrm{init}}r/d$ with probability $1-o(1)$. The uniform conditional estimate in Proposition~\ref{prop:h2_logarithmic_growth} then shows that~\eqref{eq:h2_total_overlap_growth} also holds up to $T_2$ with probability $1-o(1)$. Work on this intersection.

If $\tau_a^R>T_2$, then
\[
\begin{aligned}
a &>R_{T_2}\\ &\ge \theta c_{\mathrm{init}}\frac rd\exp(c_a\eta_dT_2)\\ &\ge \theta c_{\mathrm{init}}\left(\frac dr\right)^{c_aC_0-1}.
\end{aligned}
\]
The last expression diverges because $d/r\to\infty$ and $c_aC_0>1$, a contradiction. Thus $\tau_a^R\le T_2$.

At that time,~\eqref{eq:h2_subspace_alignment} gives
\[
\|w_{\tau_a^R}^\parallel\|^2 \ge \frac{R_{\tau_a^R}}{1+r\mu_\ast} \ge \frac{a}{1+r\mu_\ast} > \rho
\]
for sufficiently large $d$. Consequently,
\[
\tau_\rho^{(2),\mathsf T} \le \tau_a^R \le \left\lceil\frac{C_0\log(d/r)}{\eta_d}\right\rceil.
\]
Absorbing the ceiling into a constant $C_T(\rho)$ proves part~\textup{(T)}.
\end{proof}

\subsection{Augmented and untied upper bounds}
\label{sec:h2_other_students}

We complete the proof of Theorem~\ref{thm:h2_weak_recovery_appendix} using the componentwise reductions from Section~\ref{sec:weak_recovery_untied}. For each procedure, it suffices to follow one fixed weight.

\begin{proof}[Proof of part~\textup{(A)} of Theorem~\ref{thm:h2_weak_recovery_appendix}]
By~\eqref{eq:augmented_effective_update}, the first augmented weight follows the tied update with reference step size $\eta_d$. Its initialization has the same distribution as the tied initialization, so part~\textup{(T)} gives, with probability $1-o(1)$,
\[
\max_{0\le t\le C_T\log(d/r)/\eta_d}\|(w_t^{1,\mathsf A})^\parallel\|^2\ge\rho.
\]
Recovery of this weight implies recovery of the student. Thus part~\textup{(A)} holds with $C_A=C_T$.
\end{proof}

\begin{proof}[Proof of part~\textup{(U)} of Theorem~\ref{thm:h2_weak_recovery_appendix}]
\noindent\textbf{Step 1: Effective recursion.}
Retain the single-weight trajectory $u_t=w_t^{1,\mathsf U}$ and the quantities $G_t^{\mathrm S}$ and $\Xi_t^{\mathrm S}$ from Subsection~\ref{sec:h2_noise_moments}, and write
\[
R_t^{\mathsf U}\coloneqq R(u_t).
\]
By~\eqref{eq:single_filter_population}, $\Psi=\Phi/\sqrt r$. Applying the quadratic-form calculation of Lemma~\ref{lem:h2_dynamics} to the effective update~\eqref{eq:untied_effective_update} therefore gives
\begin{equation}
\left(1+\frac{\eta_d^2}{r}\|G_t^{\mathrm S}\|^2\right)R_t^{\mathsf U} =R_{t-1}^{\mathsf U} + \frac{\eta_d}{r}\|\nabla\Phi(u_{t-1})\|^2 + \frac{2\eta_d}{\sqrt r}\Xi_t^{\mathrm S} + \frac{\eta_d^2}{r}\langle G_t^{\mathrm S},B^\ast G_t^{\mathrm S}\rangle.
\label{eq:h2_untied_energy}
\end{equation}
Since $B^\ast$ is positive semidefinite, Proposition~\ref{prop:h2_population_growth} yields
\begin{equation}
\left(1+\frac{\eta_d^2}{r}\|G_t^{\mathrm S}\|^2\right)R_t^{\mathsf U} \ge \left(1+\frac{c_a\eta_d}{r}\right)R_{t-1}^{\mathsf U} + \frac{2\eta_d}{\sqrt r}\Xi_t^{\mathrm S}
\label{eq:h2_untied_lower_recursion}
\end{equation}
whenever $0<R_{t-1}^{\mathsf U}<a$, on $\mathcal G_d$ and for sufficiently large $d$. Thus the drift and normalization coefficients are divided by $r$, while the noise coefficient is divided by $\sqrt r$.

\medskip
\noindent\textbf{Step 2: Stopped logarithmic growth.}
Choose $a=a(\rho)\in(\rho,1)$, fix $\theta,c_{\mathrm{init}}\in(0,1)$ and $\beta\in(0,\theta)$, and choose $C_0=C_0(\rho)$ such that $c_aC_0>1$. Set
\[
T_2^{\mathsf U} \coloneqq \left\lceil\frac{C_0r\log(d/r)}{\eta_d}\right\rceil.
\]
Condition on a teacher $w^\ast\in\mathcal G_d$ and an initialization satisfying $R_0^{\mathsf U}\ge c_{\mathrm{init}}r/d$. If $R_0^{\mathsf U}\ge a$, recovery already holds for sufficiently large $d$ by~\eqref{eq:h2_subspace_alignment}. Otherwise, define
\[
\widehat\tau_a^{R,\mathsf U} \coloneqq \inf\{t\ge0:R_t^{\mathsf U}\le\beta R_0^{\mathsf U}\ \text{or}\ R_t^{\mathsf U}\ge a\}.
\]
Before and including the update reaching this stopping time,
\[
R_{t-1}^{\mathsf U} > \beta c_{\mathrm{init}}\frac rd.
\]
Hence the untied relative-noise estimate from Subsection~\ref{sec:h2_noise_moments} applies.

The variables $\Xi_t^{\mathrm S}/(\sqrt d\,R_{t-1}^{\mathsf U})$, set to zero after $\widehat\tau_a^{R,\mathsf U}$, are martingale differences with conditional $L^\nu$-norm at most $C\nu^2$. The stopping indicator includes the exit update and is measurable with respect to $\mathcal F_{t-1}^{\mathrm S}$. Since $r=d^\delta$ with $\delta<1/2$ and $A$ is fixed,
\[
T_2^{\mathsf U} \asymp d^{1+\delta}(\log d)^{A+1} \le d^2
\]
for sufficiently large $d$. Thus the truncation and concentration argument in the proof of Lemma~\ref{lem:h2_relative_martingale} applies with $T_2$ replaced by $T_2^{\mathsf U}$. It controls both the cumulative stopped relative noise and the individual increments over this horizon. Together with the untied version of \eqref{eq:h2_uniform_gradient_norm}, these estimates give the logarithmic comparison below.

More precisely, the one-step perturbation of the numerator in~\eqref{eq:h2_untied_lower_recursion}, after division by $R_{t-1}^{\mathsf U}$, satisfies
\[
\left|\frac{c_a\eta_d}{r} + \frac{2\eta_d}{\sqrt r}\frac{\Xi_t^{\mathrm S}}{R_{t-1}^{\mathsf U}}\right| \le C_D\frac{\eta_d}{\sqrt r}\sqrt d(\log d)^2 = o(1).
\]
Thus the numerator remains positive through the exit update. Expanding the logarithm and summing, as in Subsection~\ref{sec:h2_tied_recovery}, gives, for every fixed $D>0$, with conditional probability at least $1-d^{-D}$, simultaneously for $0 \le n\le T_2^{\mathsf U}$,
\[
\log\frac{R_{n\wedge\widehat\tau_a^{R,\mathsf U}}^{\mathsf U}}{R_0^{\mathsf U}} \ge
\frac{c_a\eta_d}{r}(n\wedge\widehat\tau_a^{R,\mathsf U})-C_D\left[\frac{\eta_d}{\sqrt r}\sqrt{dT_2^{\mathsf U}\log d} + \frac{\eta_d^2}{r}dT_2^{\mathsf U}(\log d)^4\right].
\]
All conditional estimates are uniform over the teacher and initialization under consideration.

Since $T_2^{\mathsf U}\asymp r\eta_d^{-1}\log d$, the accumulated errors satisfy
\[\frac{\eta_d}{\sqrt r}\sqrt{dT_2^{\mathsf U}\log d} = O\bigl((\log d)^{1-A/2}\bigr), \qquad \frac{\eta_d^2}{r}dT_2^{\mathsf U}(\log d)^4 = O\bigl((\log d)^{5-A}\bigr).
\]
These are the same orders as in the tied proof. For the same sufficiently large universal choice of $A_0$, their total contribution tends to zero. Consequently, for sufficiently large $d$,
\[
R_{n\wedge\widehat\tau_a^{R,\mathsf U}}^{\mathsf U} \ge \theta R_0^{\mathsf U}\exp\left(
\frac{c_a\eta_d}{r}(n\wedge\widehat\tau_a^{R,\mathsf U})\right), \qquad 0\le n\le T_2^{\mathsf U}.
\]

\medskip
\noindent\textbf{Step 3: Reaching the recovery threshold.}

Since $\theta>\beta$, this excludes an exit through the lower boundary before $T_2^{\mathsf U}$, exactly as in the tied proof. Moreover, if $\widehat\tau_a^{R,\mathsf U}>T_2^{\mathsf U}$, then
\[
a>R_{T_2^{\mathsf U}}^{\mathsf U} \ge \theta c_{\mathrm{init}}\left(\frac dr\right)^{c_aC_0-1},
\]
which is impossible for sufficiently large $d$. Therefore the process reaches $R_t^{\mathsf U} \ge a$ at some $t\le T_2^{\mathsf U}$. At that time,
\[
\|u_t^\parallel\|^2 \ge \frac{R_t^{\mathsf U}}{1+r\mu_\ast} \ge \frac{a}{1+r\mu_\ast}>\rho.
\]

Finally, Lemma~\ref{lem:h2_initialization} bounds the excluded initialization probability by $Ce^{-cr}$, uniformly over $w^\ast\in\mathcal G_d$. Averaging the conditional estimates and using Lemma~\ref{lem:teacher_geometry} shows that, with probability $1-o(1)$,
\[
\tau_\rho^{(2),\mathsf U} \le T_2^{\mathsf U} \le \frac{C_U(\rho)r\log(d/r)}{\eta_d},
\]
after absorbing the ceiling into $C_U(\rho)$. This proves part~\textup{(U)}.
\end{proof}

Together with part~\textup{(T)}, a finite union bound gives the three recovery guarantees simultaneously with probability $1-o(1)$.

\begin{remark}[Shifted tied weights]
The subspaces $\mathcal S^\ast$ and $(\mathcal S^\ast)^\perp$ are invariant under $\Pi$. Consequently,
\[
(\Pi^kw)^\parallel=\Pi^kw^\parallel, \qquad \|(\Pi^kw)^\parallel\|^2=\|w^\parallel\|^2.
\]
Thus, when the tied weight achieves subspace recovery, all its cyclic shifts have the same projection mass in the teacher subspace.
\end{remark}

\section{Lower bounds for weak recovery}
\label{sec:weak_recovery_lower_bounds}

In this section, we prove Theorem~\ref{thm:nonescape} under its assumptions on the teacher, initialization, and scaling. The link is a fixed centered polynomial of degree $P$ with information exponent $p\ge2$. The population identities and componentwise reductions from Sections~\ref{sec:corr_dynamics} and~\ref{sec:weak_recovery_untied} apply to this entire class: their derivations use the finite Hermite expansion and do not require a parity assumption. Throughout, $A>0$ is fixed, the reference learning rate is
\[
\eta_d=d^{-p/2}(\log d)^{-A},
\]
and the three procedures use the step sizes in~\eqref{eq:student_stepsizes}. We retain the notation $w_t^{j,\mathsf k}$ and $s_{\mathsf k}$ for the iterates and numbers of weights.

\subsection{Proof strategy}
\label{sec:nonescape_strategy}

We show that all absolute correlations remain at their initialization scale $O(\gamma_d)$ for a sufficiently small constant multiple of the microscopic growth time. This excludes both weak directional recovery and weak subspace recovery. The argument applies simultaneously to every weight, including when the procedures share their observations.

The proof has two ingredients. First, we bound the absolute population drift by the largest absolute correlation raised to the power $p-1$. Lemma~\ref{lem:exact_damping} then gives a pathwise upper comparison that retains the exact spherical normalization. Second, we control the accumulated coordinate noise until the first escape from the initialization scale, uniformly over all weights and teacher directions.

For the tied and augmented procedures, the relevant horizon is of order $(\eta_d\gamma_d^{p-2})^{-1}$. For the untied procedure, it is longer by a factor $r$. On these horizons, the stochastic contribution is $o(\gamma_d)$, while a sufficiently small time constant makes the accumulated drift smaller than the distance to the escape threshold. Evaluating the comparison at a hypothetical first exit gives the desired contradiction.

\subsection{Initialization and escape times}
\label{sec:nonescape_initialization}

Define the largest absolute correlation of a weight and of a procedure by
\[
\widehat M(w) \coloneqq \max_{0\le k<r}|m_k(w)|, \qquad \widehat M_t^{\mathsf k} \coloneqq \max_{1\le j\le s_{\mathsf k}}\widehat M(w_t^{j,\mathsf k}).
\]
We control absolute correlations for every $p\ge2$.

Lemma~\ref{lem:initialization_scale} establishes that $\widehat M(w_0)=O(\gamma_d)$ with probability $1-o(1)$
for a single random weight. Here we need this upper bound simultaneously over all initial weights, so we use a quantitative tail estimate before taking the union bound.

Fix a teacher $w^\ast$. By Assumption~\ref{hyp:teacher_init}, each initial weight is conditionally uniform on $\mathbb S^{d-1}$. The standard spherical-cap bound therefore gives, for every procedure, weight, and teacher direction,
\[
\mathbb P\left(|m_k(w_0^{j,\mathsf k})|>z\,\middle|\,w^\ast\right) \le C e^{-cdz^2}, \qquad 0<z<1,
\]
with universal constants $c,C>0$.
Since $\gamma_d=o(1)$, we may take $z=C_{\mathrm{init}}\gamma_d$ for any fixed $C_{\mathrm{init}}>0$ and sufficiently large $d$. A union bound over the $(1+2s)r$ initial correlations then yields
\begin{equation}
\begin{split}
\mathbb P\left(\max_{\mathsf k\in\{\mathsf T,\mathsf U,\mathsf A\}}\widehat M_0^{\mathsf k} > C_{\mathrm{init}}\gamma_d\,\middle|\,w^\ast\right) &\le C(1+2s)re^{-cC_{\mathrm{init}}^2\log r} \le C d^{\kappa+\delta-cC_{\mathrm{init}}^2\delta}.
\end{split}
\label{eq:nonescape_initialization}
\end{equation}
The last inequality uses $s\le d^\kappa$ and $r=d^\delta$ from Assumption~\ref{hyp:scaling}. Choose $C_{\mathrm{init}}=C_{\mathrm{init}}(\delta,\kappa)$ large enough that $cC_{\mathrm{init}}^2\delta>\kappa+\delta$. Then~\eqref{eq:nonescape_initialization} is $o(1)$, uniformly over the teacher.

Combining this estimate with Lemma~\ref{lem:teacher_geometry}, we obtain, with probability $1-o(1)$,
\[
w^\ast\in\mathcal G_d, \qquad \max_{\mathsf k\in\{\mathsf T,\mathsf U,\mathsf A\}}\widehat M_0^{\mathsf k} \le C_{\mathrm{init}}\gamma_d.
\]
We will first establish the subsequent estimates conditionally on a teacher and initial weights satisfying these conditions.

Fix any constant $C_{\mathrm{esc}}>C_{\mathrm{init}}$ and write $\tau_{\mathrm{esc}}^{\mathsf k} \coloneqq\tau_{\mathrm{esc}}^{\mathsf k}(C_{\mathrm{esc}})$ for the escape times defined in~\eqref{eq:nonescape_escape_times}. On the initialization event above, $\tau_{\mathrm{esc}}^{\mathsf k}\ge1$. For every integer $1\le t\le\tau_{\mathrm{esc}}^{\mathsf k}$,
\[
\max_{1\le j\le s_{\mathsf k}}\widehat M(w_{t-1}^{j,\mathsf k}) < C_{\mathrm{esc}}\gamma_d.
\]
Thus bounds evaluated at the preceding iterate remain valid through the update producing the first exit. We will use this observation to apply the upper comparison at a hypothetical exit time and obtain a contradiction.

\subsection{An upper comparison for the correlations}
\label{sec:nonescape_comparison}

Throughout this subsection, $\Phi$ denotes the tied population objective. The finite Hermite expansion implies
\[
|\phi'(z)|\le C_\sigma|z|^{p-1}, \qquad |z|\le1.
\]
We combine this bound with Lemma~\ref{lem:population_drift} and the overlap-energy estimate~\eqref{eq:teacher_overlap_energy} from Appendix~\ref{sec:coordinate_noise_moments_proof}.

Let $\mathcal F_t$ denote the joint filtration generated by the teacher, all initial weights, and all observations used through time $t$. For $\mathsf k\in\{\mathsf T,\mathsf A\}$, let $G_t^{j,\mathsf k}$ be the tied sample gradient evaluated at $w_{t-1}^{j,\mathsf k}$. For $\mathsf k=\mathsf U$, let it be the single-weight sample gradient associated with \eqref{eq:single_filter_loss}. Define
\[
\xi_{k,t}^{j,\mathsf k} \coloneqq -\left\langle\Pi^kw^\ast,\, G_t^{j,\mathsf k} - \mathbb E[G_t^{j,\mathsf k}\mid\mathcal F_{t-1}]\right\rangle.
\]
These are the coordinate fluctuations already used in Sections~\ref{sec:corr_dynamics} and~\ref{sec:weak_recovery_untied}, now indexed by the procedure and weight. Fresh-sample independence gives
\[
\mathbb E[G_t^{j,\mathsf k}\mid\mathcal F_{t-1}] =
\begin{cases}
\nabla\Phi(w_{t-1}^{j,\mathsf k}), &\mathsf k\in\{\mathsf T,\mathsf A\},\\ 
\nabla\Psi(w_{t-1}^{j,\mathsf U}), &\mathsf k=\mathsf U,
\end{cases}
\]
where $\Psi$ is the single-weight population objective from~\eqref{eq:single_filter_population}. By the definition of the centered fluctuations,
\[
\mathbb E[\xi_{k,t}^{j,\mathsf k}\mid\mathcal F_{t-1}] = 0.
\]

\begin{lemma}[Upper comparison before escape]
\label{lem:nonescape_drift}
On $\mathcal G_d$, there exists $C_\sigma<\infty$ such that, for all sufficiently large $d$ and every $w\in\mathbb S^{d-1}$, 
\[
\max_{0\le k<r}\left|\left\langle\Pi^kw^\ast,\nabla\Phi(w)\right\rangle\right| \le C_\sigma\widehat M(w)^{p-1}.
\]
Moreover, for $\mathsf k\in\{\mathsf T,\mathsf A\}$ and every integer $0\le n\le\tau_{\mathrm{esc}}^{\mathsf k}$,
\begin{equation}
\widehat M_n^{\mathsf k} \le \widehat M_0^{\mathsf k} + C_\sigma\eta_dn(C_{\mathrm{esc}}\gamma_d)^{p-1}+2\eta_d\max_{\substack{1\le j\le s_{\mathsf k}\\0\le k<r}}\max_{0\le h\le n}\left|\sum_{t=1}^h\xi_{k,t}^{j,\mathsf k}\right|.
\label{eq:nonescape_tied_comparison}
\end{equation}
For the untied procedure and every integer $0\le n\le\tau_{\mathrm{esc}}^{\mathsf U}$,
\begin{equation}
\widehat M_n^{\mathsf U} \le \widehat M_0^{\mathsf U} + C_\sigma\frac{\eta_d}{r}
n(C_{\mathrm{esc}}\gamma_d)^{p-1}+\frac{2\eta_d}{\sqrt r}\max_{\substack{1\le j\le s\\0\le k<r}}\max_{0\le h\le n}\left|\sum_{t=1}^h\xi_{k,t}^{j,\mathsf U}\right|.
\label{eq:nonescape_untied_comparison}
\end{equation}
Both comparisons are pathwise and include the first exit step.
\end{lemma}

\begin{proof}
\noindent\textbf{Population drift.}
Fix $w\in\mathbb S^{d-1}$ and write $M=\widehat M(w)$. If $M=0$, then every correlation vanishes and $\phi'(0)=0$, so Lemma~\ref{lem:population_drift} gives $\nabla\Phi(w)=0$. Suppose henceforth that $M>0$. Since $p\ge2$,
\[
\sum_{\ell=0}^{r-1}|m_\ell(w)|^p \le M^{p-2}\sum_{\ell=0}^{r-1}m_\ell(w)^2.
\]
By Lemma~\ref{lem:population_drift}, the derivative bound above, and the definition of $\mu_\ast$,
\[
\begin{aligned}
\left|\left\langle\Pi^kw^\ast,\nabla\Phi(w)\right\rangle\right| &\le C_\sigma\left[(1+r\mu_\ast)M^{p-1} + M\sum_{\ell=0}^{r-1}|m_\ell(w)|^p\right]\\ &\le C_\sigma M^{p-1}\left[1+r\mu_\ast+\sum_{\ell=0}^{r-1}m_\ell(w)^2\right]\\ &\le C_\sigma M^{p-1}.
\end{aligned}
\]
The last inequality uses $r\mu_\ast=o(1)$ on $\mathcal G_d$ and the uniform overlap-energy bound~\eqref{eq:teacher_overlap_energy}. After enlarging $C_\sigma$, this proves the first assertion, including the case $p=2$.

\medskip
\noindent\textbf{Tied and augmented weights.}
Fix $\mathsf k\in\{\mathsf T,\mathsf A\}$, a weight index $j$, and a teacher direction $k$. The exact coordinate update from the proof of Lemma~\ref{lem:correlation_evolution}, together with \eqref{eq:augmented_effective_update}, has the form
\[
d_t x_{t+1}=x_t+F_t+N_{t+1}-N_t,
\]
where
\[
x_t=m_k(w_t^{j,\mathsf k}), \qquad d_t=\sqrt{1+\eta_d^2\|G_{t+1}^{j,\mathsf k}\|^2},
\]
and
\[
\begin{aligned}
F_t &= -\eta_d\left\langle\Pi^kw^\ast,\nabla\Phi(w_t^{j,\mathsf k})\right\rangle,\\
N_t &=\eta_d\sum_{\ell=1}^t\xi_{k,\ell}^{j,\mathsf k}.
\end{aligned}
\]
Thus $d_t\ge1$ and $N_0=0$, so Lemma~\ref{lem:exact_damping} applies with $q_t=1$.

For every integer $n\le\tau_{\mathrm{esc}}^{\mathsf k}$, each $t<n$ lies strictly before escape. The population bound just proved therefore gives
\[
\begin{aligned}
\sum_{t=0}^{n-1}|F_t| \le C_\sigma\eta_d\sum_{t=0}^{n-1}\widehat M(w_t^{j,\mathsf k})^{p-1} \le C_\sigma\eta_dn(C_{\mathrm{esc}}\gamma_d)^{p-1}.
\end{aligned}
\]
Consequently, Lemma~\ref{lem:exact_damping} yields
\[
\begin{split}
|m_k(w_n^{j,\mathsf k})| \le |m_k(w_0^{j,\mathsf k})| + C_\sigma\eta_dn(C_{\mathrm{esc}}\gamma_d)^{p-1}+2\eta_d\max_{0\le h\le n}\left|\sum_{\ell=1}^h\xi_{k,\ell}^{j,\mathsf k}\right|.
\end{split}
\]
Taking the maximum over $j$ and $k$ proves~\eqref{eq:nonescape_tied_comparison}.

\medskip
\noindent\textbf{Untied weights.}
By~\eqref{eq:single_filter_population}, $\nabla\Psi=\nabla\Phi/\sqrt r$. For a fixed untied weight $j$ and teacher direction $k$, the exact update~\eqref{eq:untied_overlap_exact} therefore has the same damping form, now with
\[
x_t = m_k(w_t^{j,\mathsf U}), \qquad d_t = \sqrt{1+\frac{\eta_d^2}{r}\|G_{t+1}^{j,\mathsf U}\|^2},
\]
and
\[
\begin{aligned}
F_t &= -\frac{\eta_d}{r}\left\langle\Pi^kw^\ast,\nabla\Phi(w_t^{j,\mathsf U})\right\rangle,\\
N_t &= \frac{\eta_d}{\sqrt r}\sum_{\ell=1}^t\xi_{k,\ell}^{j,\mathsf U}.
\end{aligned}
\]
Again $d_t\ge1$, and for $n\le\tau_{\mathrm{esc}}^{\mathsf U}$,
\[
\sum_{t=0}^{n-1}|F_t| \le C_\sigma\frac{\eta_d}{r} n(C_{\mathrm{esc}}\gamma_d)^{p-1}.
\]
Applying Lemma~\ref{lem:exact_damping} with $q_t=1$ and taking the maximum over all weights and teacher directions gives~\eqref{eq:nonescape_untied_comparison}.

Both applications are pathwise. In particular, all drift terms used to bound the correlation at time $n$ are evaluated at times $t<n$. The comparisons therefore remain valid when $n$ is the first exit time.
\end{proof}

\subsection{Noise before escape}
\label{sec:nonescape_noise}

The moment estimates of Lemma~\ref{lem:coordinate_noise_moments} apply to the tied and augmented weights, while Lemma~\ref{lem:single_filter_stochastic} applies to the untied weights. Their proofs use only the finite Hermite expansion and $p\ge2$, so the estimates hold for every polynomial link considered in Theorem~\ref{thm:nonescape}. Before escape,
\[
r\,m_k(w_{t-1}^{j,\mathsf k})^2 \le C_{\mathrm{esc}}^2r\gamma_d^2=o(1).
\]
Thus the stopped coordinate fluctuations have bounded conditional variance and conditional $L^\nu$-norm of order $\nu^P$.

\begin{lemma}[Noise before escape]
\label{lem:nonescape_noise}

Fix a teacher $w^\ast\in\mathcal G_d$ and initial weights satisfying
\[
\max_{\mathsf k\in\{\mathsf T,\mathsf U,\mathsf A\}}\widehat M_0^{\mathsf k} \le C_{\mathrm{init}}\gamma_d.
\]
For every fixed $D>0$, there exists $C_{\sigma,D}<\infty$ such that the following holds for all sufficiently large $d$. For any deterministic integer horizons $1\le T_{\mathsf k}\le d^p$, with probability at least $1-d^{-D}$ over the SGD samples, simultaneously for $\mathsf k\in\{\mathsf T,\mathsf U,\mathsf A\}$,
\begin{equation}
\max_{\substack{1\le j\le s_{\mathsf k}\\0\le k<r}}\max_{0\le n\le T_{\mathsf k}}\left|\sum_{t=1}^{n\wedge\tau_{\mathrm{esc}}^{\mathsf k}}\xi_{k,t}^{j,\mathsf k}\right|\le C_{\sigma,D}\left(\sqrt{T_{\mathsf k}\log d} + (\log d)^{P+1}\right).
\label{eq:nonescape_noise}
\end{equation}
The constant may also depend on $\delta,\kappa,C_{\mathrm{init}},C_{\mathrm{esc}}$. The constants and dimension threshold are uniform over the conditioned teacher and initial weights.
\end{lemma}

\begin{proof}
Fix the teacher and initial weights as in the statement, and work under the corresponding conditional law.

For each procedure, weight, and teacher direction, set
\[
Y_t \coloneqq \mathbf1_{\{t\le\tau_{\mathrm{esc}}^{\mathsf k}\}}\xi_{k,t}^{j,\mathsf k}, \qquad 1\le t\le T_{\mathsf k},
\]
suppressing the coordinate indices in $Y_t$. The stopping indicator is $\mathcal F_{t-1}$-measurable,
so $(Y_t)$ is a martingale difference sequence. On $\{t\le\tau_{\mathrm{esc}}^{\mathsf k}\}$,
\[
r\,m_k(w_{t-1}^{j,\mathsf k})^2 \le C_{\mathrm{esc}}^2r\gamma_d^2=o(1).
\]
Consequently, Lemmas~\ref{lem:coordinate_noise_moments} and~\ref{lem:single_filter_stochastic} give, uniformly over all coordinate processes,
\[
\mathbb E[Y_t^2\mid\mathcal F_{t-1}] \le C_\sigma, \qquad \left(\mathbb E[|Y_t|^\nu\mid\mathcal F_{t-1}]\right)^{1/\nu} \le C_\sigma\nu^P, \qquad \nu\ge2.
\]
The earlier moment calculations apply to the joint filtration because the observation at time $t$ is fresh conditionally on $\mathcal F_{t-1}$.

These are exactly the conditional moment bounds used in the proof of Lemma~\ref{lem:martingale}. Apply its truncation, recentering, and maximal-concentration argument with moment order proportional to $\log d$ and cutoff proportional to $(\log d)^P$. For any fixed $b>p$, choosing the proportionality constants sufficiently large gives
\[
\mathbb P\left(\max_{0\le n\le T_{\mathsf k}}\left|\sum_{t=1}^nY_t\right| > C_{\sigma,b}\left[\sqrt{T_{\mathsf k}\log d} + (\log d)^{P+1}\right]\right)\le C(T_{\mathsf k}+1)d^{-b}.
\]
Here the factor $T_{\mathsf k}$ accounts for the truncation events, while the maximal martingale bound contributes $Cd^{-b}$. The accumulated centering correction is at most
\[
C_{\sigma,b}T_{\mathsf k}(\log d)^P d^{-b} \le C_{\sigma,b}d^{p-b}(\log d)^P=o(1),
\]
and is absorbed into the displayed threshold.

Choose $b>D+\kappa+\delta+p$. A union bound over the $(1+2s)r$ coordinate processes bounds the total failure probability by
\[
C(1+2s)r(d^p+1)d^{-b} \le C d^{\kappa+\delta+p-b} \le d^{-D}
\]
for sufficiently large $d$. No independence between these processes is needed.

Finally,
\[
\sum_{t=1}^nY_t = \sum_{t=1}^{n\wedge\tau_{\mathrm{esc}}^{\mathsf k}}\xi_{k,t}^{j,\mathsf k},
\]
so the simultaneous partial-sum bounds give~\eqref{eq:nonescape_noise}. The stopping convention includes the exit-producing increment, and all estimates are uniform over the conditioned teacher and initial weights.
\end{proof}

\subsection{Proof of the non-escape theorem}
\label{sec:nonescape_theorem}

We now combine Lemmas~\ref{lem:nonescape_drift} and~\ref{lem:nonescape_noise} to prove Theorem~\ref{thm:nonescape}.

\begin{proof}[Proof of Theorem~\ref{thm:nonescape}]
Take $C_{\mathrm{init}}$ as in Subsection~\ref{sec:nonescape_initialization} and fix $C_{\mathrm{esc}}>C_{\mathrm{init}}$ depending only on $\delta,\kappa$.
Let $C_\sigma$ be the constant in Lemma~\ref{lem:nonescape_drift}, and choose $c>0$ so that
\[
C_{\mathrm{init}} + C_\sigma cC_{\mathrm{esc}}^{p-1} < C_{\mathrm{esc}}.
\]
With $T_d$ as in the theorem, define
\[
T_{\mathsf T} = T_{\mathsf A} \coloneqq \lfloor T_d\rfloor, \qquad T_{\mathsf U} \coloneqq \lfloor rT_d\rfloor.
\]
Since
\[
\frac1{\eta_d\gamma_d^{p-2}} = \frac{d^{p-1}(\log d)^A}{(\log r)^{(p-2)/2}} \asymp d^{p-1}(\log d)^{A-(p-2)/2},
\]
all three horizons diverge. Moreover,
\[
T_{\mathsf U} = O\!\left(d^{p-1+\delta}(\log d)^{A-(p-2)/2}\right) = o(d^p).
\]
Thus $1\le T_{\mathsf k}\le d^p$ for sufficiently large $d$, and Lemma~\ref{lem:nonescape_noise} applies.

\medskip
\noindent\textbf{Step 1: The accumulated noise is smaller than initialization.}
Fix a teacher $w^\ast\in\mathcal G_d$ and initial weights satisfying $\max_{\mathsf k}\widehat M_0^{\mathsf k} \le C_{\mathrm{init}}\gamma_d$.
Apply Lemma~\ref{lem:nonescape_noise} with the horizons above and any fixed $D>0$. On its event, which has conditional probability at least $1-d^{-D}$, the stochastic terms in \eqref{eq:nonescape_tied_comparison} and~\eqref{eq:nonescape_untied_comparison} are $o(\gamma_d)$ simultaneously for all procedures. Indeed, the definitions of the horizons give
\[
\max\left\{\frac{\eta_d\sqrt{T_{\mathsf T}\log d}}{\gamma_d}, \frac{\eta_d\sqrt{T_{\mathsf U}\log d}}{\sqrt r\,\gamma_d}\right\} \le \sqrt{c\eta_d\gamma_d^{-p}\log d} = O\!\left((\log d)^{-A/2-(p-2)/4}\right) = o(1).
\]
For the remaining term,
\[
\frac{\eta_d(\log d)^{P+1}}{\gamma_d} = O\!\left(d^{-(p-1)/2}(\log d)^{P+1/2-A}\right) = o(1).
\]
The corresponding untied term is smaller by $r^{-1/2}$. Both estimates hold for every fixed $A>0$ and $2\le p\le P<\infty$.

\medskip
\noindent\textbf{Step 2: Excluding the first exit.}

Work on this common noise event. For $n\le\tau_{\mathrm{esc}}^{\mathsf k}\wedge T_{\mathsf k}$, every partial sum with $h\le n$ agrees with its stopped
counterpart in~\eqref{eq:nonescape_noise}. For $\mathsf k\in\{\mathsf T,\mathsf A\}$,
Lemma~\ref{lem:nonescape_drift} therefore gives
\[
\widehat M_n^{\mathsf k} \le C_{\mathrm{init}}\gamma_d + C_\sigma\eta_dT_{\mathsf k}(C_{\mathrm{esc}}\gamma_d)^{p-1}+o(\gamma_d) \le \left(C_{\mathrm{init}} + C_\sigma cC_{\mathrm{esc}}^{p-1}+o(1)\right)\gamma_d.
\]
For the untied procedure, the same comparison yields
\[
\widehat M_n^{\mathsf U}\le C_{\mathrm{init}}\gamma_d + C_\sigma\frac{\eta_d}{r}T_{\mathsf U}(C_{\mathrm{esc}}\gamma_d)^{p-1}+o(\gamma_d) \le \left(C_{\mathrm{init}} + C_\sigma cC_{\mathrm{esc}}^{p-1}+o(1)\right)\gamma_d.
\]
By the choice of $c$, these upper bounds are strictly smaller than $C_{\mathrm{esc}}\gamma_d$ for sufficiently large $d$. If any procedure had $\tau_{\mathrm{esc}}^{\mathsf k}\le T_{\mathsf k}$, applying its comparison at $n=\tau_{\mathrm{esc}}^{\mathsf k}$ would contradict the definition of the exit time. Consequently, simultaneously for all three procedures,
\[
\tau_{\mathrm{esc}}^{\mathsf k}>T_{\mathsf k}, \qquad \max_{0\le t\le T_{\mathsf k}}\widehat M_t^{\mathsf k}<C_{\mathrm{esc}}\gamma_d.
\]
The conditional estimates are uniform over the teacher and initial weights under consideration. Lemma~\ref{lem:teacher_geometry} and \eqref{eq:nonescape_initialization} show that these conditioning requirements hold with probability $1-o(1)$. Averaging gives the same unconditional probability for the simultaneous non-escape bounds. Since the escape times are integer-valued, they exceed the unrounded horizon $T_d$ for the tied and augmented procedures and $rT_d$ for the untied procedure. This proves the theorem, with $c$ and $C_{\mathrm{esc}}$ depending only on $\sigma,\delta,\kappa$.
\end{proof}

%% file: appendices.tex
\clearpage
\appendix

\section{Teacher geometry and initialization}\label{sec:teacher_geometry}

The purpose of this section is to characterize the random geometry of the teacher orbit and the scale of the correlations at initialization. For a fixed unit vector \(v\in\mathbb{S}^{d-1}\) and an independent vector \(u\sim\operatorname{Unif}(\mathbb{S}^{d-1})\), the correlation \(\langle v,u\rangle\) is typically of order \(d^{-1/2}\). In our model, however, we consider the correlations of a random direction with the \(r\) shifted teacher directions \(w^\ast,\Pi w^\ast,\ldots,\Pi^{r-1}w^\ast\). Because these directions are nearly orthogonal, the corresponding correlations behave approximately as \(r\) independent Gaussian coordinates. Their largest relevant value is therefore enhanced by the usual extreme-value factor and occurs at the scale
\[
\gamma_d\coloneqq\sqrt{\frac{\log r}{d}}.
\]
We first show that the teacher orbit is nearly orthogonal. We then establish the initialization scale \(\gamma_d\) and a quantitative gap between the largest relevant correlation and its closest competitor, where in the even-\(p\) case correlations are compared in absolute value. This gap provides the initial asymmetry that is amplified by the nonlinear SGD dynamics.

Recall that the teacher direction is drawn uniformly at random, \(w^\ast\sim\operatorname{Unif}(\mathbb S^{d-1})\). We begin by controlling the correlations among its cyclic shifts. Define the maximum correlation within the teacher orbit by
\[
\mu_\ast \coloneqq \max_{1\le k \le r-1} \left| \left\langle w^\ast,\Pi^k w^\ast\right\rangle \right|.
\]

\begin{lemma}[Geometry of the teacher orbit] \label{lem:teacher_geometry}
Under Assumptions~\ref{hyp:teacher_init} and~\ref{hyp:scaling}, there exists a constant \(C_\ast <\infty\) such that, for the event
\[
\mathcal{G}_d \coloneqq \left\{\mu_\ast\le C_\ast\gamma_d\right\},
\]
it holds that
\[
\mathbb{P} (\mathcal{G}_d)\to 1.
\]
Moreover, on \(\mathcal{G}_d\), \(r\mu_\ast = o \left ((\log d)^{-2} \right)\).
\end{lemma}

\begin{proof}
Assumption~\ref{hyp:teacher_init} allows us to write $w^\ast=g/\|g\|$ with $g\sim\mathcal{N}(0,\Id_d)$, so that
\[
\mu_\ast = \max_{1 \le a \le r-1} \frac{\left |g^\top \Pi^ag \right|}{\|g\|^2}.
\]
Thus we need an upper bound on \(\max_{1 \le a \le r-1} \left |g^\top \Pi^ag \right|\) and a lower bound on \(\|g\|^2\). By the Hanson--Wright inequality (see e.g.~\cite[Theorem 6.2.1]{vershynin}), for every symmetric matrix $A$,
\begin{equation} \label{eq:Hanson_Wright}
\mathbb{P} \left( \left|g^\top A g-\E[g^\top A g]\right|>t \right) \le 2 \exp \left ( -c \min \left\{ \frac{t^2}{\|A\|_{\mathrm F}^2}, \frac{t}{\|A\|_{\mathrm{op}}} \right \} \right ), \qquad t>0.
\end{equation}
Note that $\E [g^\top A g] = \tr A$. Fix $1\le a\le r-1$. Since $a \Delta = ad/r$ is not a multiple of $d$, the cyclic shift $\Pi^a$ has no fixed coordinates and hence \(\tr (\Pi^a) =0\). Since the matrix \(\Pi^a\) is not necessarily symmetric, we consider its symmetric part
\[
A_a\coloneqq\frac{\Pi^a+\Pi^{-a}}{2}.
\]
Then
\[
\tr(A_a)=0, \qquad \|A_a\|_{\mathrm{op}}\le1, \qquad \|A_a\|_{\mathrm{F}}^2 \le d.
\]
Moreover, since the skew-symmetric part does not contribute to a quadratic form, \(g^\top \Pi^ag = g^\top A_a g\). Hence, setting \(t = C \sqrt{d \log r}\) in~\eqref{eq:Hanson_Wright} gives
\[
\mathbb{P} \left( \left|g^\top \Pi^ag\right| > C\sqrt{d\log r} \right) \le 2 \exp \left ( - c \min \left \{C^2 \log r, C \sqrt{d \log r} \right \} \right).
\]
Since \(\log r = \delta \log d = o(d)\), for all sufficiently large \(d\),
\[
\mathbb{P} \left( \left|g^\top \Pi^ag\right| > C\sqrt{d\log r} \right) \le 2 e^{-c C^2 \log r} = 2r^{-cC^2}.
\]
Taking a union bound over \(a = 1,\ldots,r-1\), we obtain
\begin{equation} \label{eq:HW1}
\mathbb{P} \left( \max_{1\le a\le r-1} \left |g^\top \Pi^a g \right| > C \sqrt{d\log r} \right) \le 2r^{1-cC^2}.
\end{equation}
Choosing \(C\) sufficiently large, the right-hand side is \(o(1)\). For the lower bound on \(\|g\|^2\), we apply again~\eqref{eq:Hanson_Wright} to \(A=\Id_d\), for which \(\E\|g\|^2=d\), with \(t=d/2\). This gives
\begin{equation} \label{eq:HW2}
\mathbb{P} \left(\|g\|^2 < \frac{d}{2} \right) \le \mathbb{P} \left( \left|\|g\|^2-d\right|>\frac{d}{2} \right) \le 2e^{-cd}.
\end{equation}
Therefore, combining~\eqref{eq:HW1} and~\eqref{eq:HW2}, with probability at least \(1-2r^{1-cC^2}-2e^{-cd} = 1-o(1)\), we have
\[
\mu_\ast = \max_{1\le a\le r-1} \frac{\left|g^\top \Pi^ag\right|}{\|g\|^2} \le 2C \sqrt{\frac{\log r}{d}}.
\]
Finally, since \(r=d^\delta\) with \(\delta<1/2\) by Assumption~\ref{hyp:scaling},
\[
r\mu_\ast \le 2C d^{\delta-\frac12}\sqrt{\log r} = 2C\sqrt{\delta} d^{\delta-\frac12} \sqrt{\log d} = o\left((\log d)^{-2}\right).
\]
This completes the proof.
\end{proof}

We next introduce a generic initialization \(u\sim\operatorname{Unif}(\mathbb S^{d-1})\), independent of \(w^\ast\), and study its correlations with the teacher orbit. They can therefore be applied to the tied initialization and, separately, to any fixed component of the untied or augmented initialization. Recall that
\[
m_k(u) = \langle \Pi^k w^\ast,u\rangle, \qquad  q_k(u) =
\begin{cases}
|m_k(u)|, & p\text{ even},\\[1mm]
m_k(u), & p\text{ odd}.
\end{cases}
\]

\begin{lemma}[Initialization scale]\label{lem:initialization_scale}
Suppose that Assumption~\ref{hyp:scaling} holds. Let \(w^\ast,u\) be independent and uniformly distributed on \(\mathbb S^{d-1}\). Then
\[
\max_{0\le k<r} \left| m_k (u)  \right| = \left (\sqrt{2}+ o_{\mathbb{P}}(1)\right ) \gamma_d.
\]
Moreover,
\[
\max_{0\le k<r} q_k(u) = \left (\sqrt2+o_{\mathbb{P}}(1) \right) \gamma_d.
\]
In particular, there exist constants \(0<c<C<\infty\) such that, with probability \(1-o(1)\),
\[
\max_{0\le k<r}|m_k(u)|\le C\gamma_d, \qquad \max_{0\le k<r}q_k(u)\ge c\gamma_d.
\]
\end{lemma}

We prove Lemma~\ref{lem:initialization_scale} below. We next show that the largest relevant correlation is separated from all competing correlations by a relative gap of order \((\log d)^{-2}\).

\begin{lemma}[Initialization gap]\label{lem:initialization_gap}
Suppose that Assumption~\ref{hyp:scaling} holds. Let \(w^\ast,u\) be independent and uniformly distributed on \(\mathbb S^{d-1}\), and set \(\nu_d\coloneqq(\log d)^{-2}\). With probability \(1-o(1)\), the maximizer
\[
i_\ast\coloneqq\arg\max_{0\le k<r}q_k(u)
\]
is unique and satisfies
\[
\max_{k\neq i_\ast}q_k(u) \le (1-\nu_d)q_{i_\ast}(u).
\]
\end{lemma}

We now prove Lemmas~\ref{lem:initialization_scale} and \ref{lem:initialization_gap} together.
\begin{proof}[Proof of Lemmas~\ref{lem:initialization_scale} and \ref{lem:initialization_gap}]
We first record the extreme-value estimates used below. Let \(Z_0,\ldots,Z_{r-1}\) be i.i.d.\ \(\mathcal{N} (0,1)\). Denote by \(A_1\ge A_2\) the two largest values among \(|Z_0|,\ldots,|Z_{r-1}|\), and by \(P_1\ge P_2\) the two largest values among \(Z_0,\ldots,Z_{r-1}\). Standard Gaussian tail estimates imply
\begin{equation} \label{eq:gaussian_maxima}
\frac{A_1}{\sqrt{2\log r}}\xrightarrow{\mathbb{P}}1, \qquad \frac{P_1}{\sqrt{2\log r}}\xrightarrow{\mathbb{P}}1.
\end{equation}
Moreover, applying the asymptotics of extreme spacings in the Gumbel domain~\cite[Theorem~5]{nagaraja} to the standard normal and half-normal distributions gives
\begin{equation} \label{eq:extreme_spacing}
\sqrt{2\log r}\,(A_1-A_2) \xrightarrow{d}\operatorname{Exp}(1), \qquad \sqrt{2\log r}\,(P_1-P_2)\xrightarrow{d}\operatorname{Exp}(1).
\end{equation}
We now deduce the relative gap needed below. Recall that \(\nu_d = (\log d)^{-2}\). Since \(r=d^\delta\) by Assumption~\ref{hyp:scaling}, $\nu_d \log r = \delta ( \log d)^{-1} \to 0$. Moreover, by~\eqref{eq:gaussian_maxima},
\[
\nu_d \sqrt{2\log r} A_1 = (1+o_{\mathbb{P}}(1)) 2\nu_d\log r \xrightarrow{\mathbb{P}}0.
\]
For every fixed \(a>0\), setting
\[
U_d \coloneqq \sqrt{2 \log r} \, (A_1-A_2), \qquad V_d \coloneqq a \nu_d \sqrt{2 \log r} \, A_1,
\]
we have \(U_d\xrightarrow{d}\operatorname{Exp}(1)\) and \(V_d\xrightarrow{\mathbb{P}}0\). For every \(\varepsilon>0\),
\[
\mathbb{P}(U_d\le V_d) \le \mathbb{P}(U_d\le\varepsilon) + \mathbb{P}(V_d>\varepsilon).
\]
Taking the limit superior and then letting \(\varepsilon\downarrow0\) gives
\[
\mathbb{P}(U_d\le V_d)\to 0,
\]
since an \(\operatorname{Exp}(1)\) random variable is strictly positive almost surely. Hence, for every fixed \(a>0\),
\[
\mathbb{P} \left( A_1-A_2 \le a\nu_d A_1 \right) \to 0.
\]
Equivalently,
\begin{equation}\label{eq:absolute_gaussian_gap_general}
A_2\le(1-a\nu_d)A_1
\end{equation}
with probability \(1-o(1)\). The same argument yields
\begin{equation}\label{eq:signed_gaussian_gap_general}
P_2\le(1-a\nu_d)P_1
\end{equation}
with probability \(1-o(1)\).

We now transfer the preceding extreme-value estimates from independent Gaussian coordinates to the actual correlations with the teacher orbit. Define the teacher-geometry event of Lemma~\ref{lem:teacher_geometry}:
\[
\mathcal{G}_d \coloneqq \left \{ \mu_\ast \le C \gamma_d\right\}.
\]
By that lemma, \(\mathbb{P} \left ( \mathcal{G}_d \right ) = 1-o(1)\). Fix \(w^\ast \in \mathcal{G}_d\). In what follows, probabilities are taken conditionally on this realization of \(w^\ast\). Write
\[
u=\frac{h}{\|h\|}, \qquad h\sim\mathcal{N}(0,\Id_d),
\]
where \(h\) is independent of \(w^\ast\), and define
\[
X_k \coloneqq \left\langle \Pi^kw^\ast,h\right\rangle, \qquad 0\le k < r.
\]
Indeed, for \(0\le i,j<r\),
\[
\E[X_iX_j\mid w^\ast] = \left\langle\Pi^iw^\ast,\Pi^jw^\ast\right\rangle.
\]
Conditionally on \(w^\ast\), the vector \(X \coloneqq (X_0,\ldots,X_{r-1})\) is centered Gaussian with covariance matrix
\[
W^\ast \coloneqq \left( \left\langle \Pi^iw^\ast,\Pi^jw^\ast \right\rangle \right)_{0\le i,j\le r-1}.
\]
For a matrix $A = (A_{ij})_{i,j=1}^r$, let
\[
\|A\|_{\infty \to \infty} \coloneqq \sup_{\|x\|_\infty \le 1} \|A x\|_\infty = \max_{1 \le i \le r} \sum_{j=1}^r |A_{ij}|
\]
denote the operator norm induced by the vector \(\ell_\infty\)-norm. The diagonal entries of \(W^\ast\) are equal to one. Moreover, if \(i \neq j\), then
\[
\left| \left\langle\Pi^iw^\ast,\Pi^jw^\ast\right\rangle \right| = \left| \left\langle w^\ast,\Pi^{j-i} w^\ast \right\rangle \right| \le\mu_\ast,
\]
where the exponent is understood modulo \(r\). Hence, on \(\mathcal{G}_d\),
\[
\|W^\ast-\Id_r\|_{\infty\to\infty} \le (r-1)\mu_\ast \le C r \gamma_d.
\]
Moreover, \(b_d \coloneqq C r \gamma_d = o(\nu_d)\) according to Lemma~\ref{lem:teacher_geometry}. Since \(W^\ast-\Id_r\) is symmetric,
\[
\|W^\ast-\Id_r\|_{\mathrm{op}} \le \|W^\ast-\Id_r\|_{\infty\to\infty} \le b_d =o(1).
\]
Therefore, for all sufficiently large \(d\), every eigenvalue of \(W^\ast\) belongs to \([1-b_d , 1 + b_d] \subset (0,\infty)\), so \(W^\ast\) is positive definite. Set \(B \coloneqq W^\ast - \Id_r\). Then, for all sufficiently large \(d\), \(\|B\|_{\infty \to \infty} \le b_d < 1/2\). Hence the binomial series for the principal square root converges absolutely in the \(\|\cdot \|_{\infty \to \infty}\)-norm:
\[
(W^\ast)^{1/2} = (\Id_r + B)^{1/2} = \Id_r + \sum_{m=1}^\infty \binom{1/2}{m} B^m.
\]
By submultiplicativity,
\[
\| (W^\ast)^{1/2} - \Id_r \|_{\infty \to \infty} \le  \sum_{m=1}^\infty  \left | \binom{1/2}{m} \right | \|B\|_{\infty \to \infty}^m \le C  \|B\|_{\infty \to \infty} \le C b_d,
\]
where \(C<\infty\) is a universal constant. Therefore, if we set
\[
\varepsilon_d (w^\ast)\coloneqq \| (W^\ast)^{1/2} - \Id_r \|_{\infty \to \infty},
\]
then, uniformly over \(w^\ast \in \mathcal{G}_d\),
\begin{equation}\label{eq:epsilon_bound}
\varepsilon_d  (w^\ast) \le C b_d = o (\nu_d).
\end{equation}
Now define \(Z \coloneqq (W^\ast)^{-1/2}X\). Conditionally on \(w^\ast\), \(Z \sim \mathcal{N} (0,\Id_r)\), so its coordinates are independent standard Gaussian variables. Let \(A_1\ge A_2\) and \(P_1\ge P_2\) denote the corresponding absolute and signed order statistics, as above. Since \(X=(W^\ast)^{1/2}Z\), we have
\begin{equation}\label{eq:gaussian_coupling_error}
\|X-Z \|_\infty \le \varepsilon_d (w^\ast) \|Z\|_\infty = \varepsilon_d (w^\ast) A_1.
\end{equation}

\paragraph{Absolute maximum.}
Let \(k_\ast\) be the almost surely unique index satisfying $|Z_{k_\ast}|=A_1$. Taking \(a=2\) in~\eqref{eq:absolute_gaussian_gap_general}, with conditional probability \(1-o(1)\),
\[
A_2\le(1-2\nu_d)A_1.
\]
On this event,
\[
|X_{k_\ast}| \ge A_1-\|X-Z\|_\infty \ge (1-\varepsilon_d)A_1,
\]
while
\[
\max_{k\neq k_\ast}|X_k| \le A_2+\|X-Z\|_\infty \le (1-2\nu_d+\varepsilon_d)A_1.
\]
By~\eqref{eq:epsilon_bound}, \(\varepsilon_d=o(\nu_d)\) uniformly on \(\mathcal{G}_d\). Hence, for all sufficiently large \(d\),
\[
1-2\nu_d+\varepsilon_d \le (1-\nu_d)(1-\varepsilon_d).
\]
It follows that
\begin{equation}\label{eq:absolute_X_gap}
\max_{k\neq k_\ast}|X_k| \le (1-\nu_d)|X_{k_\ast}|.
\end{equation}
In particular, \(k_\ast\) is also the unique maximizer of \(|X_k|\). Furthermore,
\[
\left| \max_k|X_k|-A_1 \right| \le \|X-Z\|_\infty \le \varepsilon_d A_1.
\]
Since \(\varepsilon_d=o(1)\) uniformly on \(\mathcal{G}_d\), it follows from~\eqref{eq:gaussian_maxima} that, conditionally on \(w^\ast\in\mathcal{G}_d\),
\begin{equation}\label{eq:absolute_X_scale}
\max_{0\le k<r}|X_k| = (1+o_{\mathbb{P}}(1))\sqrt{2\log r},
\end{equation}
where the \(o_{\mathbb{P}}(1)\) is uniform over \(w^\ast\in\mathcal G_d\).

\paragraph{Signed maximum.} Let \(j_\ast\) be the almost surely unique index satisfying $Z_{j_\ast}=P_1$. By~\eqref{eq:gaussian_maxima},
\[
\frac{A_1}{P_1} \xrightarrow{\mathbb{P}}1.
\]
In particular, with probability \(1-o(1)\), \(A_1 \le 2P_1\). Intersect this event with the spacing event \(P_2\le(1-2\nu_d)P_1\), which also has probability \(1-o(1)\) by~\eqref{eq:signed_gaussian_gap_general}. On their intersection,
\[
X_{j_\ast} \ge P_1-\varepsilon_dA_1 \ge (1-2\varepsilon_d)P_1,
\]
while
\[
\max_{k\neq j_\ast}X_k \le P_2 + \varepsilon_dA_1 \le (1-2\nu_d+2\varepsilon_d)P_1.
\]
Since \(\varepsilon_d=o(\nu_d)\), for all sufficiently large \(d\),
\[
1-2\nu_d+2\varepsilon_d \le (1-\nu_d)(1-2\varepsilon_d).
\]
Therefore
\begin{equation}\label{eq:signed_X_gap}
\max_{k\neq j_\ast}X_k \le (1-\nu_d)X_{j_\ast}.
\end{equation}
Thus \(j_\ast\) is also the unique maximizer of the signed coordinates of \(X\). Furthermore,
\[
\left|\max_{0\le k<r}X_k-P_1\right| \le \|X-Z\|_\infty \le \varepsilon_dA_1.
\]
Since \(A_1/P_1\to1\) in probability and \(\varepsilon_d=o(1)\), we obtain
\begin{equation}\label{eq:signed_X_scale}
\max_kX_k = (1+o_{\mathbb{P}}(1))\sqrt{2\log r},
\end{equation}
again uniformly over \(w^\ast\in\mathcal{G}_d\).

\paragraph{Normalization and conclusion.}
Finally,
\[
m_k(u) = \left\langle\Pi^kw^\ast,u\right\rangle = \frac{X_k}{\|h\|}.
\]
Since $ \|h\| / \sqrt d \xrightarrow{\mathbb{P}}1$, combining this with~\eqref{eq:absolute_X_scale} yields
\[
\max_{0\le k<r}|m_k(u)| = (1+o_{\mathbb{P}}(1))\sqrt{\frac{2\log r}{d}} = (\sqrt2+o_{\mathbb{P}}(1))\gamma_d.
\]
Similarly,~\eqref{eq:signed_X_scale} gives
\[
\max_{0\le k<r}m_k(u) = (\sqrt2+o_{\mathbb{P}}(1))\gamma_d.
\]
Consequently,
\[
\max_{0\le k<r}q_k(u) = (\sqrt2+o_{\mathbb{P}}(1))\gamma_d,
\]
which proves Lemma~\ref{lem:initialization_scale}. The normalization by the common positive factor \(\|h\|\) does not change either the maximizing index or any relative gap. Hence~\eqref{eq:absolute_X_gap} gives, when \(p\) is even,
\[
\max_{k\neq i_\ast}q_k(u) = \max_{k\neq i_\ast}|m_k(u)| \le (1-\nu_d)|m_{i_\ast}(u)| = (1-\nu_d)q_{i_\ast}(u),
\]
whereas~\eqref{eq:signed_X_gap} gives, when \(p\) is odd,
\[
\max_{k\neq i_\ast}q_k(u) = \max_{k\neq i_\ast}m_k(u) \le (1-\nu_d)m_{i_\ast}(u) = (1-\nu_d)q_{i_\ast}(u).
\]
In both cases the maximizer is unique. All the preceding conditional estimates hold uniformly over \(w^\ast\in\mathcal{G}_d\). Therefore, if \(\mathcal{E}_d\) denotes the event on which the conclusions of Lemmas~\ref{lem:initialization_scale} and~\ref{lem:initialization_gap} hold, then
\[
\mathbb{P}(\mathcal{E}_d^\mathrm{c}) \le \mathbb{P}(\mathcal{G}_d^\mathrm{c}) + \sup_{w^\ast\in\mathcal{G}_d} \mathbb{P}(\mathcal{E}_d^\mathrm{c} \mid w^\ast) = o(1).
\]
This completes the proofs of Lemmas~\ref{lem:initialization_scale} and~\ref{lem:initialization_gap}.
\end{proof}

\section{Proofs of stochastic gradient estimates}
\label{sec:stochastic_gradient_proofs}

We use the notation of Section~\ref{sec:corr_dynamics} and the teacher-geometry event of Lemma~\ref{lem:teacher_geometry}.

\subsection{Feature-gradient covariance}
\label{sec:feature_gradient_covariance_proof}

\begin{proof}[Proof of Lemma~\ref{lem:feature_gradient_covariance}]
Computing the Euclidean gradient of \(f_w(x)\) defined in Definition~\ref{def:tied_student} gives
\[
\widehat\nabla f_w(x) = \frac1{\sqrt r} \sum_{i=0}^{r-1} \sigma' \left(\langle\Pi^iw,x\rangle\right)\Pi^{-i}x.
\]
Set \(Z_i = \langle \Pi^i w, x \rangle\) and \(V_i = \langle \Pi^i v, x \rangle\). Then
\[
\langle v, \widehat\nabla f_w(x) \rangle = \frac{1}{\sqrt{r}} \sum_{i=0}^{r-1} \sigma' (Z_i) V_i,
\]
and therefore, by cyclic invariance,
\[
\E \left [ \langle v, \widehat\nabla f_w(x) \rangle^2 \right ] = \sum_{j=0}^{r-1} \E \left [\sigma'(Z_0) \sigma' (Z_j) V_0 V_j \right ].
\]
For the jointly Gaussian vector \((Z_0, Z_j, V_0,V_j)\), since 
\[
\operatorname{Cov} (V_0,Z_0) = \operatorname{Cov} (V_j,Z_j) = \langle v,w\rangle =0, 
\]
Gaussian integration by parts gives
\[
 \E \left [\sigma'(Z_0) \sigma' (Z_j) V_0 V_j \right ] = \langle v, \Pi^j v\rangle \E [\sigma'(Z_0)\sigma'(Z_j)] + \langle \Pi^j w,v\rangle \langle w, \Pi^j v\rangle \E[\sigma''(Z_0)\sigma''(Z_j)] .
\]
Using
\[
\E [\sigma'(Z_0)\sigma'(Z_j)] = \phi'(a_j(w)), \qquad \E[\sigma''(Z_0)\sigma''(Z_j)] = \phi''(a_j(w)),
\]
yields the announced statement.
\end{proof}

\subsection{Coordinate noise moments}
\label{sec:coordinate_noise_moments_proof}

\begin{proof}[Proof of Lemma~\ref{lem:coordinate_noise_moments}]
Fix a teacher \(w^\ast\in\mathcal{G}_d\) and condition on \(\mathcal{F}_{t-1}\). Then \(w\coloneqq w_{t-1}\) is fixed and the only remaining randomness comes from the fresh Gaussian input \(x_t\). Throughout the proof, all expectations are with respect to this fresh sample. The constants below are uniform over \(w\in\mathbb S^{d-1}\) and \(w^\ast\in\mathcal{G}_d\).

We first claim that
\begin{equation}\label{eq:teacher_overlap_energy}
\sum_{k=0}^{r-1} m_k(w)^2 \le 2.
\end{equation}
Let
\[
U^\ast \coloneqq [ w^\ast, \Pi w^\ast, \ldots, \Pi^{r-1} w^\ast] \in \R^{d \times r},
\]
so that \(\left ( m_0(w), \ldots, m_{r-1} (w) \right ) = (U^\ast)^\top w\). Its Gram matrix is \( W^\ast = (U^\ast)^\top U^\ast\). On \(\mathcal{G}_d\), the diagonal entries of \(W^\ast\) are equal to one and the absolute sum of the off-diagonal entries in each row is at most
\[
(r-1) \mu_\ast \le r \mu_\ast = o(1).
\]
Hence, for all sufficiently large \(d\),
\[
\| W^\ast \|_{\mathrm{op}} \le \| W^\ast \|_{\infty \to \infty} \le 1 + (r-1) \mu_\ast \le 2.
\]
This implies that
\[
\sum_{k=0}^{r-1} m_k(w)^2  = \|(U^\ast)^\top w \|^2 \le \| U^\ast (U^\ast)^\top\|_{\mathrm{op}} = \|W^\ast\|_{\mathrm{op}} \le 2,
\]
thus proving~\eqref{eq:teacher_overlap_energy}.

Recall that $c_a=\langle w^\ast,\Pi^aw^\ast\rangle$ and $\phi (c) = \sum_{k=p}^P a_k^2 c^k$. By~\eqref{eq:teacher},
\[
\E[y^2] = \frac{1}{r} \sum_{i,j=0}^{r-1} \phi (c_{j-i}) = \sum_{a=0}^{r-1} \phi(c_a),
\]
because each difference \(a = j- i (\mod r)\) occurs exactly \(r\) times. For every \(|c| \le 1\),
\[
|\phi (c)| \le \sum_{k=p}^P a_k^2 |c|^k \le \left (\sum_{k=p}^P a_k^2 \right ) |c|^p = \|\sigma\|^2_{L^2} |c|^p.
\]
Therefore
\[
\E[y^2] \le \phi(1)+\sum_{a=1}^{r-1}|\phi(c_a)| \le C_\sigma(1+r \mu_\ast^p)\le C_\sigma,
\]
where the last inequality holds uniformly on \(\mathcal{G}_d\) for all sufficiently large \(d\). Indeed, \(\mu_\ast \le 1\) and \(r \mu_\ast = o(1)\). Since $y$ is a polynomial of degree at most $P$ in the Gaussian input, Gaussian hypercontractivity gives
\begin{equation}\label{eq:y_L4_general}
\|y\|_{L^4} \le C_\sigma.
\end{equation}

Fix \(0 \le k< r\), set \(m \coloneqq m_k(w)\), and define the tangent projection of the \(k\)th teacher direction by
\[
v \coloneqq (\Id_d - ww^\top) \Pi^k w^\ast = \Pi^k w^\ast - m w.
\]
Then $v\perp w$ and $\|v\|^2=1-m^2 \le 1$. Moreover, since \(G =  - y \nabla f_w(x)\), we have
\[
\langle \Pi^kw^\ast,G\rangle =-y\langle v,\widehat\nabla f_w(x)\rangle.
\]
Set \(Q \coloneqq \langle v, \widehat \nabla f_w(x) \rangle\) and \(a_j(w) \coloneqq \langle w, \Pi^j w \rangle\). By Lemma~\ref{lem:feature_gradient_covariance},
\begin{equation}\label{eq:Q_second_moment}
\E Q^2 = \sum_{j=0}^{r-1} \left( \phi'(a_j(w))\langle v,\Pi^jv\rangle + \phi''(a_j(w)) \langle w,\Pi^jv\rangle \langle\Pi^jw,v\rangle \right).
\end{equation}
For \(j=0\),
\[
\langle v,\Pi^0v\rangle=\|v\|^2\le1, \qquad \langle w,v\rangle=0,
\]
so the \(j=0\) contribution to~\eqref{eq:Q_second_moment} is bounded by a constant depending only on \(\sigma\). For \(j\neq0\), with all indices understood modulo \(r\),
\begin{align*}
\langle v,\Pi^jv\rangle &= c_j-m\bigl(m_{k+j}(w)+m_{k-j}(w)\bigr)+m^2a_j(w),\\
\langle w,\Pi^jv\rangle\langle \Pi^jw,v\rangle &= \bigl(m_{k+j}(w)-ma_j(w)\bigr)\bigl(m_{k-j}(w)-ma_j(w)\bigr).
\end{align*}
Using~\eqref{eq:teacher_overlap_energy}, the Cauchy--Schwarz inequality, and \(|a_j(w)|\le 1\), we obtain
\begin{equation} \label{eq:v_shift_v_sum}
\sum_{j\neq0} |\langle v,\Pi^jv\rangle| \le r\mu_\ast + 2|m|\sum_{j=0}^{r-1}|m_j(w)| + rm^2 \le r\mu_\ast + 2\sqrt{2r}|m| + rm^2,
\end{equation}
and
\begin{equation}\label{eq:v_shift_product_sum}
\begin{split}
\sum_{j\neq0} \left| \langle w,\Pi^jv\rangle \langle\Pi^jw,v\rangle \right| & \le \sum_{j=0}^{r-1} |m_{k+j}(w)m_{k-j}(w)| + 2|m|\sum_{j=0}^{r-1}|m_j(w)|+ rm^2 \\ &\le 2+2\sqrt{2r}|m|+rm^2.
\end{split}
\end{equation}
Here, for the first term in the last display, Cauchy--Schwarz and~\eqref{eq:teacher_overlap_energy} give
\[
\sum_{j=0}^{r-1} |m_{k+j}(w)m_{k-j}(w)| \le \left(\sum_jm_{k+j}(w)^2\right)^{1/2} \left(\sum_jm_{k-j}(w)^2\right)^{1/2} \le2.
\]
Since \(\phi'\) and \(\phi''\) are bounded on \([-1,1]\), combining~\eqref{eq:Q_second_moment}--\eqref{eq:v_shift_product_sum} gives
\[
\E Q^2 \le C_\sigma \left( 1+r\mu_\ast+\sqrt r\,|m|+rm^2 \right).
\]
On \(\mathcal{G}_d\), \(r\mu_\ast=o(1)\), and \(\sqrt r\,|m| \le (1+rm^2)/2\). Therefore,
\begin{equation}\label{eq:Q_second_moment_bound}
\E Q^2 \le C_\sigma(1+rm^2).
\end{equation}
The polynomial \(Q\) has degree at most \(P\) in the Gaussian input. Hence, by Gaussian hypercontractivity,
\[
\|Q\|_{L^4} \le 3^{P/2}\|Q\|_{L^2}.
\]
Using Cauchy--Schwarz,~\eqref{eq:y_L4_general}, and~\eqref{eq:Q_second_moment_bound},
\[
\E \left[ \langle\Pi^kw^\ast,G\rangle^2 \right] = \E [y^2Q^2] \le \|y\|_{L^4}^2 \|Q\|_{L^4}^2 \le C_\sigma(1+rm^2).
\]
Now
\[
\xi_{k,t} = -\left( \langle \Pi^kw^\ast,G_t\rangle - \E [ \langle\Pi^kw^\ast,G_t\rangle \mid \mathcal{F}_{t-1}] \right).
\]
Centering can only decrease the second moment, and therefore
\[
\E [ \xi_{k,t}^2 \mid \mathcal{F}_{t-1}] \le C_\sigma \left(1+r\,m_k(w_{t-1})^2\right),
\]
which proves~\eqref{eq:coordinate_noise_variance}.

Conditionally on \(\mathcal{F}_{t-1}\), the random variable \(\langle\Pi^kw^\ast,G_t\rangle\) is a polynomial of degree at most \(2P\) in the fresh Gaussian input \(x_t\). Subtracting its conditional expectation does not increase its degree, so \(\xi_{k,t}\) is also a Gaussian polynomial of degree at most \(2P\). Hence, for every \(\nu\ge2\), Gaussian hypercontractivity yields
\[
\|\xi_{k,t}\|_{L^\nu(\,\cdot\mid\mathcal{F}_{t-1})} \le (\nu-1)^P\|\xi_{k,t}\|_{L^2(\,\cdot\mid\mathcal{F}_{t-1})}.
\]
Using~\eqref{eq:coordinate_noise_variance},
\[
\left( \E[ |\xi_{k,t}|^\nu \mid\mathcal{F}_{t-1}] \right)^{1/\nu} \le C_\sigma\nu^P\sqrt{1+r\,m_k(w_{t-1})^2} \le C_\sigma\nu^P\left( 1+\sqrt r\,|m_k(w_{t-1})| \right),
\]
which proves~\eqref{eq:coordinate_noise_moments}.
\end{proof}

\subsection{Uniform control of the stochastic gradient}
\label{sec:uniform_gradient_proof}

\begin{proof}[Proof of Lemma~\ref{lem:uniform_gradient}]
We first work conditionally on a teacher \(w^\ast\in\mathcal{G}_d\) and establish the uniform second-moment bound
\begin{equation}\label{eq:second_moment}
\sup_{w\in\mathbb{S}^{d-1}} \E \|G(w;x,y)\|^2 \le C_\sigma d.
\end{equation}
All expectations in this first part are with respect to a fresh input \(x\sim\mathcal{N}(0,\Id_d)\). Fix $w$ and write $\mathrm{P}_w \coloneqq \Id_d-ww^\top$. Since
\[
G(w; x,y) = -y \nabla_w f_w(x) = - y \, \mathrm{P}_w \widehat{\nabla}_w f_w (x),
\]
Cauchy--Schwarz inequality gives
\begin{equation} \label{eq:norm_G}
\E \| G\|^2 = \E \left [ y^2  \| \mathrm{P}_w \widehat{\nabla}_w f_w (x)\|^2 \right] \le \|y\|_{L^4}^2 \left( \E\|\mathrm{P}_w \widehat{\nabla}_w f_w (x)\|^4 \right)^{1/2}.
\end{equation}
By~\eqref{eq:y_L4_general}, \(\|y\|_{L^4} \le C_\sigma\), uniformly over teachers in \(\mathcal{G}_d\). Choose an orthonormal basis $(v_1,\ldots,v_{d-1})$ of $w^\perp$, and set \(Q_i \coloneqq \langle v_i, \widehat{\nabla}_w f_w(x)\rangle\). Then
\[
\|\mathrm{P}_w \widehat{\nabla}_w f_w (x)\|^2 = \sum_{i=1}^{d-1} Q_i^2.
\]
Hence, by Minkowski's inequality in \(L^2\),
\[
\left(\E\|\mathrm{P}_w \widehat{\nabla}_w f_w (x)\|^4 \right )^{1/2}  = \left\|\sum_{i=1}^{d-1} Q_i^2 \right\|_{L^2} \le \sum_{i=1}^{d-1}\left\|Q_i^2\right\|_{L^2} =  \sum_{i=1}^{d-1}  \left\|Q_i\right\|_{L^4}^2.
\]
Each \(Q_i\) is a polynomial of degree at most $P$ in the Gaussian input $x$. Gaussian hypercontractivity therefore yields
\[
\left\|Q_i\right\|_{L^4} \le 3^{P/2} \left\|Q_i\right\|_{L^2},
\]
and consequently
\begin{equation} \label{eq:minkowski}
\left (\E\|\mathrm{P}_w \widehat{\nabla}_w f_w (x)\|^4 \right)^{1/2} \le 3^P\sum_{i=1}^{d-1}\E\left[Q_i^2 \right]  .
\end{equation}
It remains to control the sum on the right-hand side of~\eqref{eq:minkowski}. By Lemma~\ref{lem:feature_gradient_covariance},
\[
\sum_{i=1}^{d-1}\E\left[Q_i^2 \right ] = \sum_{j=0}^{r-1} \left(\phi'(a_j(w))\sum_{i=1}^{d-1}\langle v_i,\Pi^j v_i\rangle + \phi''(a_j(w))\sum_{i=1}^{d-1}\langle w,\Pi^jv_i\rangle \langle\Pi^jw,v_i\rangle \right),
\]
where  \(a_j(w)= \langle w, \Pi^j w \rangle\). For the first sum, since \(\mathrm{P}_w = \sum_{i=1}^{d-1} v_iv_i^\top\), we have
\[
\sum_{i=1}^{d-1} \langle v_i, \Pi^j v_i \rangle = \tr(\mathrm{P}_w \Pi^j).
\]
If \(j=0\), this equals \(d-1\). If \(1 \le j < r\), then \(\tr (\Pi^j)=0\), and therefore
\[
\tr(\mathrm{P}_w \Pi^j) = - \langle w, \Pi^j w\rangle = - a_j (w).
\]
As a result,
\[
\sum_{j=0}^{r-1}\left|\sum_{i=1}^{d-1}\langle v_i, \Pi^j v_i\rangle\right| \le d + r.
\]
For the second sum,
\[
\sum_{i=1}^{d-1}\langle w, \Pi^j v_i\rangle \langle\Pi^j w, v_i\rangle = \langle\mathrm{P}_w\Pi^{-j} w, \mathrm{P}_w\Pi^j w\rangle,
\]
and hence its absolute value is at most one. Since \(\phi'\) and \(\phi''\) are uniformly bounded on \([-1,1]\),
\[
\sum_{i=1}^{d-1}\E\left[Q_i^2\right] \le C_\sigma(d+r) \le C_\sigma d,
\]
where we used \(r=d^\delta \le d\). Combining this estimate with~\eqref{eq:minkowski} and~\eqref{eq:norm_G} proves~\eqref{eq:second_moment}.

We now derive higher conditional moments. Fix \(t\ge1\) and condition on \(\mathcal{F}_{t-1}\). Then \(w_{t-1}\) is fixed and the only randomness comes from the fresh Gaussian input \(x_t\). Choose an orthonormal basis \((v_1,\ldots,v_{d-1})\) of \(w_{t-1}^\perp\), and set \(R_i\coloneqq\langle v_i,G_t\rangle\). Each \(R_i\) is a polynomial of degree at most \(2P\) in \(x_t\). Thus, for every \(\nu\ge2\), Gaussian hypercontractivity gives
\[
\|R_i\|_{L^\nu(\,\cdot\mid\mathcal{F}_{t-1})} \le (\nu-1)^P \|R_i\|_{L^2(\,\cdot\mid\mathcal{F}_{t-1})}.
\]
Using Minkowski in \(L^{\nu/2}\),
\[
\begin{aligned}
\left( \E [ \|G_t\|^\nu \mid\mathcal{F}_{t-1}] \right)^{2/\nu} & = \left\| \sum_{i=1}^{d-1} R_i^2 \right\|_{L^{\nu/2}}\\ &\le \sum_{i=1}^{d-1} \|R_i\|_{L^\nu}^2\\ &\le (\nu-1)^{2P} \sum_{i=1}^{d-1} \|R_i\|_{L^2}^2\\ & = (\nu-1)^{2P} \E[ \|G_t\|^2 \mid\mathcal{F}_{t-1}]\\ & \le C_\sigma d\,\nu^{2P},
\end{aligned}
\]
where the last inequality follows from~\eqref{eq:second_moment}. Therefore
\begin{equation}\label{eq:gradient_high_moments}
\left( \E[ \|G_t\|^\nu \mid\mathcal{F}_{t-1}] \right)^{1/\nu} \le C_\sigma\sqrt d\nu^P.
\end{equation}
Take \(\nu_d=\lceil c\log d\rceil\) for a fixed \(c>0\). By conditional Markov's inequality and~\eqref{eq:gradient_high_moments}, for every \(D>0\) one can choose \(L=L(\sigma,D)<\infty\) sufficiently large so that
\[
\mathbb{P} \left( \|G_t\|> L\sqrt d\,(\log d)^P \,\middle|\, \mathcal{F}_{t-1} \right) \le d^{-D}
\]
uniformly in \(t\), for all sufficiently large \(d\), on the event \(\mathcal{G}_d\). Choose \(D>C+1\). Taking expectations and applying a union bound over \(1\le t\le T\le d^C\) gives
\[
\mathbb{P} \left( \left \{ \max_{1\le t\le T} \|G_t\| > L\sqrt d\,(\log d)^P \right \} \cap  \mathcal{G}_d \right) \le d^{C-D} =o(1).
\]
Finally, \(\mathbb{P}(\mathcal{G}_d^\mathrm{c})=o(1)\) by Lemma~\ref{lem:teacher_geometry}, so that
\[
\mathbb{P} \left( \max_{1\le t\le T} \|G_t\| > L \sqrt{d}\, (\log d)^P \right) = o(1).
\]
This proves the lemma.
\end{proof}